\documentclass[11pt]{article}

\usepackage{microtype}
\usepackage{graphicx}
\usepackage{subfigure}
\usepackage{booktabs} 

\usepackage{amsmath}
\usepackage{graphicx}

\usepackage[letterpaper,margin=1in]{geometry}

\usepackage{algorithm}
\usepackage{algorithmic}

\usepackage{amsmath}
\usepackage{amssymb}
\usepackage{mathtools}
\usepackage{amsthm}

\usepackage[textsize=tiny]{todonotes}

\usepackage[english]{babel}
\usepackage{mysty}

\newcommand{\citep}{\cite}
\newcommand{\F}{\mathcal{F}}
\newcommand{\U}{\mathcal{U}_{(a,b)}}
\newcommand{\Ltwo}{\mathcal{L}_2}
\newcommand{\Lsur}{\mathcal{L}_{\mathrm{sur}}}
\newcommand{\htLsur}{\widehat{\mathcal{L}}_{\mathrm{sur}}}

\newcommand{\htystr}{\hat{y}^*}
\newcommand{\htustrmt}{\hat{u}^{*t}}
\newcommand{\htumt}{\hat{u}^t}
\newcommand{\barw}{\bar{\w}}
\newcommand{\htw}{\widehat{\w}}
\newcommand{\frkwstr}{\beta}
\newcommand{\tv}{\tilde{\bv}}

\newcommand{\uw}{u_{\w}}
\newcommand{\ustrw}{u_{\w}^*}
\newcommand{\htuw}{\hat{u}_{\w}}
\newcommand{\htustrw}{\hat{u}_{\w}^*}

\newcommand{\htvmt}{\hat{v}^t}

\crefname{ineq}{inequality}{inequalities}
\Crefname{ineq}{Inequality}{Inequalities}
\creflabelformat{ineq}{#2{\upshape(#1)}#3}

\title{Robustly Learning Single-Index Models via Alignment Sharpness}
\author{
Nikos Zarifis\thanks{Supported in part by NSF Medium Award CCF-2107079 and NSF award 2023239.} ‖\\
UW Madison\\
{\tt zarifis@wisc.edu}\\
\and Puqian Wang\thanks{Supported in part by NSF Award CCF-2007757.} ‖\\
UW Madison\\
{\tt pwang333@wisc.edu}\\
\and Ilias Diakonikolas\thanks{Supported by NSF Medium Award CCF-2107079 and
		a DARPA Learning with Less Labels (LwLL) grant.}\\
		UW Madison\\		{\tt ilias@cs.wisc.edu}\\
\and Jelena Diakonikolas\thanks{Supported by NSF Award CCF-2007757 and by the U.\ S.\ Office of
Naval Research under award number N00014-22-1-2348.}\\
UW Madison\\
{\tt jelena@cs.wisc.edu}\\
}
\date{}

\begin{document}
\maketitle
\def\thefootnote{‖}\footnotetext{Equal contribution.}
\def\thefootnote{*}
\def\thefootnote{\arabic{footnote}}

\begin{abstract}
We study the problem of learning Single-Index Models 
under the $L_2^2$ loss in the agnostic model. 
We give an efficient  learning algorithm,
achieving a constant factor approximation to the optimal loss,
that succeeds under a range of distributions 
(including log-concave distributions) and a broad class of 
monotone and Lipschitz link functions.
This is the first efficient constant factor approximate agnostic learner, even for Gaussian data 
and for any nontrivial class of link functions. Prior 
work for the case of unknown link function either works in 
the realizable setting or does not attain constant factor approximation. The main technical 
ingredient enabling our algorithm and analysis 
is a novel notion of a 
local error bound in optimization
that we term {\em alignment sharpness} 
and that may be of broader interest. 
\end{abstract}

\setcounter{page}{0}
\thispagestyle{empty}

\newpage

\section{Introduction}

Single-index models (SIMs)~\cite{ichimura1993semiparametric, hristache2001direct,
hardle2004nonparametric, dalalyan2008new, kalai2009isotron,kakade2011efficient, DudejaH18} are a classical supervised learning model extensively studied 
in statistics and machine learning. 
SIMs capture the common assumption 
that the target function $f$ depends 
on an unknown direction $\vec w$, 
i.e., $f(\x) = u(\vec w \cdot \x)$ 
for some link (a.k.a.\ activation) function $u:\R \mapsto \R$ and $\vec w \in \R^d$. 
In most settings, the link function is unknown 
and is assumed to satisfy certain regularity properties. 
Classical works~\cite{kalai2009isotron,kakade2011efficient} 
studied the efficient learnability of SIMs 
for monotone and Lipschitz link functions
and data distributed on the unit ball. 
These early algorithmic results succeed 
in the realizable setting (i.e., with clean labels) 
or in the presence of zero-mean label noise.

The focus of this work is on learning SIMs in 
the challenging agnostic (or adversarial label 
noise) model~\cite{Haussler:92, KSS:94}, where no assumptions are 
made on the labels of the examples and the goal is to compute a hypothesis 
that is competitive with the {\em best-fit} function in the class. 
Importantly, as will be formalized below, we will not assume a priori 
knowledge of the link function. In more detail, let $\D$ be a distribution 
on labeled examples 
$(\x, y) \in \R^d \times \R$ and 
$\Ltwo(h) =\Exy[(h(\x) - y)^2]$ be the squared loss 
of the hypothesis $h: \R^d \to \R$ with respect to $\D$. 
Given i.i.d.\ samples from $\D$, 
the goal of the learner is to 
output a hypothesis $h$ with squared error competitive with $\opt$, where 
$\opt = \inf_{f \in \cal{C}}\Ltwo(f)$ is the best attainable error 
by any function in the target class $\cal{C}$. 

In the context of this paper, 
the class $\cal{C}$ above 
is the class of SIMs, i.e., 
all functions of the form $f(\x) = u(\vec w \cdot \x)$ 
where both the weight vector $\w$ and the link function $u$ are unknown. 
For this task to be even information-theoretically solvable, 
one requires some assumptions on the
vector $\w$ and the link function $u$. 
We will assume, as is standard, that the $\ell_2$-norm of $\w$ 
is bounded by a parameter $W$. 
We will similarly assume that the link function lies in a family 
of well-behaved functions
that are monotone and satisfy certain Lipschitz properties 
(see \Cref{def:well-behaved-unbounded-intro}). 

For a weight vector $\w$ and link function $u$, 
the $L_2^2$ loss of the SIM hypothesis $u(\w\cdot\x)$ (defined by $u$ and $\w$)
is $\Ltwo(\w;u) = \Exy[(u(\w\cdot\x) - y)^2].$
Our problem of robustly learning SIMs is defined as follows.

\begin{problem}[Robustly Learning Single-Index Models] \label{def:agnostic-learning}
Fix a class of distributions $\mathcal{G}$ on $\R^d$ and a class of link functions\footnote{Throughout this paper, we will use the terms ``link function'' and ``activation'' interchangeably.} 
$\mathcal{F}$. Let $\D$ be a distribution of labeled examples
$(\x,y) \in \R^d \times \R$ such that its $\x$-marginal 
$\D_\x$ belongs to $\mathcal{G}$. We say that an algorithm 
is a {\em $C$-approximate proper SIM learner}, for some $C \geq 1$, if given 
$\eps>0$, $W>0$, and i.i.d.\ samples from $\D$, the algorithm outputs 
a link function $\hat{u}\in \mathcal{F}$ and a vector $\htw\in \R^d$ 
such that with high probability it holds 
\( \Ltwo(\htw;\hat{u}) \leq C \, \opt  +\eps \), where  
$\opt \triangleq \min_{\|\vec w\|_2 \leq W, u\in \mathcal{F}} \Ltwo (\vec w; u)$.  
\end{problem}

Throughout this paper, we use $u^*(\w^*\cdot\x)$ to denote 
a fixed (but arbitrary) optimal solution to the above learning problem, 
i.e., one satisfying $\Ltwo(\w^*;u^*) = \opt$.

Some comments are in order. First, \Cref{def:agnostic-learning} 
does not make realizability assumptions on the distribution $\D$. 
That is, the labels are allowed to be arbitrary and the goal
is to be competitive against the best-fit function 
in the class $\mathcal{C} = \{ u (\w\cdot\x) \mid \w \in \R^d, \|\w\|_2 \leq W, u  \in\mathcal{F} \}$\;.
Second, our focus is on obtaining efficient learners that achieve
a {\em constant factor approximation} to the optimum loss, 
i.e., where $C$ in \Cref{def:agnostic-learning} 
is a {\em universal constant} --- independent 
of the dimension $d$ and the radius $W$ of the weight space.

Ideally, one would like an efficient learner 
that succeeds for all marginal distributions 
and achieves optimal error of $\opt+\eps$ (corresponding to $C=1$). 
Unfortunately, known computational hardness results 
rule out this possibility. 
Even for the very special case that the marginal distribution is Gaussian and the link function 
is known (e.g., a ReLU), there is strong evidence that any algorithm achieving error $\opt+\eps$ 
requires $d^{\poly(1/\eps)}$ time~\cite{DKZ20, GGK20, DKPZ21, DKR23}. 
Moreover, even if we relax our goal to constant factor approximation (i.e., $C = O(1)$), 
distributional assumptions are required both  for proper~\cite{Sima02, MR18} and 
improper learning~\cite{DKMR22}. As a consequence, algorithmic 
research in this area has focused on 
constant factor approximate learners that succeed 
under mild distributional assumptions. 

Recent works~\cite{DGKKS20,DKTZ22,ATV22, WZDD2023} 
gave efficient, constant factor approximate
learners, under natural distributional assumptions, 
for the special case of \Cref{def:agnostic-learning} where the link
function is known {\em a priori} (see also~\cite{FCG20}). 
For the general setting, 
the only prior algorithmic result was recently obtained in~\cite{GGKS23}. 
Specifically, \cite{GGKS23} gave an efficient algorithm that succeeds 
for the class of monotone $1$-Lipschitz link functions 
and any marginal distribution with second moment bounded by $\lambda$. 
Their algorithm achieves $L_2^2$ error
\begin{equation} \label{eqn:omni}
O(W \sqrt{\lambda}\sqrt{\opt}) +\eps
\end{equation}
under the assumption that the labels are bounded in $[0, 1]$. 
The error guarantee \eqref{eqn:omni} is substantially weaker --- both qualitatively and quantitatively --- from the goal of this paper. Firstly, the dependence 
on $\opt$ scales with its square root, as opposed to linearly. Secondly, and arguably  
more importantly, the multiplicative factor inside the big-O scales (linearly) with the 
diameter of the space $W$.

Interestingly, \cite{GGKS23} showed --- 
via a hardness construction from~\cite{DKMR22} ---  that, 
under their distributional assumptions, 
a multiplicative dependence on $W$ (in the error guarantee) 
is inherent for efficient algorithms. 
That is, to obtain an efficient constant factor approximation, 
it is necessary to restrict ourselves to distributions 
with \emph{additional} structural properties. 
This discussion raises the following question:
\begin{center}
{\em Can we obtain efficient {\em constant factor} learners for Problem~\ref{def:agnostic-learning} 
under mild distributional assumptions?}
\end{center}
The natural goal here is to match the guarantees of known algorithmic results 
for the special case of known link function~\cite{DKTZ22,WZDD2023}.

{\em As our main contribution, we answer this question in the affirmative.}
That is, we give {\em the first} efficient  constant-factor approximate learner 
that succeeds for natural and broad families of distributions 
(including log-concave distributions) and a {broad class of link functions}.
We emphasize that this is the first polynomial-time constant factor approximate learner
even for Gaussian marginals and for any nontrivial class of link functions. 
Roughly speaking, our distributional assumptions 
require concentration and (anti)-anti-concentration (see \Cref{def:bounds}).

\subsection{Overview of Results} 
We start by stating the distributional assumptions 
and defining the family of link functions for which our algorithm succeeds.

\paragraph{Distributional Assumptions} 
Our algorithm succeeds for the following class 
of structured distributions.

\begin{definition}[Well-Behaved Distributions] \label{def:bounds}
Let $L, R >0$. Let $V$ be any subspace in $\R^d$ of dimension at most $2$.
A
distribution $\D_{\bx}$ on $\R^d$
is called $(L,R)$-well-behaved if for any projection $(\D_{\bx})_V$ of $\D_{\bx}$
onto subspace $V$, the corresponding pdf $\gamma_V$ on $\R^2$
satisfies the following:
\begin{itemize}
\item   For all $\x_V \in V$ such that $\snorm{\infty}{\x_V} \leq R$, $\gamma_V(\x_V) \geq L$  (anti-anti-concentration). \item  For all $\x_V \in V$,  $\gamma_V(\x_V) \leq (1/L)(e^{-L \| \x_V \|_2 })$ (anti-concentration and concentration).
\end{itemize}
\end{definition}

As a consequence of sub-exponential concentration, we can assume without loss of 
generality that the operator norm of $\E_{\bx \sim \D_\bx}[\bx \bx^\top]$ is bounded above 
by an absolute constant. For simplicity, we take 
$\E_{\bx \sim \D_\bx}[\bx \bx^\top] \preccurlyeq \mathbf{I}$, 
which can be ensured by simple rescaling of the data. 

The distribution class of \Cref{def:bounds} was  introduced in \cite{DKTZ20}, 
in the context of learning linear separators with noise, and 
has since been used in a number of prior works --- 
including for robustly learning SIMs with known link function~\cite{DKTZ22}. 
{The parameters $L, R$ in \Cref{def:bounds} are viewed as universal constants, i.e., $L, R = O(1)$. Indeed, it is known that many natural distributions, 
most importantly isotropic log-concave distributions, fall in this category; 
see, e.g.,~\cite{DKTZ20}.}

\paragraph{Unbounded Activations} 
{Our algorithm succeeds for a broad class of link functions
that contains many well-studied activations, including ReLUs. 
This class, defined in \cite{DKTZ22} and used in~\cite{WZDD2023}, 
 requires the link function to be monotone, Lipschitz-continuous and strictly increasing 
in the positive region.}
\begin{definition}[Unbounded Activations]
\label{def:well-behaved-unbounded-intro}
 Let $u:\R\mapsto\R$.  
 Given $a, b \in \R$ such that $0< a \leq b$, we say that $u(z)$ 
 is $(a, b)$-unbounded if $u(0) = 0$ and $u(z)$ is non-decreasing, 
 $b$-Lipschitz-continuous, and $u(z) - u(z') \geq a(z - z')$ 
 for all {$z \geq z' \geq 0$}. We denote this function class by $\U$.
\end{definition}

A simplified version of our main algorithmic result is as follows 
(see \Cref{thm:fast-converge-main} for a more detailed statement):

\begin{theorem}[Main Algorithmic Result, Informal] \label{thm:main-inf}
Given \Cref{def:agnostic-learning}, where $\mathcal{G}$ is the class of $(L, R)$-well behaved distributions with $L, R = O(1)$ and $\mathcal{F} = \U$ 
such that $(1/a), b = O(1)$, there is an algorithm that draws $N = \poly(W) \tilde{O}(d/\eps^{2})$ samples 
from $\D$, 
runs in $\poly(N, d)$ time, and outputs 
a hypothesis $\hat{u}(\htw\cdot\x)$ with 
$\hat{u} \in \U, \|\htw\|_2 \leq W$ 
such that $\Ltwo(\htw;\hat{u}) = C \, \opt + \eps$ with high probability, 
where $C>0$ is an absolute constant.
\end{theorem}

{We reiterate that the approximation factor $C$ in 
\Cref{thm:main-inf} is a universal constant, independent 
of the dimension and the diameter of the space. That is, our main result provides the first efficient learning algorithm achieving a constant factor approximation, even for the most basic case of Gaussian data and any non-trivial class of link functions.}

\subsection{Technical Overview} \label{ssec:technical}

When it comes to learning SIMs in the agnostic model with target error 
$C \opt + \epsilon,$ to the best of our knowledge, all prior work 
that achieves such a guarantee with $C$ being an absolute constant 
only applies to the special case of known link function $u^*.$ 
Such results are established by proving growth conditions (local error bounds) 
that relate either the $L_2^2$ loss or a surrogate loss to (squared) distance 
to the set of target solutions, using assumptions about the link function 
and the data distribution, such as concentration and (anti-)anti-concentration \cite{DGKKS20, DKTZ22, WZDD2023}. Among these, most relevant to our work is \cite{WZDD2023}, which proved a ``sharpness'' property
for the convex surrogate function defined by
\begin{equation}\label{eq:cvx-surrogate}
    \Lsur(\w;u) = \Exy\bigg[\int_{0}^{\w\cdot\x}(u(r) - y) \diff{r}\bigg],
\end{equation}
based on certain assumptions about the link function (that are the same as ours) and 
 distributional assumptions (that are  somewhat weaker but comparable to ours). 
 Their sharpness result corresponds to guaranteeing that for vectors $\w$ that 
 are not already $O(\opt) + \eps$ accurate solutions, the following holds:
\begin{equation}\label{eq:prior-sharpness}
     \nabla\Lsur(\w;u^*)\cdot(\w - \wstar) \gtrsim \|\w - \wstar\|_2^2,
 \end{equation}
where $\wstar$ is a vector that achieves error $O(\opt) + \eps$ 
and $u^*$ is the (a priori known) link function.

 One may hope that the sharpness result of \cite{WZDD2023} can be generalized to the case of unknown link function and leveraged to obtain constant factor robust learners in this more general setting. However, as we discuss below, such direct generalizations are not possible and there are several technical challenges that had to be overcome in our work. To illustrate some of the intricacies, consider first the following example.

\begin{example}\label{ex:toy}
{\em Let $\x\sim\mathcal{N}(\vec 0,\vec I)$ and $\w  = (1/2)\w^*$, where $\w^*$ is an arbitrary but fixed target unit vector. Let $b>2a$. Suppose that the link function at hand is $u(z) = bz$ and the target link function is $u^*(z) = az$. Observe that both $u, u^* \in \U$, as required by our model. Furthermore, suppose there is no label noise, in which case $\opt = 0$. Note that the $L_2^2$ error of $u(\w\cdot\x)$ in this case is
\begin{align*}
    \Ltwo(\w;u) &= \E_{\x\sim\mathcal{N}(\vec 0,\vec I)}[(u(\w\cdot\x) - u^*(\w^*\cdot\x))^2]\\
    &= \E_{z\sim\mathcal{N}(0,1)}[(b/2 - a)^2z^2] = (b/2-a)^2 = \Theta(1).
\end{align*}
However, the gradient of the surrogate loss, $\nabla\Lsur(\w;u) = \E[(u(\w\cdot\x) - u^*(\w^*\cdot\x))\x]$, 
is negatively correlated with $\w - \w^*$, i.e., $\nabla\Lsur(\w;u)\cdot(\w - \w^*)<0$, contrary to what we would hope for if a sharpness property as in \cite{WZDD2023} were to hold. Thus, although $\w$ and $u$ are both still far away from the target parameters $\w^*$ and $u^*$, the gradient of the surrogate loss cannot provide useful information about the direction in which to update $\w$.} \end{example}

What \Cref{ex:toy} demonstrates is that we cannot hope for the surrogate loss to satisfy a local error bound for an {arbitrary} parameter pair $(u, \w)$ that would guide the convergence of an algorithm toward a target parameter pair $(u^*, \wstar).$ This seemingly insurmountable obstacle is surpassed by observing that we do not, in fact, need the surrogate loss to contain a ``signal'' that would guide us toward target parameters for an \emph{arbitrary} pair $(u, \w).$ Instead, we can restrict our attention to pairs $(u, \w)$ satisfying that $u$ is a ``reasonably good'' link function for the vector $\w.$ Ideally, we would like to only consider link functions $u$ that minimize the $L_2^2$ loss --- considering that $u^*$ must minimize the $L_2^2$ loss for a given, fixed $\w^*$ --- but it is unclear how to achieve that in a statistically and computationally efficient manner. As a natural approach, we consider link functions that are the best fitting functions in an empirical distribution sense. In particular, given a sample set $S = \{(\x\ith,y\ith)\}_{i=1}^m$ and a parameter $\w$, we select a function $\hat{u}_\w$ that solves the following (convex) optimization problem:\begin{equation}\label{def:htumt}\tag{P}
    \hat{u}_\w\in\argmin_{u\in\U}\frac{1}{m}\sum_{i=1}^m (u(\w\cdot\x\ith) - y\ith)^2.
\end{equation}
{For notational simplicity, we drop the parameter $\w^t$ from 
$\hat{u}_{\w^t}$ and use $\htumt$ instead.}  
It is worth pointing out here that in general the problem of finding the best 
function that minimizes the $L_2^2$ error fails under the category of non-parametric 
regression, which unfortunately requires exponentially many samples 
(namely, $\Omega(1/\eps^d$)). Fortunately, in our setting, we are looking for the 
best function that lies in a \emph{one-dimensional space}. Therefore, instead of 
looking at all possible directions, we can project all the points of the sample 
set $S$ to the direction $\vec w$ and find the best fitting link function 
efficiently. We provide the full details for efficiently solving 
the optimization problem \eqref{def:htumt}  in \Cref{app:isotonic-regression}. 

Having set on the ``best-fit'' link functions in the sense of the problem \eqref{def:htumt}, 
the next obstacle one encounters when trying to prove a  ``sharpness-like'' result is that 
\emph{neither the $L_2^2$ loss nor its surrogate convey information about the scale} of $\w$ 
and $\wstar.$ This is because models determined by $u, \w$ and $u/c, c\w$ for some 
parameter $c >0$ can have the same value of both loss functions. 
Thus, it seems unlikely that a more traditional local error bound, as in \eqref{eq:prior-sharpness}, can be established in general, for either the surrogate loss or the original $L_2^2$ loss. Instead, we prove a weaker property that establishes strong correlation between the gradient of the empirical surrogate loss $\nabla \htLsur(\w^t;\htumt)= (1/m)\sum_{i=1}^{m} (\hat{u}^t(\w^t\cdot\x\ith) - y\ith)\x\ith$ and the direction $\vec w^t-\wstar$ that holds whenever $\w^t$ is not an $O(\opt) +\eps$ error solution and which is \emph{independent of the scale} of $\w^t.$  This constitutes our key structural result, stated as \Cref{main:thm:sharpness} and discussed in detail in \Cref{sec:sharp}. We further discuss how this result relates to classical and recent local error bounds in \Cref{app:local-error-bounds}.

In addition to this weaker version of a sharpness property, we further prove in \Cref{main:cor:||htumt-u*(w*x)||<=||w-w*||} that given a parameter $\w^t$ and a dataset of $m$ samples from $\D$, the activation $\htumt(\w^t\cdot\x)$ generated by optimizing the empirical risk on the dataset as in \eqref{def:htumt} satisfies $\Ex[(\htumt(\w^t\cdot\x) - u^*(\w^*\cdot
\x))^2]\lesssim b^2\|\w^t -\w^*\|_2^2$ with high probability. As a result, we can guarantee that when $\|\w^t - \w^*\|_2$ decreases, the $L_2^2$ distance between $\htumt$ and $u^*$ diminishes as well. This is crucial, since without such a coupling we would not be able to argue about convergence over both model parameters $u, \w$.

Leveraging these results, we arrive at an algorithm that alternates between ``gradient 
descent-style'' updates for $\w$ and best-fit updates for $u$. We note in passing that 
similar alternating updates have been used in classical work on SIM learning in the less 
challenging, non-agnostic setting \cite{kakade2011efficient}. In more detail, our algorithm  
fixes the scale $\frkwstr$ of $\|\w^t\|_2$ and alternates between taking a Riemannian 
gradient descent step on a sphere for $\w^t$ with respect to the empirical surrogate loss and 
solving \eqref{def:htumt}. The unknown scale for the true parameter vector $\wstar$ is resolved by applying this 
approach using $\frkwstr$ chosen from a sufficiently fine grid of the interval $[0, W]$ and 
employing a testing procedure at the end to select the best parameter vector. Although the 
idea is simple, the proof of correctness is quite technical, as it requires ensuring that the entire 
process does not accumulate spurious errors arising from the stochastic nature of the 
problem, adversarial labels, and approximate minimization of the surrogate loss, and, as a 
result, that it converges to the target error.  
{
\paragraph{Technical Comparison to~\cite{GGKS23}} The only prior work addressing SIM 
learning (with unknown link functions) in the agnostic model is \cite{GGKS23}, thus here we 
provide a technical comparison. While both \cite{GGKS23} and our work make use of the 
surrogate loss function from \eqref{eq:cvx-surrogate}, on a technical level the two works 
are completely disjoint. \cite{GGKS23} uses a framework of omnipredictors to minimize the 
surrogate loss and then relates this result to the $L_2^2$ loss. Although they handle more 
general distributions and activations, their learner outputs a hypothesis with error that 
cannot be considered constant factor approximation (see \eqref{eqn:omni}) and is improper. 
By contrast, our work does not seek to minimize the surrogate loss. Instead, our main 
insight is that the gradient of the surrogate loss at a vector $\w$ conveys  information 
about the direction of a target vector $\wstar,$ for a \emph{fixed} link function that 
minimizes \emph{the $L_2^2$ loss}. We leverage this property to construct a proper learner 
achieving constant factor approximation.      
}

\section{Preliminaries}

\paragraph{Basic Notation}
For $n \in \Z_+$, let $[n] \eqdef \{1, \ldots, n\}$.  We use lowercase boldface characters for vectors. We use $\bx \cdot \by $ for the inner product 
of $\bx, \by \in \R^d$
and $ \theta(\bx, \by)$ for the angle between $\bx, \by$.
For $\bx \in \R^d$ and $k \in [d]$, $\x_k$ denotes the
$k^{\mathrm{th}}$ coordinate of $\bx$, and $\|\bx\|_2$ denotes the
$\ell_2$-norm of $\bx$.    We  use $\1_A = \1\{A\}$ to denote the
characteristic function of the set $A$.
For vectors $\vec v,\vec u\in \R^d$, we denote by  $\vec v^{\perp_{\vec u}}$ the projection of $\vec v$ onto the subspace orthogonal to $\vec u$, i.e., $\vec v^{\perp_{\vec u}} \eqdef \vec v - ((\vec v\cdot\vec u)\vec u)/\|\vec u\|_2^2$. 
We use $\B(r)$ to denote the $\ell_2$ ball in $\R^d$ of radius $r$,  centered at the origin.

\paragraph{Asymptotic Notation}
We use the standard $O(\cdot), \Theta(\cdot), \Omega(\cdot)$ asymptotic notation. We use
$\wt{O}(\cdot)$ to omit polylogarithmic factors in the 
argument. {We use $O_p(\cdot)$ to suppress 
polynomial dependence on $p$, i.e., $O_p(\omega) = O(\poly(p)\omega)$.} $\Theta_p(\cdot)$ and 
$\Omega_p(\cdot)$ are defined similarly. 
We write $E \gtrsim F$ for two non-negative
expressions $E$ and $F$ to denote that \emph{there exists} some positive universal constant $c > 0$
(independent of the variables or parameters on which $E$ and $F$ depend) such that $E \geq c \, F$.
The notation $\lesssim$ is defined similarly.  

\paragraph{Probability Notation}
We use $\E_{X\sim \D}[X]$ for the expectation of a random variable $X$ according to the
distribution $\D$ and $\pr[\mathcal{E}]$ for the probability of event $\mathcal{E}$. For simplicity
of notation, we  omit the distribution when it is clear from the context.  For $(\x,y)$
distributed according to $\D$, we use $\D_\x$ to denote the marginal distribution of $\x$.

\paragraph{Organization}
In \Cref{sec:sharp}, we establish our main structural result of alignment sharpness. In \Cref{sec:optimization}, we describe and analyze 
our constant factor approximate SIM learner. We conclude the paper in \Cref{sec:conclusion}. Some of the proofs and technical details are deferred to the Appendix.

\section{Main Structural Result: Alignment Sharpness of  Surrogate Loss}\label{sec:sharp}

In this section, we establish our main structural result (\Cref{main:thm:sharpness}), which 
is what crucially enables us to obtain the target $O(\opt) + \eps$ error for the studied 
problem. {\Cref{main:thm:sharpness} states that the empirical gradient of the 
surrogate loss \eqref{eq:cvx-surrogate} positively correlates with the direction 
of $\w^t - \w^*$ whenever $\w^t$ does not correspond to an $O(\opt)+\epsilon$ error solution; 
and, moreover, the correlation is proportional to the quantity 
$\|(\w^*)^{\perp_{\w^t}}\|_2^2$. 
This is a key property that is leveraged in our algorithmic result (\Cref{thm:fast-converge-main}), 
both in obtaining an $O(\opt) + \eps$ error result, and in arguing about the convergence 
and computational efficiency of our algorithm. 

Intuitively, what \Cref{main:thm:sharpness} allows us to argue 
is that as long as the angle between $\w^t$ and $\w^*$ 
is not close to zero, we can update $\w^t$  to better align it with $\wstar$  
(in the sense that we reduce the angle between these two vectors). 
} 
To understand this statement better,  note that when $\|\w^t\|_2 \approx \|\wstar\|_2$, 
we also have $\|(\w^*)^{\perp_{\w^t}}\|_2 \approx \|\w^t - \wstar\|_2$. Additionally, 
$\|\w^t - \wstar\|_2 = O(\opt + \eps)$ implies that the $L_2^2$ error of the hypothesis 
defined by $\hat{u}^t, \w^t$ is $O(\opt + \eps)$ (see \Cref{main:lem:L2-error-upbd-||w - w^*||^2}). 
Thus, for a sufficiently good guess of the value of $\|\wstar\|_2$, 
\Cref{main:thm:sharpness} provides a local error bound of the form 
$\nabla\htLsur(\w^t;\htumt)\cdot(\w^t - \w^*)\gtrsim \mu\|\w^t - \wstar\|_2^2$ 
that holds outside of the set of $O(\opt + \eps)$ error solutions, 
allowing us to contract the distance to this set. 

\begin{restatable}[{Alignment} Sharpness of the {Convex Surrogate}]{proposition}{sharpness}\label{main:thm:sharpness}
    Suppose that $\D_\x$ is $(L,R)$-well-behaved, $\U$ is as in \Cref{def:well-behaved-unbounded-intro}, and $\eps, \delta > 0.$ Let $\mu \gtrsim {a^2LR^4}/{b}$. 
Given any $\w^t\in\B(W)$, denote by $\htumt$ the optimal solution to \eqref{def:htumt} with respect to $\w^t$ and the sample set $S = \{(\x\ith,y\ith)\}_{i=1}^m$ drawn i.i.d.\ from $\D$. 
If $m$ satisfies 
    \begin{equation*}
        m \gtrsim dW^{9/2}b^4L^{-4}\log^4(d/(\eps\delta))(1/\eps^{3/2} + 1/(\eps\delta)) \;,
    \end{equation*}
    then,  with probability at least $1 - \delta$, \begin{equation*}
\nabla\htLsur(\w^t;\htumt)\cdot(\w^t - \w^*)\geq \mu\|(\w^*)^{\perp_{\w^t}}\|_2^2  - 2(\opt + \eps)/b - 2(\sqrt{\opt} + \sqrt{\eps})\|\w^t - \w^*\|_2\;.
    \end{equation*}
\end{restatable}

To prove \Cref{main:thm:sharpness}, we rely on the following key ingredients. 
In \Cref{subsec:misalignment-lemma}, we prove our main technical lemma 
(\Cref{main:lem:lower-bound-(f(wx)-u(w*x))^2}), which states that the $L_2^2$ distance 
between the hypothesis $u(\w\cdot\x)$ and the target $u^*(\w^*\cdot\x)$ is bounded below 
by the misalignment of $\w^t$ and $\w^*$, i.e., the squared norm of the component of $\w^*$ that is 
orthogonal to $\w^t$, $\|(\w^*)^{\perp_{\w^t}}\|_2^2$. As will become apparent in the proof of 
\Cref{main:thm:sharpness}, the inner product $\nabla\htLsur(\w^t;\htumt)\cdot(\w^t - \w^*)$ 
can be bounded below as a function of the empirical $L_2^2$ error for $\w^t$ and a different (but related) 
activation $\htustrmt$, which can in turn be argued to be close to the population $L_2^2$ error for a 
sufficiently large sample size, using concentration. Thus, \Cref{main:lem:lower-bound-(f(wx)-u(w*x))^2} can 
be leveraged to obtain a term scaling with $\|(\w^*)^{\perp_{\w^t}}\|_2^2$ in the lower bound on 
$\nabla\htLsur(\w^t;\htumt)\cdot(\w^t - \w^*)$. 

In \Cref{subsec:closeness-of-activations}, we characterize structural properties of the 
population-optimal link functions $u^t$ and $u^{*t}$ 
{(see (\ref{def:ut}) and (\ref{def:ut*}))},
which play a crucial role in the proof of 
\Cref{main:thm:sharpness}. Specifically, we show that the activation $u^t$ 
is close to the 
idealized activation $u^{*t}$ (the optimal activation without noise, 
given $\w^t$) in $L_2^2$ 
distance (\Cref{main:lem:upper-bound-u_t-u_t^*}). Since by standard uniform convergence 
results we have that $\htumt$ and $\htustrmt$ are close to their population counterparts $u^t$ and 
$u^{*t}$, respectively, \Cref{main:lem:upper-bound-u_t-u_t^*} certifies that $\htumt$ is not 
far from $\htustrmt$. This property enables us to replace $\htumt$ by (the idealized) 
$\htustrmt$ in the empirical surrogate gradient $\nabla\htLsur(\w^t;\htumt)$, which is easier 
to {analyze}, 
since $\htustrmt$ is defined with respect to the ``ideal'' dataset 
(with uncorrupted labels).

{Finally, as a simple corollary of \Cref{main:lem:upper-bound-u_t-u_t^*}, 
we obtain \Cref{main:cor:||htumt-u*(w*x)||<=||w-w*||}, which gives a clear explanation of why our algorithm, 
which alternates between updating $\w^t$ and $\htumt$, works: we show that the $L_2^2$ loss between the 
hypothesis generated by our algorithm $\htumt(\w^t\cdot\x)$ and the underlying optimal hypothesis 
$u^*(\w^*\cdot\x)$ is bounded above by the distance between $\w^t$ and $\w^*$. 
Since our structural sharpness result (\Cref{main:thm:sharpness}) enables us to 
decrease $\|\w^t - \w^*\|_2$, \Cref{main:cor:||htumt-u*(w*x)||<=||w-w*||} certifies that choosing the 
empirically-optimal activation leads to convergence of the hypothesis $\hat{u}^t(\w^t\cdot \x)$. }

Equipped with these technical lemmas, we prove our main structural result (\Cref{main:thm:sharpness}) in \Cref{subsec:proof-of-sharpness}.

\subsection{$L_2^2$ Error and Misalignment}\label{subsec:misalignment-lemma}

Our first key result is \Cref{main:lem:lower-bound-(f(wx)-u(w*x))^2} below, 
which plays a critical role in the proof of \Cref{main:thm:sharpness}. 
As discussed in \Cref{ssec:technical}, 
for two different activations $u$ and $u^*$ and parameters $\w$ and $\w^*$ 
such that $\w$ and $\w^*$ are parallel, even when the $L_2^2$ error is $\Omega(1),$ 
the gradient $\nabla\Lsur(\w;u)$ might not significantly align 
with the direction of $\w - \w^*$, and thus cannot provide sufficient information 
about the direction to decrease $\|\w - \w^*\|_2$. 
Intuitively, the following lemma shows that this is the only thing that can go wrong, and it happens when $\w$ and $\wstar$ are parallel. 
In particular, \Cref{main:lem:lower-bound-(f(wx)-u(w*x))^2} shows 
that for any {square integrable link function $f$}, 
we can relate the $L_2^2$ distance $\Ex[({f}(\w\cdot\x) - u^*(\w^*\cdot\x))^2]$ 
to the magnitude 
of the component of $\w^*$ that is \emph{orthogonal} to $\w$. Although its proof is quite technical, this lemma is the main supporting result allowing us to prove \Cref{main:thm:sharpness}, thus we provide its full proof below. It is however possible to follow the rest of this section by only relying on its statement.

\begin{restatable}[Lower Bound on $L_2^2$ Error by Misalignment]{lemma}{LowerBoundFwAndUwstr}\label{main:lem:lower-bound-(f(wx)-u(w*x))^2}
    Let $u^* \in \U$, $\D_\x$ be $(L,R)$-well-behaved, and 
    $f:\R\mapsto\R$ be square-integrable  
    with respect to the measure of the distribution $\D_\x$. 
    Then, for any $\vec w,\wstar\in \R^d$, 
    \[
    \Ex[(f(\vec w\cdot\x)-u^*(\wstar\cdot\x))^2]\gtrsim a^2 L R^4\|(\wstar)^{\perp_{\vec w}}\|_2^2\;.
    \]
\end{restatable}

\begin{proof}
    The statement holds trivially if $\w$ is parallel to $\wstar,$ so assume this is not the case.  Let $\vec v=(\wstar)^{\perp_{\vec w}} = \wstar - (\wstar \cdot {\vec w}) {\vec w}/\|\vec w\|_2^2$. {Suppose first that $\w\cdot\w^*\geq 0$.} Then $\wstar = \alpha\vec w+\vec v$, for some $\alpha>0$.
    Let $V$ be the subspace spanned by $\vec w,\vec v$. Then, \[ 
    \Ex[(f(\vec w\cdot\x)-u^*(\wstar\cdot\x))^2]=\Ex[(f(\vec w\cdot\x_V)-u^*(\wstar\cdot\x_V))^2]\geq \Ex[(f(\vec w\cdot\x_V)-u^*(\wstar\cdot\x_V))^2\1\{\x_V\in A\}]\;,
    \]
for any $A\subseteq \R^d$. For ease of notation, we drop the subscript $V$, and we assume that all $\x$ are projected to the subspace $V$.
We denote by $\tilde{\vec w} = \w/\|\w\|_2$ (resp. $\tv = \bv/\|\bv\|_2$) the unit vector in the direction of $\vec w$ (resp. $\vec v$). We choose $A=\{ \vec x \in \R^d: \vec w\cdot \x\geq 0,\tv\cdot \x\in (R/16,R/8)\cup(3R/8,R/2)\}$. 

The idea of the proof is to utilize the non-decreasing property of $u^*$ and the fact that the marginal 
distribution $\D_\x$ is anti-concentrated on the subspace $V$.
In short, for any {$\bx$ such that} $|\tv\cdot\x|\leq R$, by the non-decreasing property of $u^*$ 
we know that $f(\w\cdot\x)$ falls into one of the following four intervals: 
\begin{align*}
    & \left( -\infty,\; u^*(\alpha \vec w\cdot \x+\|\vec v\|_2R/32) \right], \quad && \left(u^*(\alpha \vec w\cdot \x+\|\vec v\|_2R/32),\; u^*(\alpha \vec w\cdot \x+\|\vec v\|_2R/4) \right],\\
    & \left( u^*(\alpha \vec w\cdot \x+\|\vec v\|_2R/4),\; u^*(\alpha \vec w\cdot \x+\|\vec v\|_2R) \right], \quad && \left( u^*(\alpha \vec w\cdot \x+\|\vec v\|_2R),\; +\infty \right) \;. 
\end{align*}
When $f(\w\cdot\x)$ belongs to any of the intervals above, we can show that with some constant probability, 
the difference between $\w^*\cdot\x$ and $\w\cdot\x$ is proportional to $\|\bv\|_2$, 
and hence $u^*(\w^*\cdot\x)$ is far from $f(\w\cdot\x)$ 
(due to the well-behaved property of the marginal $\D_\x$). 

To indicate that $f(\w\cdot\x)$ belongs to one of the intervals above, denote 
\begin{align*}
    I_1(\x) &= f(\vec w\cdot \x)-u^*(\alpha \vec w\cdot \x+\|\vec v\|_2R/32) \;,\\
I_2(\x) &= f(\vec w\cdot \x)-u^*(\alpha \vec w\cdot \x+\|\vec v\|_2R/4) \;,\\
I_3(\x) &= f(\vec w\cdot \x)-u^*(\alpha \vec w\cdot \x+\|\vec v\|_2R) \;. 
\end{align*}
For any $\x\in\R^d$, using the assumption that $u^*$ is 
non-decreasing, we have that $I_1(\x) \geq I_2(\x) \geq I_3(\x)$; 
as a consequence, it must be that 
$I_1(\x) I_2(\x)\geq 0$ or $I_2(\x) I_3(\x)\geq 0$. 

\begin{figure}[ht]
    \centering
    \includegraphics[width=0.8\textwidth]{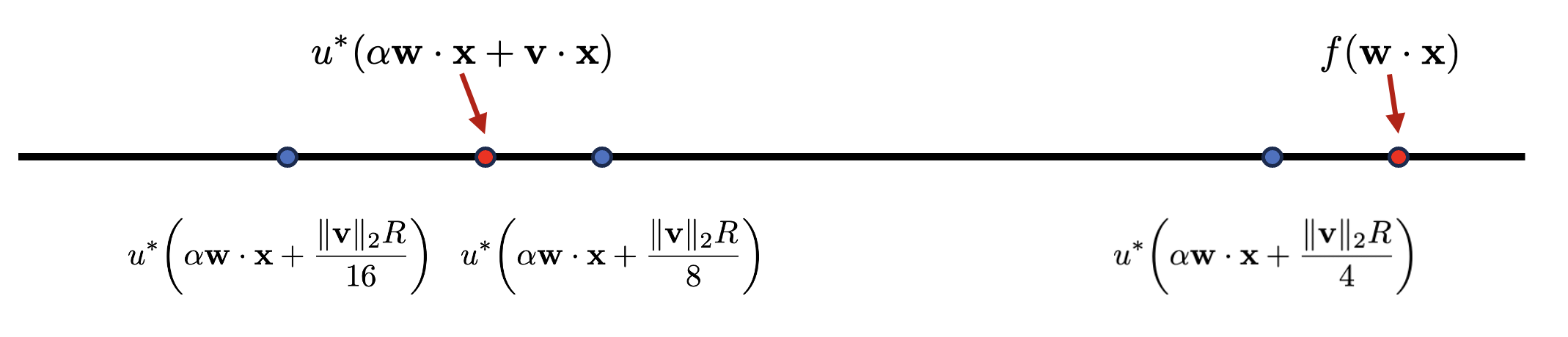}
    \caption{\textit{Under the assumption that $\tv\cdot\x\in(R/16,R/8)$, and $I_1(\x)\geq 0, I_2(\x)\geq 0$, the distance between $f(\w\cdot\x)$ and $u^*(\w^*\cdot\x)$ is at least $|u^*(\alpha \vec w\cdot \x+\|\vec v\|_2R/4)-u^*(\wstar\cdot\x)|\geq a\|\bv\|_2R/8$.}}
    \label{fig:I1I2-geq0}
\end{figure}

\vspace{5pt}

\noindent\textbf{Case 1:} $f(\w\cdot\x)\in (u^*(\alpha \vec w\cdot \x+\|\vec v\|_2R/4),\, \infty)$. Then $I_1(\x) \geq I_2(\x) \geq 0.$ 
Let $$B:=\{\x \in \R^d:  \vec w\cdot \x\geq 0,\tv\cdot\x\in (R/16,R/8)\}$$ and notice that $B\subseteq A$. 
We have that when $\x\in B$, 
$$
u^*(\w^*\cdot\x) = u^*(\alpha\w\cdot\x + \|\bv\|_2\tv\cdot\x)\in(u^*(\alpha\w\cdot\x + \|\bv\|_2R/16),\; u^*(\alpha\w\cdot\x + \|\bv\|_2R/8)),
$$
thus we can conclude that
\begin{align*}
     &\quad (f(\vec w\cdot\x)-u^*(\wstar\cdot\x))^2\1\{\x\in B\}\\
     &= \big(\{f(\vec w\cdot\x)-u^*(\alpha \vec w\cdot \x+\|\vec v\|_2R/4)\} + \{ u^*(\alpha \vec w\cdot \x+\|\vec v\|_2R/4)-u^*(\wstar\cdot\x)\}\big)^2\1\{\x\in B\}
     \\&\geq (u^*(\alpha \vec w\cdot \x+\|\vec v\|_2R/4)-u^*(\wstar\cdot\x))^2\1\{\x\in B\}\;,
\end{align*}
where in the last inequality we used that 
$I_2(\x)=f(\vec w\cdot\x)-u^*(\alpha \vec w\cdot \x+\|\vec v\|_2R/4)\geq 0$ 
and $u^*(\alpha \vec w\cdot \x+\|\vec v\|_2R/4)-u^*(\wstar\cdot\x)\geq 0$ 
by the non-decreasing property of $u^*$, and the elementary inequality 
$(a+b)^2\geq \max(a,b)^2$ for $a,b\geq 0$. 
Further, using $u^*(t) - u^*(t')\geq a(t-t')$ for $t\geq t'\geq 0$ 
(which holds by assumption) and $\wstar = \alpha \vec w + \vec v$, 
we have 
\begin{align*}
  (u^*(\alpha \vec w\cdot \x+\|\vec v\|_2R/4)-u^*(\wstar\cdot\x))^2\1\{\x\in B\}&\geq  a^2(\|\vec v\|_2R/4-\vec v\cdot \x)^2\1\{\x\in B\}
  \\&\geq a^2\|\vec v\|_2^2(R/8)^2\1\{\x\in B\}\;,
\end{align*}
where in the last inequality we used that $0\leq \tv\cdot\x\leq R/8$ 
(by the definition of the event $B$). 
A visual {illustration of the argument} above is given 
in \Cref{fig:I1I2-geq0}.

\vspace{5pt}
\noindent\textbf{Case 2:} $f(\w\cdot\x)\in(-\infty,\, u^*(\alpha \vec w\cdot \x+\|\vec v\|_2R/32))$. Then $0 \geq I_1(\x) \geq I_2(\x)$. We follow a similar argument as in the previous case. In particular, we begin with 
\begin{align}
    &\quad (f(\w\cdot\x) - u^*(\w^*\cdot\x))^2\1\{\x\in B\}\notag\\
    & = \big(\{f(\w\cdot\x) - u^*(\alpha\w\cdot\x + \|\bv\|_2R/32)\} + \{u^*(\alpha\w\cdot\x + \|\bv\|_2R/32) - u^*(\w^*\cdot\x)\}\big)^2\1\{\x\in B\}.\label{eq:case-2-aux-I_1+compl}
\end{align}
Note that $I_1(\x) \leq 0$ and 
$u^*(\w^*\cdot\x) = u^*(\alpha\w\cdot\x + \|\bv\|_2\tv\cdot\x)\geq u^*(\alpha\w\cdot\x + \|\bv\|_2R/32)$ 
since $\tv\cdot\x\geq R/16\geq R/32$ for $\x\in B$; 
thus, the two terms in curly brackets in \eqref{eq:case-2-aux-I_1+compl} 
have the same sign and we further have:
\begin{align*}
    (f(\w\cdot\x) - u^*(\w^*\cdot\x))^2\1\{\x\in B\}&\geq (u^*(\alpha\w\cdot\x + \|\bv\|_2R/32) - u^*(\w^*\cdot\x))^2\1\{\x\in B\}\\
    &\geq a^2\|\bv\|_2^2 (R/32)^2\1\{\x\in B\} \;,
\end{align*}
where in the first inequality we used the fact that $(a+b)^2\geq \max\{a^2, b^2\}$ when both $a,b\leq 0$. 

By the analysis of {Case 1 and Case 2}, we can conclude that when $I_1(\x)I_2(\x)\geq 0$, it must be:
\begin{equation}\label{ineq:I_1I_2>=0}
    (f(\w\cdot\x) - u^*(\w^*\cdot\x))^2\1\{\x\in B\}\geq a^2\|\bv\|_2^2R^2/2^{10}\1\{\x\in B\} \;.
\end{equation}

\vspace{5pt}
\noindent\textbf{Case 3:} $f(\w\cdot\x)\in (u^*(\alpha\w\cdot\x + \|\bv\|_2R),\, +\infty)$. Then $I_2(\x) \geq I_3(\x)\geq 0$ and we choose 
$$B'=\{\x \in \R^d: \vec w\cdot \x\geq 0,\tv\cdot\x \in (3R/8,R/2)\} \;.$$ 
Following the same reasoning as in the previous two cases, we have
\begin{align*}
    &\quad (f(\vec w\cdot\x)-u^*(\wstar\cdot\x))^2\1\{\x\in B'\}\\
    & = \big(\{f(\vec w\cdot\x) - u^*(\alpha\w\cdot\x + \|\bv\|_2R)\} + \{u^*(\alpha\w\cdot\x + \|\bv\|_2R) - u^*(\wstar\cdot\x)\}\big)^2\1\{\x\in B'\}\\
    &\geq (u^*(\alpha\w\cdot\x + \|\bv\|_2R) - u^*(\wstar\cdot\x))^2\1\{\x\in B'\}\\
    &\geq a^2\|\bv\|_2^2 (R/2)^2 \1\{\x\in B'\} \;.
\end{align*}

\vspace{5pt}
\noindent\textbf{Case 4:} $f(\w\cdot\x)\in (-\infty, u^*(\alpha\w\cdot\x + \|\bv\|_2R/4))$. Then $ 0 \geq I_2(\x) \geq I_3(\x).$ It follows that
\begin{align*}
    &\quad (f(\vec w\cdot\x)-u(\wstar\cdot\x))^2\1\{\x\in B'\}\\
    & = \big(\{f(\vec w\cdot\x)- u^*(\alpha\w\cdot\x + \|\bv\|_2R/4)\}  + \{u^*(\alpha\w\cdot\x + \|\bv\|_2R/4) - u(\wstar\cdot\x)\}\big)^2\1\{\x\in B'\}\\
    &\geq a^2\|\bv\|_2^2 (R/8)^2 \1\{\x\in B'\} \;.
\end{align*}
Thus, from the analysis of {Case 3 and Case 4}, we conclude that when $I_2(\x)I_3(\x)\geq 0$, we have
\begin{equation}\label{ineq:I_2I_3>=0}
    (f(\w\cdot\x) - u^*(\w^*\cdot\x))^2\1\{\x\in B'\}\geq a^2\|\bv\|_2^2(R^2/64)\1\{\x\in B'\} \;.
\end{equation}
Recall that for any $\x$, at least one of the inequalities 
$I_1(\x)I_2(\x)\geq 0$ or $I_2(\x)I_3(\x)\geq 0$ happens, thus, 
$\1\{I_1(\x)I_2(\x)\geq 0\}\geq 1 - \1\{I_2(\x)I_3(\x)\geq 0\}$. Therefore, the 
probability mass of the region 
$$(B\cap\{I_1(\x)I_2(\x)\geq 0\}) \cup (B'\cap\{I_2(\x)I_3(\x)\geq 0\})$$ can 
be bounded below by:
\begin{align}\label{ineq:lower-bound-prob-mass}
    &\quad \pr\bigg[\x\in (B\cap\{I_1(\x)I_2(\x)\geq 0\}) \cup (B'\cap\{I_2(\x)I_3(\x)\geq 0\})\bigg] \nonumber\\
    &= \int_V \bigg(\1\{\x\in B\}\1\{I_1(\x)I_2(\x)\geq 0\} + \1\{\x\in B'\}\1\{I_2(\x)I_3(\x)\geq 0\}\bigg)\gamma(\x)\diff{\x}\nonumber\\
    &\geq \int_{V, \|\x\|_\infty\leq R} \bigg(\1\{\x\in B\}\1\{I_1(\x)I_2(\x)\geq 0\} + \1\{\x\in B'\}\1\{I_2(\x)I_3(\x)\geq 0\}\bigg)L\diff{\x}\nonumber\\
    &\geq L\int_{V, \|\x\|_\infty\leq R} \bigg(\1\{\x\in B\} + (\1\{\x\in B'\} - \1\{\x\in B\}) \1\{I_2(\x)I_3(\x)\geq 0\}\bigg)\diff{\x} \;,
\end{align}
where in the first inequality we used the assumption that $\D_\x$ is $(L,R)$-
well-behaved. As a visual illustration of the lower bound argument 
above, {the reader is referred to} \Cref{fig:I1I2I3}.

\begin{figure}[h]
    \centering
    \includegraphics[width=0.8\linewidth]{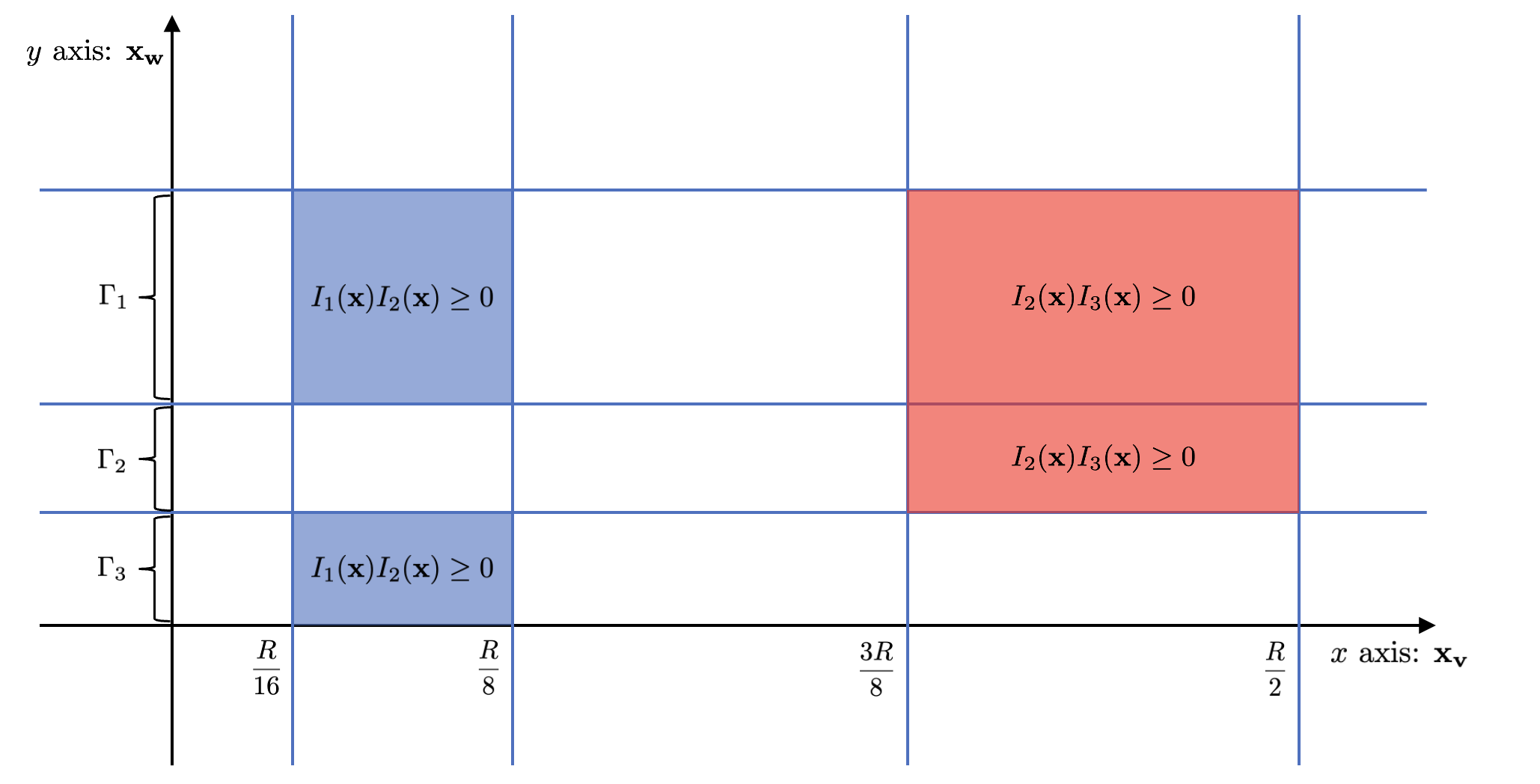}
    \caption{\textit{On the 2-dimensional space $V$ spanned by $(\x_{\bv},\x_{\w})$, at each point $\x\in B\cup B'$, it must be that $I_1(\x)I_2(\x)\geq 0$ or $I_2(\x)I_3(\x)\geq 0$. $\Gamma_1$ denotes the interval of $\x_\w = \w\cdot\x$ such that $f(\w\cdot\x)\geq u^*(\alpha\w\cdot\x + \|\bv\|_2R)$, hence both $I_1(\x)I_2(\x)\geq 0,\, I_2(\x)I_3(\x)\geq 0$; $\Gamma_2$ denotes the interval of $\x_\w$ such that $f(\w\cdot\x)\in (u^*(\alpha\w\cdot\x + \|\bv\|_2R/32), u^*(\alpha\w\cdot\x + \|\bv\|_2R/4))$, hence $I_2(\x)I_3(\x)\geq 0$; finally, $\Gamma_3$ denotes the interval of $\x_\w$ such that $f(\w\cdot\x)\in (u^*(\alpha\w\cdot\x + \|\bv\|_2R/4), u^*(\alpha\w\cdot\x + \|\bv\|_2R/))$, hence $I_1(\x)I_2(\x)\geq 0$. The area of the union of the red and blue regions is the lower bound on the probability in \eqref{ineq:lower-bound-prob-mass}. As displayed in the figure, the sum of the blue and red region is lower bounded by $\1\{\x\in B\} + (\1\{\x\in B'\} - \1\{\x\in B\}) \1\{I_2(\x)I_3(\x)\geq 0\}$.}}
    \label{fig:I1I2I3}
\end{figure}
\vspace{5pt}

To finish bounding below the probability in \eqref{ineq:lower-bound-prob-mass}, it 
remains to bound the integral from its final inequality, which now does not 
involve the probability density function anymore, as we used the anti-
concentration property of $\D_\x$ to uniformly bound below $\gamma(\x).$  
Recall that by definition, $I_1(\x),I_2(\x),I_3(\x)$ are  functions of $\w\cdot\x$ that do not depend on $\tv\cdot\x$. Denote the projection of $\x$ on the standard basis of space $V$ by $\x_{\tilde{\w}} = \tilde{\w}\cdot\x$ and $\x_{\tilde{\bv}} = \tv\cdot\x$. Then, we have:
\begin{align*}
    &\quad \int_{V, \|\x\|_\infty\leq R}\bigg(\1\{\x\in B'\} - \1\{\x\in B\}\bigg) \1\{I_2(\x)I_3(\x)\geq 0\}\diff{\x}\\
    & = \int_{|\x_{\tilde{\w}}|\leq R}\int_{|\x_{\tilde{\bv}}|\leq R} \bigg(\1\bigg\{\x_{\tilde{\bv}}\in \bigg(\frac{3R}{8}, \frac{R}{2}\bigg)\bigg\} - \1\bigg\{\x_{\tilde{\bv}}\in\bigg(\frac{R}{16}, \frac{R}{8}\bigg)\bigg\}\bigg)\diff{\x_{\tilde{\bv}}}\1\{\x_{\tilde{\w}}\geq 0, I_2(\x)I_3(\x)\geq 0\}\diff{\x_{\tilde{\w}}}\\
    &= \int_{|\x_{\tilde{\w}}|\leq R}\1\{\x_{\tilde{\w}}\geq 0, I_2(\x)I_3(\x)\geq 0\}\diff{\x_{\tilde{\w}}} \int_{|\x_{\tilde{\bv}}|\leq R} \bigg(\1\bigg\{\x_{\tilde{\bv}}\in \bigg(\frac{3R}{8}, \frac{R}{2}\bigg)\bigg\} - \1\bigg\{\x_{\tilde{\bv}}\in\bigg(\frac{R}{16}, \frac{R}{8}\bigg)\bigg\}\bigg)\diff{\x_{\tilde{\bv}}}
    \\
    &\geq 0 \;.
\end{align*}
Plugging the inequality above back into \eqref{ineq:lower-bound-prob-mass}, we get:
\begin{align}\label{ineq:lower-bound-prob-mass-2}
    &\quad \pr\bigg[\x\in \big(B\cap\{I_1(\x)I_2(\x)\geq 0\}\big) \cup \big(B'\cap\{I_2(\x)I_3(\x)\geq 0\}\big)\bigg]\nonumber\\
    &\geq L\int_{V,\|\x\|_\infty\leq R} \1\{\w\cdot\x\geq 0, \tilde{\bv}\cdot\x\in (R/16, R/8) \}\diff{\x}\nonumber\\
    &=L\iint(\1\{\x_{\tilde{\w}}\in(0, R)\}\diff{\x_{\tilde{\w}}})\1\{\x_{\tilde{\bv}}\in(R/16,R/8)\}\diff{\x_{\tilde{\bv}}} = LR^2/16 \;.
 \end{align}
We are now ready to provide a lower bound on the $L_2^2$ distance between $f(\w\cdot\x)$ and $u^*(\w^*\cdot\x)$. Combining the inequalities from \eqref{ineq:I_1I_2>=0} and \eqref{ineq:I_2I_3>=0}, we get
\begin{align*}
     &\quad \Ex[(f(\vec w\cdot\x)-u^*(\wstar\cdot\x))^2]\\
     &\geq \Ex[(f(\vec w\cdot\x_V)-u^*(\wstar\cdot\x_V))^2\1\{\x_V\in A\}]
     \\&\geq a^2(R^2/1024)\|\vec v\|_2^2\Ex[\1\big\{\{\x_V\in B\cap\{I_1(\x)I_2(\x)\geq 0\} \}\cup \{B'\cap\{I_2(\x)I_3(\x)\geq 0\}\}\big\}]\\
     &\geq a^2(R^4/2^{13})L\|\vec v\|_2^2\;,
\end{align*}
  where we used \eqref{ineq:lower-bound-prob-mass-2} in the last inequality.

  {Now for the case where $\w\cdot\w^*\leq 0$, it holds $\w^* = \alpha\w + \bv$ with $\alpha\leq 0$. Considering instead 
  $A = \{\x \in \R^d: \w\cdot\x\leq 0, \tilde{\bv}\cdot\x\in(R/16,R/8)\cup(3R/8,R/2)\}$ and similarly 
  $B = \{\x \in \R^d: \w\cdot\x\leq 0, \tilde{\bv}\cdot\x\in(R/16,R/8)\}$, $B' = \{\x \in \R^d: \w\cdot\x\leq 0, \tilde{\bv}\cdot\x\in(R/3,R/2)\}$, 
  then all the steps above remains valid without modification.}
This completes the proof of \Cref{main:lem:lower-bound-(f(wx)-u(w*x))^2}.
\end{proof}

\subsection{Closeness of Idealized and Attainable Activations}\label{subsec:closeness-of-activations}

In this section, we bound the contribution of the error incurred from working with attainable link functions  $\htumt$ in the iterations of the algorithm. 
The error incurred is due to both the arbitrary noise in the labels and 
due to using a finite sample set. In bounding the error, for analysis purposes, we introduce auxiliary population-level link functions. 

 Concretely, given  $\w\in\B(W)$, a population-optimal activation is a solution to the following stochastic  convex program: \begin{equation}\label{def:ut}\tag{EP}
\uw\in\argmin_{u\in\U}\Exy[(u(\w\cdot\x) - y)^2].
\end{equation}
 We further introduce auxiliary ``idealized, noiseless'' activations, which, given noiseless labels $y^* = u^*(\w^*\cdot\x)$ and a parameter weight vector $\w,$ are defined via 
\begin{equation}\label{def:ut*}\tag{EP*}
    \ustrw\in\argmin_{u\in\U}\Exy[(u(\w\cdot\x) - y^*)^2].
\end{equation}

Below we relate $u^t := u_{\w^t}$ and $u^{*t} := u^*_{\w_t}$ and show that their $L_2^2$ error for the parameter vector $\w^t$ is bounded by $\opt$. The proof of \Cref{main:lem:upper-bound-u_t-u_t^*} is deferred to \Cref{app:pf:main:lem:upper-bound-u_t-u_t^*}.

\begin{restatable}[Closeness of Population-Optimal Activations]{lemma}{UpperBoundUtUtstr}\label{main:lem:upper-bound-u_t-u_t^*}
    {Let $\w^t \in \B(W)$} and let $u^{*t}$, $u^t$ be defined as solutions to  \eqref{def:ut*}, \eqref{def:ut}, respectively. Then, $$\Ex[(u^t(\w^t\cdot\x) - u^{*t}(\w^t\cdot\x))^2]\leq \opt.$$ 
\end{restatable}

As a consequence of the lemma above, we are able to relate $\htumt$ to the ``noiseless'' labels $y^* = u^*(\w^*\cdot\x)$ by showing that the $L_2^2$ distance between $u^*(\w^*\cdot\x)$ and the sample-optimal activation $\htumt(\w^t\cdot\x)$ is bounded by $\|\w^t - \w^*\|_2^2$. Although \Cref{main:cor:||htumt-u*(w*x)||<=||w-w*||} is not used in the proof of \Cref{main:thm:sharpness}, we still present it here as it justifies the mechanism of our approach alternating between updates for $\w^t$ and $\htumt$. The proof of \Cref{main:cor:||htumt-u*(w*x)||<=||w-w*||} can be found in \Cref{app:pf:main:cor:||htumt-u*(w*x)||<=||w-w*||}.

\begin{restatable}[Closeness of Idealized and Attainable Activations]{corollary}{UpperBoundLtwoDistUtUstr}\label{main:cor:||htumt-u*(w*x)||<=||w-w*||}
    Let $\eps, \delta > 0.$ Given a parameter $\w^t\in\B(W)$ and $m \gtrsim d\log^4(d/(\eps\delta))(b^2W^3/(L^2\eps))^{3/2}$ samples from $\D$, let $\htumt$ be the sample-optimal activation on these samples given $\w^t$, as defined in \eqref{def:htumt}. Then, with probability at least $1 - \delta$, 
\begin{equation*}
     \Ex [(\htumt(\w^t\cdot\x) - u^*(\w^*\cdot\x))^2]  
    \leq 3(\eps + \opt + b^2\|\w^t - \w^*\|_2^2) \;.
    \end{equation*}
\end{restatable}

\subsection{Proof of \Cref{main:thm:sharpness}}\label{subsec:proof-of-sharpness}

We are now ready to prove our main structural result. We focus here on the main 
argument, while the proofs of supporting technical claims are deferred to \Cref{app:pf:sharpness}. 

\begin{proof}[Proof of \Cref{main:thm:sharpness}]
    Given any weight parameter $\w^t \in \B(W)$ and $\htumt$ chosen as its corresponding sample-optimal solution to problem \eqref{def:htumt}, let $u^t$ be the population-optimal activation, as defined by Problem \eqref{def:ut}. 
    Given a sample set $S = \{(\x\ith,y\ith)\}_{i=1}^m$, consider an idealized, ``noise-free'' set {$S^*$} that assigns realizable labels to data vectors from $S$, i.e., $S^* = \{(\x\ith,y\sith)\}_{i=1}^m$, $y\sith = u^*(\w^*\cdot\x\ith)$.
    Further define idealized sample-optimal activations by 
\begin{equation}\tag{P*}\label{def:htustrmt}\htustrw \in\argmin_{u\in\U}\frac{1}{m}\sum_{i=1}^m(u(\w\cdot\x\ith) - y\sith)^2.
\end{equation}
For a parameter $\w^t$, denote $\htustrmt:=\hat{u}^*_{\w^t}$, for simplicity,
and recall that the population version of $\htustrmt$ was defined by \eqref{def:ut*}.
To prove \Cref{main:thm:sharpness}, we decompose $\nabla \htLsur(\w^t;\htumt)\cdot(\w^t - \w^*)$ into three summation terms: 
    \begin{align}
        &\quad \nabla \htLsur(\w^t;\htumt)\cdot(\w^t - \w^*)\nonumber\\
        & = \frac{1}{m}\sum_{i=1}^m (\htumt(\w^t\cdot\x\ith) - y\ith)(\w^t - \w^*)\cdot\x\ith\nonumber\\
        & = \underbrace{\frac{1}{m}\sum_{i=1}^m(\htumt(\w^t\cdot\x\ith - \htustrmt(\w^t\cdot\x\ith))(\w^t - \w^*)\cdot\x\ith}_{Q_1} + \underbrace{\frac{1}{m}\sum_{i=1}^m(\htustrmt(\w^t\cdot\x\ith) - y\sith)(\w^t - \w^*)\cdot\x\ith}_{Q_2} \nonumber\\
        &\quad  + \underbrace{\frac{1}{m}\sum_{i=1}^m(y\sith - y\ith)(\w^t\cdot\x\ith - \w^*\cdot\x\ith)}_{Q_3} \;. \label{eq:sharp-inter-1}
    \end{align}
    We tackle each term $Q_1$ to $Q_3$ in \eqref{eq:sharp-inter-1} separately, using the following arguments relying on three auxiliary claims. Because the proofs of these claims are technical, we defer them to \Cref{app:pf:sharpness}.
        
    The first claim states that $Q_1$ is of the order $(\sqrt{\eps} + \sqrt{\opt})\|\w^t - \w^*\|_2 + (\opt+\eps)/b$ with high probability.

    \begin{restatable}{claim}{claimsharpone}\label{claim:sharp-inter-1}
        Let $S  = \{(\x\ith,y\ith)\}_{i=1}^m$ be i.i.d.\ samples from $\D$ where $m$ is as specified in the statement of \Cref{main:thm:sharpness}.
Let $\htumt$ be the solution of optimization problem \eqref{def:htumt} given {$\w^t\in\B(W)$} and $S$. Furthermore, denote the idealized version of $S$ by $S^* = \{(\x\ith,y\sith)\}_{i=1}^m$, where $y\sith = u^*(\w^*\cdot\x\ith)$. Let $\htustrmt$ be the solution of problem \eqref{def:htustrmt}. Then, with probability at least $1 - \delta$, 
        \begin{equation*}
        Q_1 = \frac{1}{m}\sum_{i=1}^m((\htumt(\w^t\cdot\x\ith) - \htustrmt(\w^t\cdot\x\ith))(\w^t - \w^*)\cdot\x\ith\geq -(\sqrt{\eps} + \sqrt{\opt})\|\w^t - \w^*\|_2 - (\eps + \opt)/b \;.
    \end{equation*}
    \end{restatable}

    The proof of \Cref{claim:sharp-inter-1} is based on the following argument: first, standard concentration arguments ensure that  $\htumt$ and $\htustrmt$ are close to their population counterparts, $u^t$ and $u^{*t}$, in $L_2^2$ distance (see \Cref{app:uniform-convergence}). Therefore, applying Chebyshev's inequality, we are able to swap the sample-optimal activations in \eqref{eq:sharp-inter-1} by their population-optimal counterparts with high probability and focus on bounding 
    \begin{equation*}
         \frac{1}{m}\sum_{i=1}^m (u^t(\w^t\cdot\x\ith) - u^{*t}(\w^t\cdot\x\ith))(\w^t - \w^*)\cdot\x\ith \;.
    \end{equation*}
To bound this quantity, we leverage the result from \Cref{main:lem:upper-bound-u_t-u_t^*}, namely that $\Ex[(u^t(\w^t\cdot\x) - u^{*t}(\w^t\cdot\x))^2]\leq \opt$.

    The second claim leverages the misalignment lemma (\Cref{main:lem:lower-bound-(f(wx)-u(w*x))^2}) and shows that, up to small errors, $Q_2$ is a constant multiple of $\|(\w^*)^{\perp_{\w^t}}\|_2^2$.

    \begin{restatable}{claim}{claimsharptwo}\label{claim:sharp-inter-2}
    Let $S^* = \{(\x\ith,y\sith)\}_{i=1}^m$ be a sample set such that $\x\ith$'s are i.i.d.\ samples from $\D_\x$ and $y\sith = u^*(\w^*\cdot\x\ith)$ for each $i$. Let $m$ be the value specified in the statement of \Cref{main:thm:sharpness}.
    Then, given a parameter $\w^t\in\B(W)$, with probability at least $1 - \delta$,
\begin{equation*}
            Q_2 = \frac{1}{m} \sum_{i=1}^m (\htustrmt(\w^t\cdot\x\ith) - y\sith)(\w^t - \w^*)\cdot\x\ith \geq \frac{Ca^2LR^4}{b}\|(\w^*)^{\perp_{\w^t}}\|_2^2 - \sqrt{\eps}\|\w^t - \w^*\|_2 - \eps/b\; ,
        \end{equation*}
    where $C$ is an absolute constant.
    \end{restatable}

    The proof of \Cref{claim:sharp-inter-2} is rather technical. We first define an `empirical inverse' of the activation $u^*$, and denote it by $\hat{f}$. 
    Note that $u^*(z)\in\U$ is not necessarily strictly increasing when $z\leq 0$, therefore $(u^*)^{-1}$ is not defined everywhere on $\R$, and the introduction of this `empirical inverse' function $\hat{f}$ is needed. 
    Then, adding and subtracting $\hat{f}(\htustrmt(\w^t\cdot\x\ith))$ in the $\w^t\cdot\x\ith - \w^*\cdot\x\ith$ term, we get
    \begin{align*}
        &\quad \frac{1}{m} \sum_{i=1}^m (\htustrmt(\w^t\cdot\x\ith) - u^*(\w^*\cdot\x\ith))(\w^t - \w^*)\cdot\x\ith\nonumber\\
        & = \frac{1}{m} \sum_{i=1}^m (\htustrmt(\w^t\cdot\x\ith) - u^*(\w^*\cdot\x\ith))(\w^t\cdot\x\ith - \hat{f}(\htustrmt(\w^t\cdot\x\ith))) \nonumber\\
        & \quad + \frac{1}{m} \sum_{i=1}^m (\htustrmt(\w^t\cdot\x\ith) - u^*(\w^*\cdot\x\ith))(\hat{f}(\htustrmt(\w^t\cdot\x\ith)) - \w^*\cdot\x\ith) \;.
    \end{align*}
    Analyzing the KKT conditions of the optimization problem \eqref{def:htustrmt}, we argue that the first term in the equation above is always positive. 
    Then, we argue that our definition of the empirical inverse $\hat{f}$ ensures that the second term can be bounded below by $\frac{1}{bm}\sum_{i=1}^m (\htustrmt(\w^t\cdot\x\ith) - u^*(\w^*\cdot\x\ith))^2$. Using standard concentration arguments, the quantity above concentrates around its expectation $\Ex[(\htustrmt(\w^t\cdot\x) - u^*(\w^*\cdot\x))^2]$, 
    hence we complete the proof applying \Cref{main:lem:lower-bound-(f(wx)-u(w*x))^2}.
    
    Similar to \Cref{claim:sharp-inter-1}, the last claim shows that $Q_3$ is of the order $\sqrt{\opt}\|\w^* - \w^t\|_2$, which is  small compared to the positive term in \Cref{claim:sharp-inter-2} outside the set of $O(\opt) + \eps$ {error} solutions.
    
    \begin{restatable}{claim}{claimsharpthree}\label{claim:sharp-inter-3}
        Let $S  = \{(\x\ith,y\ith)\}_{i=1}^m$ be i.i.d.\ samples from $\D$, and
denote by $S^* = \{(\x\ith,y\sith)\}_{i=1}^m$ the idealized version of $S$, where $y\sith = u^*(\w^*\cdot\x\ith)$. Under the condition of \Cref{main:thm:sharpness}, given a parameter $\w^t\in\B(W)$, with probability at least $1 - \delta$, 
        \begin{equation*}
            Q_3 = \frac{1}{m}\sum_{i=1}^m (y\sith- y\ith)(\w^t\cdot\x\ith - \w^*\cdot\x\ith)\geq - \sqrt{\opt}\|\w^* - \w^t\|_2 - (\opt + \eps)/b \; .
        \end{equation*}
    \end{restatable}

    The proof of \Cref{claim:sharp-inter-3} follows via similar arguments as the proof of \Cref{claim:sharp-inter-1}.

    Plugging the bounds from \Cref{claim:sharp-inter-1}, \Cref{claim:sharp-inter-2}, and \Cref{claim:sharp-inter-3} back into \eqref{eq:sharp-inter-1} and using a union bound, we get that with probability at least $1 - 3\delta$,
    \begin{equation*}
        \nabla\htLsur(\w^t;\htumt)\cdot(\w^t - \w^*)\geq \frac{Ca^2LR^4}{b}\|(\w^*)^{\perp_{\w^t}}\|_2^2 - 2(\sqrt{\opt} + \sqrt{\eps})\|\w^t - \w^*\|_2 - 2(\opt + \eps)/b,
    \end{equation*}    
    for some absolute constant $C$, completing the proof.
\end{proof}
 
\section{Robust SIM Learning via Alignment Sharpness}\label{sec:optimization}

As discussed in \Cref{ssec:technical}, our algorithm can be viewed 
as employing an alternating procedure: 
taking a Riemannian gradient descent step on a sphere with respect 
to the empirical surrogate, given an estimate of the activation, 
and optimizing the activation function on the sample set 
for a given parameter weight vector. This procedure is performed 
using a fine grid of guesses of the scale of $\|\wstar\|_2.$ 
For this process to converge with the desired linear rate (even for a known value of $\|\wstar\|_2$), the algorithm needs to be properly initialized 
to ensure that the initial weight vector has a nontrivial alignment 
with the optimal vector $\wstar$. The initialization process is handled
in the following subsection.

\subsection{Initialization}

We begin by showing that the Initialization subroutine stated in \Cref{alg:initialization} 
returns a point $\barw^0$ that has a sufficient alignment with $\wstar.$ As will become apparent later in the proof of \Cref{thm:fast-converge-main}, this property of the initial point is critical for \Cref{main:alg:optimization} to converge at a linear rate. 

\begin{algorithm}[ht]
   \caption{Initialization}
   \label{alg:initialization}
\begin{algorithmic}[1]
   \STATE {\bfseries Input:} $\w^0 = 0$; $\eps,\delta>0$; positive parameters $a, b, L, R, W$; $\mu \lesssim a^2LR^4/b$, step size $\eta = \mu^3/(2^7b^4)$, number of iterations $t_0 \lesssim (b/\mu)^6\log(b/\mu)$;\FOR{$t=0$ {\bfseries to} $t_0$}
\STATE Draw $m_0 \gtrsim W^{9/2}b^{10}d\log^4(d/(\eps\delta))/(L^4\mu^6\delta\eps^{3/2})$ i.i.d.\ samples from $\D$
\STATE $\hat{u}^t = \argmin\limits_{u\in\U}\frac{1}{m_0}\sum\limits_{i=1}^{m_0} (u(\w^t\cdot\x\ith) - y\ith)^2$.
\STATE $\nabla\htLsur(\w^t;\hat{u}^t) = \frac{1}{m_0}\sum\limits_{i=1}^{m_0} (\hat{u}^t(\w^t\cdot\x\ith) - y\ith)\x\ith$.
\STATE $\w^{t+1} = \w^t - \eta\nabla\htLsur(\w^t;\hat{u}^t)$.
\ENDFOR
\STATE {\bfseries Return:} $\{\w^0,\dots,\w^{t_0}\}$\end{algorithmic}
\end{algorithm}

\begin{lemma}[Initialization]\label{main:lem:initialization}
    Let $\mu = Ca^2LR^4/b$ for an absolute constant $C > 0$ and let $\eps,\delta>0$. 
Choose the step size $\eta = \mu^3/(2^7b^4)$ in \Cref{alg:initialization}. Then, drawing $m_0$ i.i.d.\ samples from $\D$ at each iteration such that 
    \begin{equation*}
        m_0 \gtrsim \frac{W^{9/2}b^{10}d\log^4(d/(\eps\delta))}{L^4\mu^6\delta\eps^{3/2}},
    \end{equation*}
ensures that within $t_0 \lesssim b^{6}\log(b/\mu)/\mu^6$ iterations, the initialization subroutine \Cref{alg:initialization} generates a list of size $t_0$ that contains a point $\barw^0$ such that $\|(\w^*)^{\perp_{\barw^0}}\|_2\leq \max\{\mu\|\w^*\|_2/(4b), {64b^2}{/\mu^3}(\sqrt{\opt} + \sqrt{\eps})\}$, with probability at least $1 - \delta$. The total number of samples required for \Cref{alg:initialization} is $N_0 = t_0 m_0$.
\end{lemma}

\begin{proof}
     Consider first the case that $\|\w^*\|_2\leq 64b^2/\mu^3(\sqrt{\opt} + \sqrt{\eps})$. Then, for the parameter vector $\w^0 = 0$, we have  
     $$
     \|(\w^*)^{\perp_{\w^0}}\|_2 = \|\w^*\|_2 \leq 64b^2/\mu^3(\sqrt{\opt} + \sqrt{\eps})
     $$ 
     and the claimed statement holds trivially.

     Thus, in the rest of the proof we assume $\|\w^*\|_2\geq 64b^2/\mu^3(\sqrt{\opt} + \sqrt{\eps})$. Let $\bv^t$ denote the component of $\w^*$ that is orthogonal to $\bw^t$; i.e., $\bv^t = \w^* - (\w^*\cdot\w^t)\w^t/\|\w^t\|_2^2 = (\w^*)^{\perp_{\w^t}}$, where $\w^t$ is defined in \Cref{alg:initialization}. 
     Our goal is to show that when $\|\bv^t\|_2\geq \mu\|\w^*\|_2/(4b)$ at iteration $t$, the distance between $\w^{t+1}$ and $\w^*$ contracts by a constant factor $1 - c$ for some $c<1$, i.e., $\|\w^{t+1} - \w^*\|_2\leq (1 - c)\|\w^t - \w^*\|_2$. This implies that when $\|\bv^t\|_2$ is greater than $\mu\|\w^*\|_2/(4b)$, $\|\w^{t+1} - \w^t\|_2$ contracts until $\|\bv^t\|_2\geq \mu\|\w^*\|_2/(4b)$ is violated at step $t_0$; this $\w^{t_0}$ is exactly the initial point we are seeking to initialize the optimization subroutine.  
     
    Applying \Cref{main:thm:sharpness}, we get that under our choice of batch size $m$, with probability at least $1 - \delta$, at each iteration it holds
    \begin{equation*}
        \nabla\htLsur(\w^t;\htumt)\cdot(\w^t - \w^*)\geq \frac{Ca^2LR^4}{b}\|(\w^*)^{\perp_{\w^t}}\|_2^2 - 2(\sqrt{\opt} + \sqrt{\eps})\|\w^t - \w^*\|_2 - 2(\opt + \eps)/b \;.
    \end{equation*}
    We now study the distance between $\w^{t+1}$ and $\w^*$, where $\w^{t+1}$ is updated from $\w^t$ according to \Cref{alg:initialization}. 
    \begin{align}\label{initial-eq:||w^t+1 - w^*||-bound-1}
        \|\w^{t+1} - \w^*\|_2^2 &= \|\w^t - \eta\nabla\htLsur(\w^t;\htumt) - \w^*\|_2^2 \nonumber\\
        &= \|\w^t - \w^*\|_2^2 + \eta^2\|\nabla\htLsur(\w^t;\htumt)\|_2^2 - 2\eta \nabla\htLsur(\w^t;\htumt)\cdot(\w^t - \w^*) \;.
    \end{align}
    Applying \Cref{main:cor:bound-norm-empirical-grad} to \eqref{initial-eq:||w^t+1 - w^*||-bound-1}, and plugging in \Cref{main:thm:sharpness}, we get that under our choice of batch size $m$ it holds that with probability at least $1 - \delta$,
    \begin{align}\label{initial-eq:||w^t+1 - w^*||-bound-2}
        \|\w^{t+1} - \w^*\|_2^2 &\leq \|\w^t - \w^*\|_2^2 + \eta^2(10(\opt + \eps) + 4b^2\|\w^t - \w^*\|_2^2)  \nonumber\\
        &\quad + 2\eta(2(\opt + \eps)/b + 2(\sqrt{\opt} + \sqrt{\eps})\|\w^t - \w^*\|_2 - \mu\|\bv^t\|_2^2)\nonumber\\
        &\leq (1 + 4b^2\eta^2)\|\w^t - \w^*\|_2^2 + 2\eta(2(\sqrt{\opt} + \sqrt{\eps})\|\w^t - \w^*\|_2 - \mu\|\bv^t\|_2^2) \nonumber\\
        &\quad + 5\eta(\opt + \eps) \;,
    \end{align}
    where $\mu = Ca^2LR^4/b$ and $C$ is an absolute constant. Note that in the last inequality we used that $\eta\leq 1/10$, hence $10\eta^2\leq \eta$, and that $b\geq 1$.

    When $t=0$, $\bv^0 = \w^*$, hence we  have $\|\bv^0\|_2\geq \mu\|\w^*\|_2/(4b)$. Suppose that at iteration $t$, $\|\bv^t\|_2\geq \mu\|\w^*\|_2/(4b)$ is still valid. Then, \eqref{initial-eq:||w^t+1 - w^*||-bound-2} is transformed to: 
\begin{align}\label{initial-eq:||w^t+1 - w^*||-bound-3}
        \|\w^{t+1} - \w^*\|_2^2 &\leq (1 + 4b^2\eta^2)\|\w^t - \w^*\|_2^2  + 5\eta(\opt + \eps)\nonumber\\
        &\quad + 2\eta((\mu^3/(32b^2))\|\w^t - \w^*\|_2\|\w^*\|_2 - (\mu^3/(16b^2))\|\w^*\|_2^2) \;.
    \end{align}
    We use an inductive argument to show that at iteration $t$, $\|\w^{t} - \w^*\|_2\leq \|\w^*\|_2$, which must eventually yield a contraction $\|\w^{t+1} - \w^*\|_2^2\leq (1 - c)\|\w^t - \w^*\|_2^2$ for some constant $c<1$. This condition $\|\w^{t} - \w^*\|_2\leq \|\w^*\|_2$ certainly holds for the base case $t=0$ 
as $\w^0 = 0$, hence $\|\w^0 - \w^*\|_2 = \|\w^*\|_2$.
Now, suppose $\|\w^t - \w^*\|_2\leq \|\w^*\|_2$ holds for all the iterations from $0$ to $t$. Then, plugging  $\eta = \mu^3/(2^7b^4)$ into \eqref{initial-eq:||w^t+1 - w^*||-bound-3}, we get:
    \begin{align*}
        \|\w^{t+1} - \w^*\|_2^2 &\leq (1 + 4b^2\eta^2)\|\w^t - \w^*\|_2^2  + 2\eta((\mu^3/(32b^2)) - (\mu^3/(16b^2)))\|\w^t - \w^*\|_2\|\w^*\|_2 + 5\eta(\opt + \eps)\\
&\leq (1 + 4\eta^2b^2 - 2\eta\mu^3/(32b^2))\|\w^t - \w^*\|_2^2 + 5\mu^3/(2^7b^4)(\opt + \eps)\\
        &\leq (1 - \mu^6/(2^{11}b^{6}))\|\w^t - \w^*\|_2^2 + 5\mu^3/(2^7b^4)(\opt + \eps)\;.
    \end{align*}
    Since we have assumed $\sqrt{\opt} + \sqrt{\eps}\leq \mu^3/(64b^2)\|\w^*\|_2$, it holds $\|\w^t - \w^*\|_2\geq \|\bv^t\|_2\geq \mu\|\w^*\|_2/(4b)\geq (16b/\mu^2)(\sqrt{\opt} + \sqrt{\eps})$, thus, we have (noting that $\mu\leq 1$):
    \begin{equation*}
        5\mu^3/(2^7b^4)(\opt + \eps)\leq 5\mu^3/(2^7b^4)(\sqrt{\opt} + \sqrt{\eps})^2\leq \mu^6/(2^{12}b^{6})\|\w^t - \w^*\|_2 \;.
    \end{equation*}
    Therefore, combining the results above, we get:
    \begin{equation*}
        \|\w^{t+1} - \w^*\|_2^2\leq (1 - \mu^6/(2^{12}b^{6}))\|\w^t - \w^*\|_2^2 \;,
    \end{equation*}
    for any iteration $t$ such that $\|\bv^t\|_2\geq \mu\|\w^*\|_2/(4b)$ holds. This validates the induction argument that $\|\w^t - \w^*\|_2\leq \|\w^*\|_2$ for every $t = 0,\dots,t_0$ and at the same time yields the desired contraction property of the sequence $\|\w^t - \w^*\|_2$, $t = 0,\dots, t_0$. Now, since $\|\w^0 -\w^*\|_2 = \|\w^*\|_2$ and $\|\w^t - \w^*\|_2\geq \|\bv^t\|_2$, we have
    \begin{equation*}
        \|\bv^{t+1}\|_2^2\leq (1 - \mu^6/(2^{12}b^{6}))^t\|\w^*\|_2^2\leq \exp(-t\mu^6/(2^{12}b^{6}))\|\w^*\|_2^2 \;.
    \end{equation*}
    Thus, after at most $t_0 = 2^{12}b^{6}\log(4b/\mu)/\mu^6$ iterations, it must hold that among all those vectors $\bv^1,\dots,\bv^{t_0}$, there exists a vector $\bv^{t^*_0}$ such that $\|\bv^{t_0^*}\|_2\leq \mu\|\w^*\|_2/(4b)$.
Since there are only a constant number of candidates, we can feed each one as the initialized input to the optimization subroutine \Cref{main:alg:optimization}. This will only result in a constant {factor} increase in the runtime and sample complexity.

    Finally, recall that we need to draw 
    \begin{equation*}
        m \gtrsim \frac{W^{9/2}b^4\log^4(d/(\eps\delta))}{L^4}\bigg(\frac{1}{\eps^{3/2}} + \frac{1}{\eps\delta}\bigg) 
    \end{equation*}
    new samples at each iteration for \eqref{initial-eq:||w^t+1 - w^*||-bound-2} to hold with probability $1 -\delta$, and the total number of iterations is $t_0$. Thus, applying a union bound, we know that the probability that \eqref{initial-eq:||w^t+1 - w^*||-bound-2} holds for all $t_0$ is $1 - t_0\delta$. Hence, choosing $\delta\gets \delta t_0$, and {noting} that $t_0 \approx b^6/\mu^6\log(b/\mu)$, it follows that setting the batch size to be
    \begin{equation*}
        m_0 = \Theta\bigg(\frac{W^{9/2}b^4\log^4(d/(\eps\delta))}{L^4}\bigg(\frac{1}{\eps^{3/2}} + \frac{b^6\log(b/\mu)}{\mu^6\eps\delta}\bigg)\bigg) = \Theta\bigg(\frac{W^{9/2}b^{10}d\log^4(d/(\eps\delta))}{L^4\mu^6\delta\eps^{3/2}}\bigg) \;,
    \end{equation*}
    suffices and the total number of samples required for the initialization process is $t_0 m_0$.
\end{proof}

\subsection{Optimization}

Our main optimization algorithm is summarized in \Cref{main:alg:optimization} ({see \Cref{alg:optimization} for a more detailed version}). We now provide intuition for how guessing the value of $\|\wstar\|_2$ is used in the convergence analysis. Let $\w^t = \|\w^*\|_2 \barw^t/\|\barw^t\|_2$  so that $\|\w^t\|_2 = \|\wstar\|_2$ 
and let  
$\bv^t := (\w^*)^{\perp_{\w^t}}$. 
Observe that $\|\bv^t\|_2 = \|\w^t - \w^*\|_2\cos(\theta(\w^t,\w^*)/2)$.
Applying \Cref{main:thm:sharpness}, it can be shown that  $\|\barw^{t+1} - \w^*\|_2^2\leq \|\w^t - \w^*\|_2^2 - C\|\bv^t\|_2^2$ for some constant $C$. 
Thus, as long as the angle between $\w^t$ and $\w^*$ is not too large (ensured by initialization), $\|\w^t - \w^*\|_2 \approx \|\bv^t\|_2$. Hence, we can argue that $\|\w^t - \w^*\|_2$ contracts in each iteration, by observing that $\|\w^t - \w^*\|_2^2 \approx \|\bv^{t+1}\|_2^2\leq\|\barw^{t+1} - \w^*\|_2^2$.

\begin{algorithm}[ht]
   \caption{Optimization}
   \label{main:alg:optimization}
\begin{algorithmic}[1]
   \STATE {\bfseries Input:} {$\w^{\mathrm{ini}} = \vec 0$; $\eps>0$; positive parameters: $a$, $b$, $L$, $R$, $W$, $\mu$; step size $\eta$} \STATE $\{\w^\mathrm{ini}_0,\dots,\w^\mathrm{ini}_{t_0}\} = \text{Initialization}[\w^{\mathrm{ini}}]$ (\Cref{alg:initialization})
\STATE $\mathcal{P} = \{(\w = 0; u(z) = 0)\}$ 
   \FOR{$k = 0$ {\bfseries to} $t_0 \lesssim (b/\mu)^6\log(b/\mu)$}
\FOR{$j = 1$ {\bfseries to} $J = W/(\eta\sqrt{\eps})$}
\STATE $\barw^{0}_{j,k} = \w^\mathrm{ini}_k$, $\frkwstr_j = j\eta\sqrt{\eps}$
\FOR{$t=0$ {\bfseries to} {$T = O((b/\mu)^2\log(1/\eps))$}}
\STATE $\htw^t_{j,k} = \frkwstr_{j}(\barw^t_{j,k}/\|\barw^t_{j,k}\|_2)$ \STATE Draw {$m  = \Tilde{\Theta}_{W,b,1/L,1/\mu} ({d}/{\eps^{3/2}})$} new samples
\STATE $\hat{u}^t_{j,k} = \argmin\limits_{u\in\U}\frac{1}{m}\sum\limits_{i=1}^m (u(\htw^t_{j,k}\cdot\x\ith) - y\ith)^2$
\STATE $\barw^{t+1}_{j,k} = \htw^t_{j,k} - \eta\nabla\htLsur(\htw^t_{j,k};\hat{u}^t_{j,k})$
\ENDFOR
\STATE $\mathcal{P}\gets \mathcal{P}\cup \{(\htw^{T}_{j,k}; \hat{u}^T_{j,k})\}$ \ENDFOR
\ENDFOR
\STATE $(\htw; \hat{u}) = \text{Test}[(\w;u)\in\mathcal{P}]$ (\Cref{alg:testing}) \label{main:line:testing} \STATE {\bfseries Return:} $(\htw; \hat{u})$ \end{algorithmic}
\end{algorithm}

Our main result is the following theorem {(see 
\Cref{app:thm:fast-converge-main} for a more detailed statement and proof in  \Cref{app:pf:thm:fast-converge-main})}: 
{
\begin{theorem}[Main Result]\label{thm:fast-converge-main}
    Let $\D$ be a distribution in $\R^d\times \R$ and suppose that $\D_\x$ is $(L,R)$-well-behaved. Let $\U$ be as in \Cref{def:well-behaved-unbounded-intro} and let $\eps> 0$. 
    Then, \Cref{main:alg:optimization} uses $N = \Tilde{O}_{W,b,1/L,1/\mu}(d/\eps^2)$ samples, it runs for $\Tilde{O}_{W,b,1/\mu}(1/\sqrt{\eps})$ iterations, and, with probability at least $2/3$,  returns a hypothesis $(\hat{u},\htw)$, where $\hat{u}\in\U$ and $\htw\in\B(W)$,
 such that
$  \Ltwo(\htw;\hat{u}) = O_{1/L,1/R,b/a}(\opt) + \eps\;.$
\end{theorem}}

To prove \Cref{thm:fast-converge-main}, we make use of two technical results  
stated below. First, \Cref{main:cor:bound-norm-empirical-grad} provides an 
upper bound on the norm of the empirical gradient of the surrogate loss. The proof of the lemma relies on  concentration 
properties of  $(L,R)$-well behaved distributions $\D_\x$, and leverages the 
uniform convergence 
of the empirically-optimal activations $\htumt$. A more detailed statement (\Cref{app:cor:bound-norm-empirical-grad}) and the proof of \Cref{main:cor:bound-norm-empirical-grad} is deferred to \Cref{app:pf:main:cor:bound-norm-empirical-grad}.

\begin{lemma}[Bound on Empirical Gradient Norm]\label{main:cor:bound-norm-empirical-grad}
    Let $S$ be a set of i.i.d.\ samples {from $\D$} of size $m = \Tilde{\Theta}_{W,b,1/L}(d/\eps^{3/2} + d/(\eps\delta))$. 
Given any $\w^t\in\B(W)$, let $\htumt\in\U$ be the solution of optimization problem \eqref{def:htumt} with respect to $\w^t$ and sample set $S$. Then, with probability at least $1 - \delta$, $$        \|\nabla\htLsur(\w^t;\htumt)\|_2^2\leq 4b^2\|\w^t - \w^*\|_2^2 + 10 (\opt + \eps) \;.$$
\end{lemma}

The following claim bounds the $L_2^2$ error of a hypothesis 
$\hat{u}_\w(\w\cdot \x)$ by the distance between $\w$ and $\w^*$. 
We defer {a more detailed statement (\Cref{app:lem:L2-error-upbd-||w - w^*||^2})} 
and the proof to \Cref{app:pf:main:lem:L2-error-upbd-||w - w^*||^2}.

\begin{claim}\label{main:lem:L2-error-upbd-||w - w^*||^2}
    Let $\w\in \B(W)$ be any fixed vector. Let $\htuw$ be defined by \eqref{def:htumt} given $\w$ and a sample set of size $m = \Tilde{\Theta}_{W,b,1/L}(d/\eps^{3/2})$. Then, $\Exy[(\htuw(\w\cdot\x) - y)^2]\leq 8(\opt + \eps) + 4b^2\|\w - \w^*\|_2^2.$
\end{claim}

\begin{proof}[Proof Sketch of \Cref{thm:fast-converge-main}]
{For this sketch, we consider the case $\|\w^*\|_2\gtrsim b^3/\mu^4(\sqrt{\opt} + \sqrt{\eps})$ so that the initialization subroutine generates a point }$\w^{\mathrm{ini}}_{k^*}\in\{\w^{\mathrm{ini}}_{k}\}_{k=1}^{t_0}$ such that $\|(\w^*)^{\perp_{\w^{\mathrm{ini}}_{k^*}}}\|_2\leq \mu\|\w^*\|_2/(4b)$, by \Cref{main:lem:initialization}. Fix this initialized parameter $\barw^0_{j,k^*} = \w^{\mathrm{ini}}_{k^*}$ at step $k^*$  and drop the subscript $k^*$ for simplicity. Since we constructed a grid with width $\eta\sqrt{\eps}$, there {exists} an index $j^*$ such that $|\frkwstr_{j^*} - \|\w^*\|_2|\leq \eta\sqrt{\eps}$. We consider the intermediate for-loop at this iteration $j^*$, and show that the inner loop with {normalization factor} $\frkwstr_{j^*}$ outputs a solution with error $O(\opt) + \eps$. This solution can be {selected} using standard testing procedures. We now focus on the iteration $j^*$, and drop the subscript $j^*$ for notational simplicity.

    Let $\w^t = \|\w^*\|_2(\barw^t/\|\barw^t\|_2)$ and denote $\bv^t := (\w^*)^{\perp_{\htw^t}}$. Expanding $\|\barw^{t+1} - \w^*\|_2^2$ and applying \Cref{main:thm:sharpness} and \Cref{main:cor:bound-norm-empirical-grad}, we get
    \begin{align}
        \|\barw^{t+1} - \w^*\|_2^2 &= \|\htw^t - \eta\nabla\htLsur(\htw^t;\htumt) - \w^*\|_2^2 \nonumber\\
        &= \|\htw^t - \w^*\|_2^2 + \eta^2\|\nabla\htLsur(\htw^t;\htumt)\|_2^2 - 2\eta \nabla\htLsur(\htw^t;\htumt)\cdot(\htw^t - \w^*) \nonumber\\
        &\leq \|\htw^t - \w^*\|_2^2 + \eta^2(10(\opt + \eps) + 4b^2\|\htw^t - \w^*\|_2^2) \nonumber\\
        & \quad +2\eta( 2(\sqrt{\opt} + \sqrt{\eps})\|\htw^t - \w^*\|_2 - \mu\|\bv^t\|_2^2)  + 4\eta(\opt + \eps)/b \nonumber\\
        &\leq (1 + 4\eta^2b^2)\|\w^t - \w^*\|_2^2 + (24\eta^2 + {4\eta}/{b})(\opt + \eps) \nonumber\\
        & \quad + 2\eta(2(\sqrt{\opt} + \sqrt{\eps})\|\w^t - \w^*\|_2 - \mu\|\bv^t\|_2^2)  , \label{eq:main-thm-sketch-decrease-of-distance-1}
    \end{align}
    where in the last inequality we used  $\|\htw^t - \w^t\|_2 = |\beta_{j^*} - \|\wstar\|_2|\leq \eta\sqrt{\eps}$.
    
    Since $\w^t$ and $\w^*$ are on the same sphere, $\|\w^t - \w^*\|_2 \leq \|\bv^t\|_2$. In particular, letting $\rho_t = \|\bv^t\|_2/\|\w^*\|_2$, we have $\|\w^t - \w^*\|_2^2\leq (1 + \rho_t^2)\|\bv^t\|_2^2\leq 2\|\bv^t\|_2^2$. Recall that the algorithm is {initialized} from $\barw^0$ that satisfies $\rho_0\leq \mu/(4b)$. If $\rho_t\leq \mu/(4b)$, then $\|\w^t - \w^*\|_2^2\leq (1 + (\mu/(4b))^2)\|\bv^t\|_2^2$. Assuming in addition that $\|\bv^t\|_2\gtrsim (1/\mu)(\sqrt{\opt} + \sqrt{\eps})$, and choosing the step-size $\eta = \mu/(4b^2)$, \eqref{eq:main-thm-sketch-decrease-of-distance-1} implies that $$
    \|\bv^{t+1}\|_2^2\leq \|\barw^{t+1} - \w^*\|_2^2\leq (1 - \mu^2/(32b^2))\|\bv^t\|_2^2 \;,
    $$ 
    and thus, in addition, $\rho_{t+1}\leq \mu/(4b)$. 
    Therefore, by an inductive argument, we show that as long as $\w^t$ is still far from $\w^*$, i.e., $\|\bv^t\|_2\gtrsim (1/\mu)(\sqrt{\opt} + \sqrt{\eps})$, we have 
    $$
    \|\bv^{t+1}\|_2^2\leq (1 - \mu^2/(32b^2))\|\bv^t\|_2^2 \;\text{ and }\; \rho_{t+1}\leq \mu/(4b) \;.$$ 
Hence,  after $T=O((b^2/\mu^2)\log(1/\eps))$ iterations, it must be $\|\bv^T\|_2\lesssim (1/\mu)(\sqrt{\opt} + \sqrt{\eps})$, which implies 
    $$
    \|\w^T - \w^*\|_2^2\leq 2\|\bv^T\|_2^2 = O(\opt) + \eps \;.
    $$
    Finally, by \Cref{main:lem:L2-error-upbd-||w - w^*||^2},  hypothesis $\hat{u}^T(\htw^T\cdot\x)$ achieves $L_2^2$-error $O(\opt) + \eps$, which completes the proof. \end{proof}

\subsection{Testing}\label{subsec:testing}

We now briefly discuss the testing procedure, which allows our algorithm to select a hypothesis with minimum empirical error while maintaining validity of the claims. This part  relies on standard arguments and is provided for completeness. Concretely, {we} rely on the following claim, whose proof can be found in \Cref{app:pf:main:lem:testing}.

\begin{algorithm}[ht]
   \caption{Testing}
   \label{alg:testing}
\begin{algorithmic}[1]
\STATE {\bfseries Input:} $\eps>0$; positive parameters: $a$, $b$, $L$, $R$, $W$; list of solutions $\mathcal{P}$; let $r \gtrsim \frac{1}{L}\log({bW}/{(L\eps)}\log^2(1/\eps))$ 
\STATE Draw $m' \gtrsim (bW/L)^4\log^5(1/\eps)/\eps^2$ new i.i.d. samples from $\D$.
\STATE $(\htw; \hat{u}) = \argmin_{(\w;u)\in\mathcal{P}} \{\frac{1}{m'}\sum_{i=1}^{m'} (u(\w\cdot\x\ith) - y\ith)^2\1\{|\w\cdot\x\ith|\leq Wr\}\, \}$. 

\STATE {\bfseries Return:} $(\htw; \hat{u})$
\end{algorithmic}
\end{algorithm}
\begin{restatable}{claim}{TestingLemma}\label{main:lem:testing}
    Let $\mu$, $\eps_1$, $\delta \in (0, 1)$ be fixed. Let $r = \frac{1}{L}\log(\frac{Cb^4W^4}{L^6\eps_1^2}\log^2(\frac{bW}{\eps_1}))$, where $C$ is a sufficiently large absolute constant. Given a set of parameter-activation pairs $\mathcal{P}= \{(\w_j; u_j)\}_{j=1}^{t_0J}$ such that $\w_j\in\B(W)$ and $u_j\in\U$ for $j\in[t_0J]$, where $t_0J = 4b^9W/(\mu^8\sqrt{\eps_1})$, we have that using 
    $$m' = {\Theta}\bigg(\frac{b^4W^4\log(1/\delta)}{L^4\eps_1^2}\log^5\bigg(\frac{bW}{L\mu\eps_1}\bigg)\bigg),$$
    i.i.d.\ samples from $\D$, for any $(\w_j;u_j)\in\mathcal{P}$ it holds with probability at least $1 - \delta$,
    \begin{equation*}
        \bigg| \frac{1}{m'}\sum_{i=1}^{m'} (u_j(\w_j\cdot\x\ith) - y\ith)^2\1\{|\w_j\cdot\x\ith|\leq Wr\} - \Exy[(u_j(\w_j\cdot\x) - y)^2] \bigg|\leq 2\eps_1.
    \end{equation*}    
\end{restatable}

Therefore, \Cref{main:lem:testing} guarantees that selecting a hypothesis using the provided testing procedure introduces an error at most 
$2\eps_1$, with high probability.

\section{Conclusion}\label{sec:conclusion}

We presented the first constant-factor approximate SIM learner in the agnostic model, 
for the class of $(a, b)$-unbounded link functions under mild distributional 
assumptions. Immediate questions for future research involve extending these results to 
other classes of link functions. More specifically, our results require that $b/a$ is 
bounded by a constant. It is an open question whether the constant-factor approximation 
result in the agnostic model can be extended to all $b$-Lipschitz functions (with $a = 
0$). This question is open in full generality, even when the link function is known to the learner.  

\bibliographystyle{alpha}
\bibliography{mydb}
\newpage

\appendix
\section*{Appendix}\label{sec:app}

\paragraph{Organization}
The appendix is organized as follows. 
In \Cref{app:remaks}, we highlight some useful properties about the distribution class and the activation class. \Cref{app:local-error-bounds} reviews local error bounds and discussed their relation to our alignment sharpness structural result. In \Cref{app:pf:sharpness}, we provide detailed proofs omitted from \Cref{sec:sharp}, and in \Cref{app:optimization} we complete the proofs omitted from  \Cref{sec:optimization}. In \Cref{app:isotonic-regression}, we provide a detailed discussion about computing the sample-optimal activation. Finally, in \Cref{app:uniform-convergence} we state and prove  standard uniform convergence results that are used throughout the paper.

\section{Remarks about the Distribution Class and the Activation Class}\label{app:remaks}
In this section, we show that without the loss of generality we can assume that the parameters $L,R$ in the distributional assumptions (\Cref{def:bounds}) can be taken less than 1, while the parameters $a,b$ of the activations functions (see \Cref{def:well-behaved-unbounded-intro}) can be taken as $a,1/b\leq 1$.

\begin{remark}[Distribution/Activation Parameters, (\Cref{def:bounds} \& \Cref{def:well-behaved-unbounded-intro})]\label{assumption-on-parameters}
We observe that if a distribution $\D_\x$ is $(L,R)$-well-behaved, then it is also $(L',R')$-well-behaved for any $0<L'\leq L, 0<R'\leq R$. Hence, it is without loss of generality to assume that $L,R\in(0,1]$. Similarly, if an activation is $(a,b)$-unbounded, it is also an $(a,b')$-unbounded activation with $b'\geq b$. Thus, we assume that $b\geq 1$. We can similarly assume $a\leq 1$. \end{remark}

In addition, we remark that the $(L,R)$-well behaved distributions are sub-exponential.
\begin{remark}[Sub-exponential Tails of Well-Behaved Distributions, \Cref{def:bounds}]\Cref{def:bounds} might seem abstract, but to put it plain it implies that the random variable $\x$ has a $(1/L)$-sub-exponential tail, and that the pdf of the projected random variable $\x_V$ onto the space $V$ is lower bounded by $L$. To see the first statement, given any unit vector $\vec p$, let $\x_{\vec p}$ be the projection of $\x$ onto the one-dimensional linear space $V_{\vec p} = \{\vec z\in\R^d: \vec z = t\vec p, t\in\R \}$, i.e., $\x_{\vec p} = \vec p\cdot\x \in V_{\vec p}$. Then, by the anti-concentration and concentration property, we have
    \begin{equation*}
        \pr[|\vec p\cdot\x|\geq r] = \pr[|\x_{\vec p}|\geq r] \leq \int_{|x|\geq r}\gamma(x)\diff{x}\leq 2\int_r^\infty \frac{1}{L}\exp(-Lx)\diff{x} = \frac{2}{L^2}\exp(-Lr),
    \end{equation*}
    which implies that $\x$ possesses a sub-exponential tail.
\end{remark}

\section{Local Error Bounds and Alignment Sharpness}\label{app:local-error-bounds}

Given a generic optimization problem $\min_{\w} f(\w)$ and a non-negative residual function $r(\w)$ measuring the approximation error of the optimization problem, we say that the problem satisfies a local error bound if in some neighborhood of ``test'' (typically optimal) solutions $\mathcal{W}^*$ we have that
\begin{equation}\label{eq:generic-local-err-bnd}
    r(\w) \geq (\mu/\nu)\, \dist(\w, \mathcal{W}^*)^\nu.
\end{equation}
In other words, low value of the residual function implies that $\w$ must be close to the test set $\mathcal{W}^*.$ 

Local error bounds have been studied in the optimization literature for decades, starting with the seminal works of \cite{hoffman2003approximate,lojasiewicz1963propriete}; see, e.g., Chapter 6 in \cite{facchinei2003finite} for an overview of classical results and \cite{bolte2017error,karimi2016linear,roulet2017sharpness,liu2022solving} and references therein for a more cotemporary overview. While local error bounds can be shown to hold generically under fairly minimal assumptions on $f$ and for $r(\w) = f(\w) - \min_{\w'}f(\w')$ \cite{lojasiewicz1963propriete,lojasiewicz1993geometrie}, it is rarely the case that they can be ensured to hold with a parameter $\mu$ that is not trivially small. 

On the other hand, learning problems often possess very strong structural properties that can lead to stronger local error bounds. There are two main such examples we are aware of, where local error bounds can be shown to hold with $\nu = 2$ and an absolute constant $\mu >0$. The first example are low-rank matrix problems such as matrix completion and matrix sensing, which are unrelated to our work \cite{bhojanapalli2016global,zheng2016convergence,jin2017escape}. More relevant to our work is the recent result in \cite{WZDD2023}, which proved a local error bound of the form 
\begin{equation}
    r(\w) \geq \frac{\mu}{2}\dist(\w, \mathcal{W}^*)^2
\end{equation}
for the more restricted problem than ours (with a known activation function) but under somewhat more general distributional assumptions. In \cite{WZDD2023}, the residual function was defined by $r(\w^t) = \nabla\htLsur(\w^t; u^*)\cdot(\w^t - \w^*),$ where $\nabla\htLsur(\w^t; u^*)$ is the gradient of an empirical surrogate loss, and the resulting local error bound referred to as ``sharpness.''\footnote{A local error utilizing the same type of a residual was introduced in \cite{ZY2013} under the name ``restricted secant inequality.''} 

Our structural result can be seen as a weak notion of a local error bound, where the residual function for the empirical surrogate loss expressed as $r(\w^t, \htumt) = \nabla\htLsur(\w^t;\htumt)\cdot(\w^t - \w^*)$ is bounded below as a function of the magnitude of the component of $\wstar$ that is orthogonal to $\w^t.$ Compared to more traditional local error bounds and the bound from \cite{WZDD2023}, which bound below the residual error function as a function of the distance to $\mathcal{W}^*$, this is a much weaker local error bound since it does not distinguish between vectors of varying magnitudes along the direction of $\wstar.$ Since our lower bound is related to the ``sharpness'' notion studied in \cite{WZDD2023}, we refer to it as the ``alignment sharpness'' to emphasize that it only relates the misalignment (as opposed to the distance) of vectors $\w^t$ and $\wstar$ to the residual error. To the best of our knowledge, such a form of a local error bound, which only bounds the alignment of vectors as opposed to their distance, is novel. We expect it to find a more broader use in learning theory and optimization.

\section{Omitted Proofs from \Cref{sec:sharp}}\label{app:pf:sharpness}

This section provides full technical details for results omitted from \Cref{sec:sharp}.

\subsection{Proof of \Cref{main:lem:upper-bound-u_t-u_t^*}}\label{app:pf:main:lem:upper-bound-u_t-u_t^*}

To prove \Cref{main:lem:upper-bound-u_t-u_t^*}, we first prove the following auxiliary claim, which is inspired by  \cite[Lemma~9]{kakade2011efficient}.

\begin{claim}\label{app:claim:(ut-v)(y-ut)geq0}
    Let $\w^t\in\B(W)$ and let $u^{*t},u^t$ be defined as solutions to \eqref{def:ut*}, \eqref{def:ut}, respectively. Then, 
    $$\Exy[(u^t(\w^t\cdot\x) - v(\w^t\cdot\x))(y - u^t(\w^t\cdot\x))]\geq 0, \quad \forall \in\U. $$ 
    Similarly, 
    $$\Exy[(u^{*t}(\w^t\cdot\x) - v'(\w^t\cdot\x))(y^* - u^{*t}(\w^t\cdot\x))]\geq 0, \quad \forall v'\in\U.$$
\end{claim}
\begin{proof}[Proof of \Cref{app:claim:(ut-v)(y-ut)geq0}]
   Denote by $\F_t$ the set of functions of the form $f(\x) = u(\w^t\cdot\x)$, where $u\in\U$ and $\w^t$ is a fixed vector in $\B(W)$. We first argue that $\F_t$ is a convex set, using the definition of convexity. In particular, for any $\alpha\in (0,1)$ and   any $f_1,f_2\in\F_t$ such that $f_1(\x) = u_1(\w^t\cdot\x), f_2 = u_2(\w^t\cdot\x)$, let $u_3(\cdot) = \alpha u_1(\cdot) + (1 - \alpha) u_2(\cdot)$. Then:
    \begin{equation*}
        \alpha f_1(\x) + (1 - \alpha) f_2(\x) = \alpha u_1(\w^t\cdot\x) + (1 - \alpha) u_2(\w^t\cdot\x) = u_3(\w^t\cdot\x).
    \end{equation*}
    It is immediate that $u_3$ is also $(a,b)$-bounded, non-decreasing, and $u_3(0) = 0$, hence $u_3\in\U$ and $f_3(\x) = u_3(\w^t\cdot\x)\in\F_t$. Thus, $\F_t$ is convex.

    Since $\F_t$ is a convex set of functions, we can regard $u^t(\w^t\cdot\x)$ as the orthogonal projection of $y$ (which is a function of $\x$) onto the convex set $\F_t$. Classic inequalities for orthogonal projections can then be applied to our case. In particular, below we prove that
    \begin{equation}\label{eq:projection-ineq1}
        \Exy[(u^t(\w^t\cdot\x) - v(\w^t\cdot\x))(y - u^t(\w^t\cdot\x))]\geq 0,\quad \forall v\in\U.
    \end{equation}
    To prove \eqref{eq:projection-ineq1}, note first that $f_u(\x) = u^t(\w^t\cdot\x)\in \F_t$ and $f_v(\x) = v(\w^t\cdot\x)\in\F_t$ since $u^t,v\in\U$. Thus, for any $\alpha\in(0,1)$, we have $\alpha f_v(\x) + (1 - \alpha) f_u(\x)\in\F_t$. Furthermore, by definition of $u^t$, $\forall f\in\F_t$ we have $\Exy[(u^t(\w^t\cdot\x) - y)^2]\leq \Exy[(f(\x) - y)^2]$, therefore, it holds:
    \begin{align*}
        0&\leq \frac{1}{\alpha}\Exy[(\alpha f_v(\x) + (1 - \alpha) f_u(\x) - y)^2] - \frac{1}{\alpha}\Exy[(u^t(\w^t\cdot\x) - y)^2]\\
        &= \frac{1}{\alpha}\Exy[(u^t(\w^t\cdot\x) - y + \alpha(v(\w^t\cdot\x) - u^t(\w^t\cdot\x)))^2 - (u^t(\w^t\cdot\x) - y)^2]\\
        &= \Exy[2(u^t(\w^t\cdot\x) - y)(v(\w^t\cdot\x) - u^t(\w^t\cdot\x)) + \alpha(v(\w^t\cdot\x) - u^t(\w^t\cdot\x))^2].
    \end{align*}
    Let $\alpha\downarrow 0$, and note that $\Ex[(v(\w^t\cdot\x) - u^t(\w^t\cdot\x))^2]<+\infty$, we thus have
    \begin{equation*}
        \Exy[(u^t(\w^t\cdot\x) - v(\w^t\cdot\x))(y - u^t(\w^t\cdot\x))]\geq 0,
    \end{equation*}
    proving the claim.

The second claim can be proved following the same argument and is omitted for brevity. 
\end{proof}

We now proceed to the proof of \Cref{main:lem:upper-bound-u_t-u_t^*}.
\UpperBoundUtUtstr*

\begin{proof}
    Summing up the first and second statement of \Cref{app:claim:(ut-v)(y-ut)geq0} with $v = u^{*t}\in\U$ in \eqref{eq:projection-ineq1} and $v' = u^t\in\U$, we get:
    \begin{align*}
        0&\leq \Exy[(u^t(\w^t\cdot\x) - u^{*t}(\w^t\cdot\x))(y - u^t(\w^t\cdot\x)) + (u^{*t}(\w^t\cdot\x) - u^t(\w^t\cdot\x))(y^* - u^{*t}(\w^t\cdot\x))]\\
        &= \Exy[(u^t(\w^t\cdot\x) - u^{*t}(\w^t\cdot\x))(y - y^* + u^{*t}(\w^t\cdot\x) - u^t(\w^t\cdot\x))]\\
        &= \Exy[(u^t(\w^t\cdot\x) - u^{*t}(\w^t\cdot\x))(y - y^*)] - \Ex[(u^t(\w^t\cdot\x) - u^{*t}(\w^t\cdot\x))^2]
    \end{align*}
    Rearranging and applying the Cauchy-Schwarz inequality, we have
    \begin{align*}
        \Ex[(u^t(\w^t\cdot\x) - u^{*t}(\w^t\cdot\x))^2]&\leq \Exy[(u^t(\w^t\cdot\x) - u^{*t}(\w^t\cdot\x))(y - y^*)]\\
        &\leq \sqrt{\Ex[(u^t(\w^t\cdot\x) - u^{*t}(\w^t\cdot\x))^2]\E[(y - y^*)^2]}.
    \end{align*}
    To complete the proof, it remains to recall that $\E[(y - y^*)^2] = \opt$ and rearrange the last inequality.
\end{proof}

\subsection{Proof of \Cref{main:cor:||htumt-u*(w*x)||<=||w-w*||}}\label{app:pf:main:cor:||htumt-u*(w*x)||<=||w-w*||}
\UpperBoundLtwoDistUtUstr*
\begin{proof}
    The corollary follows directly from the combination of \Cref{main:lem:E[htutm(wt.x) - ut(wt.x))^2]<=eps} and \Cref{main:lem:upper-bound-u_t-u_t^*}, as we have:
    \begin{align*}
        &\quad \Ex[(\htumt(\w^t\cdot\x) - u^*(\w^*\cdot\x))^2]\\
        &=\Ex[(\htumt(\w^t\cdot\x) - u^t(\w^t\cdot\x) + u^t(\w^t\cdot\x) - u^{*t}(\w^t\cdot\x) + u^{*t}(\w^t\cdot\x) - u^*(\w^*\cdot\x))^2]\\
        &\leq 3(\Ex[(\htumt(\w^t\cdot\x) - u^t(\w^t\cdot\x))^2] + \Ex[(u^t(\w^t\cdot\x) - u^{*t}(\w^t\cdot\x))^2])\\
        &\quad + 3\Ex[(u^{*t}(\w^t\cdot\x) - u^*(\w^*\cdot\x))^2]\\
        &\leq 3(\eps + \opt + b^2\|\w^t - \w^*\|_2^2),
    \end{align*}
    where we used that because $u^{*t} \in\argmin_{u\in\U}\Ex[(u(\w^t\cdot\x) - u^*(\w^*\cdot\x))^2]$, we have $\Ex[(u^{*t}(\w^t\cdot\x) - u^*(\w^*\cdot\x))^2]\leq \Ex[(u^*(\w^t\cdot\x) - u^*(\w^*\cdot\x))^2]\leq b^2\|\w^t - \w^*\|_2^2$, with the last inequality following from the fact that {$u^*\in\U$}.
\end{proof}

\subsection{Proof of \Cref{claim:sharp-inter-1}}

In this subsection, we prove \Cref{claim:sharp-inter-1} that appeared in \Cref{subsec:proof-of-sharpness}, the proof of \Cref{main:thm:sharpness}.

\claimsharpone*
\begin{proof}
    Adding and subtracting $u^t(\w^t\cdot\x\ith)$ and $u^{*t}(\w^t\cdot\x\ith)$, we have
    \begin{align}\label{eq:decompose-(htustrmt(w.x)-htumt(w.x))(wt-ws).x)}
        &\quad \frac{1}{m}\sum_{i=1}^m((\htumt(\w^t\cdot\x\ith) - \htustrmt(\w^t\cdot\x\ith))(\w^t - \w^*)\cdot\x\ith\nonumber\\
        & = \frac{1}{m}\sum_{i=1}^m(\htumt(\w^t\cdot\x\ith) - u^t(\w^t\cdot\x\ith))(\w^t - \w^*)\cdot\x\ith  + \frac{1}{m}\sum_{i=1}^m(u^{*t}(\w^t\cdot\x\ith) - \htustrmt(\w^t\cdot\x\ith))(\w^t - \w^*)\cdot\x\ith\nonumber\\
        &\quad + \frac{1}{m}\sum_{i=1}^m(u^t(\w^t\cdot\x\ith) - u^{*t}(\w^t\cdot\x\ith))(\w^t - \w^*)\cdot\x\ith.
    \end{align}
    To proceed, we use that both $\htustrmt(z)$ and $\htumt(z)$ are close to their population counterparts $u^t(z)$ and $u^{*t}(z)$, respectively. In particular, in \Cref{main:lem:E[htutm(wt.x) - ut(wt.x))^2]<=eps} and \Cref{main:lem:E[htustrm(wt.x) - u*t(wt.x))^2]<=eps}, we show that using a dataset $S$ of $m$ samples such that
    \begin{equation*}
        m \gtrsim d\log^4(d/(\eps\delta))\bigg(\frac{b^2W^3}{L^2\eps}\bigg)^{3/2},
    \end{equation*}
    we have that with probability at least $1 - \delta$, for all $\w^t,\w^*\in\B(W)$ it holds
    \begin{equation}\label{eq:sharp-interm-0}
        \Ex[(\htumt(\w^t\cdot\x) - u^t(\w^t\cdot\x))^2]\leq \eps, \; \Ex[(\htustrmt(\w^t\cdot\x) - u^{*t}(\w^t\cdot\x))^2]\leq \eps.
    \end{equation}
    Now suppose that the inequalities in \eqref{eq:sharp-interm-0} hold for the given $\w^t\in\B(W)$ (which happens with probability at least $1 - \delta$). Applying Chebyshev's inequality to the first summation term in \eqref{eq:decompose-(htustrmt(w.x)-htumt(w.x))(wt-ws).x)}, we get:
    \begin{align}\label{eq:cheby-(htumt(wt.x)-ut(wt.x))}
        \pr\bigg[\bigg|\frac{1}{m} \sum_{i=1}^m (\htumt(\w^t\cdot\x\ith) &- u^t(\w^t\cdot\x\ith))(\w^t - \w^*)\cdot\x\ith - \Ex[(\htumt(\w^t\cdot\x) - u^t(\w^t\cdot\x))(\w^t - \w^*)\cdot\x]\bigg|\geq s\bigg]\nonumber\\
        &\leq \frac{1}{ms^2}\Ex[(\htumt(\w^t\cdot\x) - u^t(\w^t\cdot\x))^2(\w^t\cdot\x - \w^*\cdot\x)^2],
    \end{align}
    since $\x\ith$ are i.i.d.\ random variables. The next step is to bound the variance. Note that $\D_\x$ possesses a $1/L$-sub-exponential tail, thus we have $\pr[|(\w^t - \w^*)\cdot\x|\geq \|\w^t - \w^*\|_2 r]\leq (2/L^2)\exp(-Lr)$. Choose $r = \frac{2W}{L}\log(2/(L^2\eps'))$; then, we have $\pr[|(\w^t - \w^*)\cdot\x|\geq r]\leq \eps'$. Now we separate the variance under the event $A = \{\x:|(\w^t - \w^*)\cdot\x|\leq r\}$ and its complement.
    \begin{equation}\label{eq:claim-inside-sharpness-sep}
    \begin{aligned}
        &\quad \Ex[(\htumt(\w^t\cdot\x) - u^t(\w^t\cdot\x))^2(\w^t\cdot\x - \w^*\cdot\x)^2]\\
        & = \Ex[(\htumt(\w^t\cdot\x) - u^t(\w^t\cdot\x))^2(\w^t\cdot\x - \w^*\cdot\x)^2\1\{A\}]\\
        &\quad + \Ex[(\htumt(\w^t\cdot\x) - u^t(\w^t\cdot\x))^2(\w^t\cdot\x - \w^*\cdot\x)^2(1 - \1\{A\})].
    \end{aligned}
    \end{equation}
Using that $\Ex[(\htumt(\w^t\cdot\x) - u^t(\w^t\cdot\x))^2]\leq \eps$, the first term in \eqref{eq:claim-inside-sharpness-sep} can be bounded as follows:
\begin{align}\label{ineq:(htumt(wt.x) - ut(wt.x))^2(wtx-w*x)^2.1(A)}
        \Ex[(\htumt(\w^t\cdot\x) - u^t(\w^t\cdot\x))^2(\w^t\cdot\x - \w^*\cdot\x)^2\1\{A\}]&\leq r^2\Ex[(\htumt(\w^t\cdot\x) - u^t(\w^t\cdot\x))^2] \nonumber\\
        &\leq r^2\eps = \frac{4W^2\eps}{L^2}\log^2(2/(L^2\eps')).
    \end{align}
The second term in \eqref{eq:claim-inside-sharpness-sep} can be bounded using that both $\htumt$ and $u^t$ are non-decreasing $b$-Lipschitz and vanish at zero (thus $|\htumt(\w^t\cdot\x)| \leq b |\w^t\cdot\x|$ and $|u^t(\w^t\cdot\x)| \leq b |\w^t\cdot\x|$, with their signs determined by the sign of $\w^t \cdot \x$), and then applying Young's inequality:
    \begin{align*}
        &\quad \Ex[(\htumt(\w^t\cdot\x) - u^t(\w^t\cdot\x))^2(\w^t\cdot\x - \w^*\cdot\x)^2(1 - \1\{A\})]\\
        &\leq b^2 \Ex[(\w^t\cdot\x)^2(\w^t\cdot\x - \w^*\cdot\x)^2(1 - \1\{A\})]\\
        &\leq 2 b^2\Ex[((\w^t\cdot\x)^4 + (\w^t\cdot\x)^2(\w^*\cdot\x)^2)(1 - \1\{A\})] \;.
    \end{align*}
    Since $\D_\x$ is sub-exponential, we have $\E[(\bv\cdot\x)^8]\leq c^2/L^8$ for some absolute constant $c$, hence 
    $$\Ex[(\w^t\cdot\x)^4(1 - \1\{A\})]\leq \sqrt{\Ex[W^8((\w^t/\|\w^t\|_2)\cdot\x)^8]\pr[|\w^t\cdot\x|\geq r]}\leq cW^4\sqrt{\eps'}/L^4.$$
    Similarly, for $\E[(\w^t\cdot\x)^2(\w^*\cdot\x)^2(1 - \1\{A\})]$, we have:
    \begin{align*}
        \Ex[(\w^t\cdot\x)^2(\w^*\cdot\x)^2(1 - \1\{A\})]&\leq 2\Ex[((\w^t\cdot\x)^4+ (\w^*\cdot\x)^4)(1 - \1\{A\})]\leq 2c(W/L)^4\sqrt{\eps'}.
    \end{align*}
    Combining the inequalities above with \eqref{ineq:(htumt(wt.x) - ut(wt.x))^2(wtx-w*x)^2.1(A)}, we get the final upper bound on the variance in \eqref{eq:claim-inside-sharpness-sep}:
    \begin{align*}
        \Ex[(\htumt(\w^t\cdot\x) - u^t(\w^t\cdot\x))^2(\w^t\cdot\x - \w^*\cdot\x)^2]&\leq \frac{4W^2\eps}{L^2}\log^2(2/(L^2\eps')) + 6 cb^2(W/L)^4\sqrt{\eps'}.
    \end{align*}
    Thus, choosing $s = \eps/b$ in \eqref{eq:cheby-(htumt(wt.x)-ut(wt.x))}, $\eps' = \eps^2$, and using $m \gtrsim  W^4b^4\log^2(1/\eps)/(\eps\delta L^4)$ samples
we get 
    \begin{equation*}
        \frac{1}{ms^2}\bigg(\frac{4W^2\eps}{L^2}\log^2(2/(L\eps')) + \frac{12cb^2W^4\sqrt{\eps'}}{L^4}\bigg) \lesssim \frac{b^2L^4\eps\delta}{\eps^2W^4b^4\log^2(1/\eps)}\bigg(\frac{W^2\eps}{L^2}\log^2\bigg(\frac{1}{L\eps}\bigg) + \frac{b^2W^4\eps}{L^4}\bigg)\leq \delta \;.
    \end{equation*}
Plugging the inequality above back into \eqref{eq:cheby-(htumt(wt.x)-ut(wt.x))} and recalling that $\Ex[(\htumt(\w^t\cdot\x) - u^t(\w^t\cdot\x))^2]\leq \eps$ (from \eqref{eq:sharp-interm-0}), we finally have with probability at least $1 - \delta$,
    \begin{align*}
        &\quad \frac{1}{m} \sum_{i=1}^m (\htumt(\w^t\cdot\x\ith) - u^t(\w^t\cdot\x\ith))(\w^t - \w^*)\cdot\x\ith \\
        &\geq \Ex[(\htumt(\w^t\cdot\x) - u^t(\w^t\cdot\x))(\w^t - \w^*)\cdot\x] - \eps/b\\
        &\geq -\sqrt{\Ex[(\htumt(\w^t\cdot\x) - u^t(\w^t\cdot\x))^2]\Ex[(\w^t\cdot\x - \w^*\cdot\x)^2]} - \eps/b\\
        &\geq -\sqrt{\eps}\|\w^t - \w^*\|_2 - \eps/b,
    \end{align*}
    where in the second inequality we used the Cauchy-Schwarz inequality and in the last inequality we used the assumption that $\Ex[\x\x^\top] \preccurlyeq \mathbf{I}$. Finally, noting that \eqref{eq:sharp-interm-0} holds with probability at least $1 - \delta$, applying a union bound we get that with probability at least $1 - 2\delta$, we have
    \begin{equation*}
        \frac{1}{m} \sum_{i=1}^m (\htumt(\w^t\cdot\x\ith) - u^t(\w^t\cdot\x\ith))(\w^t - \w^*)\cdot\x\ith\geq -\sqrt{\eps}\|\w^t - \w^*\|_2 - \eps/b \;.
    \end{equation*}    
    In summary, to guarantee that the inequality above remains valid, we need the batch size to be:
    \begin{equation}\label{eq:sharp-choice-of-m}
        m \gtrsim \frac{dW^{9/2}b^4\log^4(d/(\eps\delta)}{L^4}\bigg(\frac{1}{\eps^{3/2}} + \frac{1}{\eps\delta}\bigg).
    \end{equation}
    We finished bounding the first term in \eqref{eq:decompose-(htustrmt(w.x)-htumt(w.x))(wt-ws).x)}.

    Since the same statements hold for the relationship between  $\htustrmt$ and $u^{*t}$ as they do for $\htumt$ and $u^t$, using the same argument we also get that with probability at least $1 - 2\delta$,
    \begin{equation*}
        \frac{1}{m} \sum_{i=1}^m (\htustrmt(\w^t\cdot\x\ith) - u^{*t}(\w^t\cdot\x\ith))(\w^t - \w^*)\cdot\x\ith\geq -\sqrt{\eps}\|\w^t - \w^*\|_2 - \eps/b,
    \end{equation*}
    which is the lower bound for the second term in \eqref{eq:decompose-(htustrmt(w.x)-htumt(w.x))(wt-ws).x)}.

    Lastly, for the third term in \eqref{eq:decompose-(htustrmt(w.x)-htumt(w.x))(wt-ws).x)}, since in \Cref{main:lem:upper-bound-u_t-u_t^*} we showed that for any $\w^t$ it always holds: 
    \begin{equation*}
        \Ex[(u^t(\w^t\cdot\x) - u^{*t}(\w^t\cdot\x))^2]\leq \opt,
    \end{equation*}
     the only change of the previous steps is at the right-hand side of \eqref{ineq:(htumt(wt.x) - ut(wt.x))^2(wtx-w*x)^2.1(A)}, where instead of having the upper bound of $r^2\eps$,  we have
    \begin{equation*}
        \Ex[(u^t(\w^t\cdot\x) - u^{*t}(\w^t\cdot\x))^2(\w^t\cdot\x - \w^*\cdot\x)^2\1\{A\}]\leq r^2\opt = \frac{4W^2\opt}{L^2}\log^2(2/(L^2\eps')).
    \end{equation*}
    By the same token, we have
    \begin{equation*}
        \Ex[(u^t(\w^t\cdot\x) - u^{*t}(\w^t\cdot\x))^2(\w^t\cdot\x - \w^*\cdot\x)^2(1 - \1\{A\})]\leq 6cb^2(W/L)^4\sqrt{\eps'}.
    \end{equation*}
    As a result, Chebyshev's inequality yields:
    \begin{align*}
        \pr\bigg[\bigg|&\frac{1}{m} \sum_{i=1}^m (u^t(\w^t\cdot\x\ith) - u^{*t}(\w^t\cdot\x\ith))(\w^t - \w^*)\cdot\x\ith - \Ex[(u^t(\w^t\cdot\x) - u^{*t}(\w^t\cdot\x))(\w^t - \w^*)\cdot\x]\bigg|\geq s\bigg]\nonumber\\
        &\leq \frac{1}{ms^2}\Ex[(u^t(\w^t\cdot\x) - u^{*t}(\w^t\cdot\x))^2(\w^t\cdot\x - \w^*\cdot\x)^2]\\
        &\leq \frac{1}{ms^2}\bigg(\frac{4W^2\opt}{L^2}\log^2(2/(L^2\eps')) + \frac{6cb^2W^4\sqrt{\eps'}}{L^4}\bigg).
    \end{align*}
    Now instead of choosing $s = \eps$, we let $s = (\opt + \eps)/b$ and keep $\eps'$ as $\eps^2$ to get
    \begin{align*}
        &\; \frac{1}{ms^2}\bigg(\frac{4W^2\opt}{L^2}\log^2\bigg(\frac{2}{L^2\eps'}\bigg) + \frac{12cb^2W\sqrt{\eps'}}{L^4}\bigg)\\
        \lesssim\; & \frac{b^2L^4\eps\delta}{dW^{9/2}b^4\log^4(d/(\eps\delta))(\opt + \eps)^2}\bigg(\frac{W^2\opt}{L^2}\log^2\bigg(\frac{1}{L\eps}\bigg) + \frac{b^2W^4\eps}{L^4}\bigg)\\
        \leq\; & \delta, \end{align*}
    under our choice of $m$ as specified in \eqref{eq:sharp-choice-of-m}.
Thus, we have that with probability at least $1 - \delta$, it holds
    \begin{align*}
        \frac{1}{m} \sum_{i=1}^m (u^t(\w^t\cdot\x\ith) - u^{*t}(\w^t\cdot\x\ith))(\w^t - \w^*)\cdot\x\ith &\geq  \Ex[(u^t(\w^t\cdot\x) - u^{*t}(\w^t\cdot\x))(\w^t - \w^*)\cdot\x] - (\opt + \eps)/b\\
        &\geq -\sqrt{\opt}\|\w^t - \w^*\|_2 - (\opt + \eps)/b,
    \end{align*}
    where in the last inequality we used the fact that
    \begin{align*}
        |\Ex[(u^t(\w^t\cdot\x) - u^{*t}(\w^t\cdot\x))(\w^t - \w^*)\cdot\x]|&\leq \sqrt{\Ex[(u^t(\w^t\cdot\x) - u^{*t}(\w^t\cdot\x))^2]\Ex[((\w^t - \w^*)\cdot\x)^2]}\\
        &\leq \sqrt{\opt}\|\w^t - \w^*\|_2,
    \end{align*}
    since $\Ex[(u^t(\w^t\cdot\x) - u^{*t}(\w^t\cdot\x))^2]\leq \opt$ by \Cref{main:lem:upper-bound-u_t-u_t^*}.

    Therefore, combining the upper bounds on the three terms in \eqref{eq:decompose-(htustrmt(w.x)-htumt(w.x))(wt-ws).x)}, we get that with probability at least $1 - 5\delta$, it holds:
    \begin{equation}\label{eq:claim-inside-sharp-1}
        \frac{1}{m}\sum_{i=1}^m((\htumt(\w^t\cdot\x\ith) - \htustrmt(\w^t\cdot\x\ith))(\w^t - \w^*)\cdot\x\ith\geq -(2\sqrt{\eps} + \sqrt{\opt})\|\w^t - \w^*\|_2 - (3\eps + \opt)/b.
    \end{equation}
Since \eqref{eq:claim-inside-sharp-1} was proved using arbitrary $\eps, \delta > 0,$ it remains to 
    replace $\delta \gets \delta/5$ and $\eps \gets \eps/4$ to complete the proof of \Cref{claim:sharp-inter-1}.
    \end{proof}

\subsection{Proof of \Cref{claim:sharp-inter-2}}

In this subsection, we prove \Cref{claim:sharp-inter-2} that appeared in the proof of \Cref{main:thm:sharpness} in \Cref{subsec:proof-of-sharpness}.

\claimsharptwo*
\begin{proof}
    Before we proceed to the proof of the claim, let us consider first the inverse of $u^*$. Since $u^*(z)\in\U$ is strictly increasing when $z\geq 0$, $(u^*)^{-1}(\alpha)$ exists for $\alpha\geq 0$. However, when $z\leq 0$, $u^*(z)$ could be  constant on some intervals, hence $(u^*)^{-1}(\alpha)$ might not exist for every $\alpha\leq 0$. We consider instead an `empirical' version of $(u^*)^{-1}(\alpha)$ based on $S^*$, which is defined on every $\alpha\in\R$. Given a sample set $S^* = \{(\x\ith,y\sith)\}$ where $y\sith = u^*(\w^*\cdot\x\ith)$, let us sort the index $i$ in the increasing order of $\w^*\cdot\x\ith$, i.e., $\w^*\cdot\x^{(1)}\leq \dots\leq \w^*\cdot\x^{(m)}$. Since $u^*$ is a monotone function, this implies $y\sith$'s are also in increasing order, i.e., we have $y^{*(1)}\leq \dots\leq y^{*(m)}$. We then partition the set $\{y\sith\}_{i=1}^m$ into blocks 
    \begin{equation*}
        \Delta_s = \{y^{*(k_{s-1}+1)},\dots,y^{*(k_s)}\},\;\mathrm{s.t.}\; y^{*(k_{s-1} + 1)} = \dots = y^{*(k_s)} = \tau_s,
    \end{equation*}
    for $s = 1,\dots, s'$. Since $\{y\sith\}$ is sorted in increasing order, we have $\tau_{s-1} < \tau_s$ for $s = 2,\dots,s'$. Note that since $u^*(z)$ is strictly increasing when $z \geq 0$ and as $u^*(0) = 0$, $\Delta_s$ is a singleton set whenever $\tau_s > 0$. Furthermore, let us denote by $s^*$ the largest index among $1,\dots,s'$ such that $\tau_{s^*}\leq 0$.
    
    Suppose first that $\tau_{s^*}<0$ and define a function $\hat{f}:\R\to\R$ in the following way: 
    \begin{equation}
        \hat{f}(\alpha) = \begin{cases}
            (u^*)^{-1}(\alpha), & \alpha > 0 \\
            \w^*\cdot\x^{(k_{s^*})} + \frac{\alpha - \tau_{s^*}}{\tau_{s^*}}(\w^*\cdot\x^{(k_{s^*})}), &\alpha\in [\tau_{s^*}, 0]\\
            \w^*\cdot\x^{(k_s)}, & \alpha = \tau_s, s = 1,\dots, s^* - 1\\
            \w^*\cdot\x^{(k_{s-1})} + \frac{\alpha - \tau_{s-1}}{\tau_s - \tau_{s-1}} (\w^*\cdot\x^{(k_{s-1} + 1)} - \w^*\cdot\x^{(k_{s-1})}), &\alpha\in(\tau_{s-1},\tau_s), s = 2,\dots,s^*\\
            \w^*\cdot\x^{(1)} + \frac{1}{b}(\alpha - \tau_{1})\;.  &\alpha\in (-\infty,\tau_{1})
        \end{cases}
    \end{equation}
    When $\tau_{s^*} = 0$, we define $(0 - \tau_{s^*})/\tau_{s^*} = -1$, and hence $\hat{f}(0) = 0$. The rest remains unchanged. A visualization of $\hat{f}$ with respect to the ReLU activation is presented in \Cref{fig:inver-relu}.

    \begin{figure}[h]
    \centering
    \includegraphics[width=0.5\textwidth]{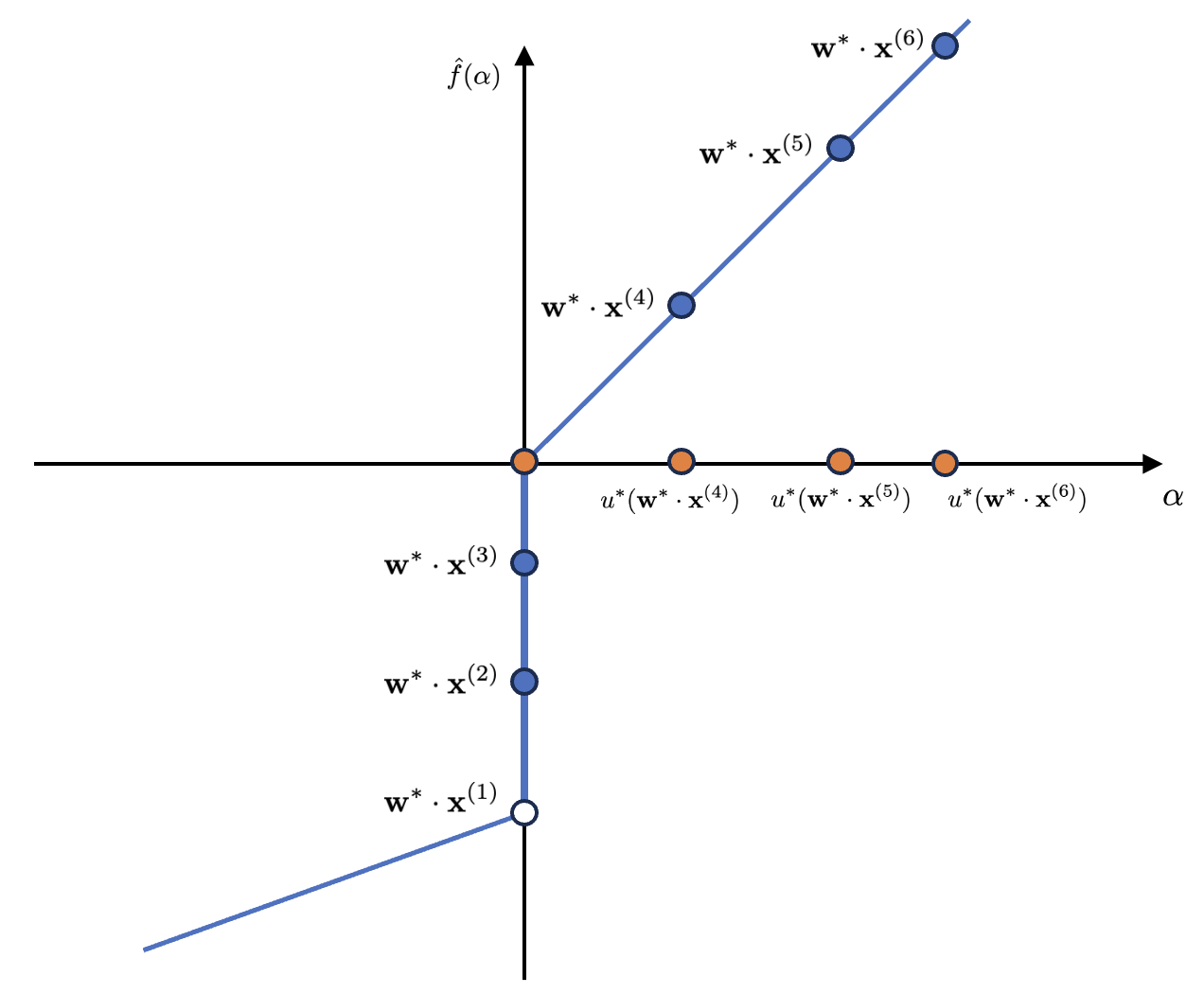}
    \caption{\textit{An illustration of $\hat{f}$ for $u^*(z) = \max\{0,z\}$ and a dataset $S^* =\{(\x^{(1)},u^*(\w^*\cdot\x^{(1)})),\dots,(\x^{(6)},u^*(\w^*\cdot\x^{(6)}))\}$ where $\w^*\cdot\x^{(1)}<\w^*\cdot\x^{(2)}<\w^*\cdot\x^{(3)}<0$.}}
    \label{fig:inver-relu}
\end{figure}

The function $\hat{f}$ has the following properties. First, $\hat{f}(\alpha)$ satisfies $\hat{f}(0) = 0$, $(\alpha_1 - \alpha_2)/a \geq \hat{f}(\alpha_1) - \hat{f}(\alpha_2)$, for all $\alpha_1\geq \alpha_2\geq 0$, since $\hat{f}(\alpha) = (u^*)^{-1}(\alpha)$ when $\alpha > 0$ and $u^*\in\U$. 
Second, $\hat{f}(\alpha_1) - \hat{f}(\alpha_2)\geq (\alpha_1 - \alpha_2)/b$ for all $\alpha_1,\alpha_2\in\R$, $\alpha_1\geq \alpha_2$. This is because each segment of $\hat{f}$ has slope at least $1/b$. 
Third, for any $\alpha \geq u^*(\w^*\cdot\x\ith)$, it holds that $\hat{f}(\alpha) - \w^*\cdot\x\ith\geq (\alpha - u^*(\w^*\cdot\x\ith))/b$. To see this, suppose $u^*(\w^*\cdot\x\ith)\in\Delta_s$. Then for any $\alpha \geq \tau_s$, we have 
$$
\hat{f}(\alpha) - \w^*\cdot\x\ith\geq \hat{f}(\alpha) - \w^*\cdot\x^{(k_s)} = \hat{f}(\alpha) - \hat{f}(\tau_s) = \hat{f}(\alpha) - \hat{f}(u^*(\w^*\cdot\x\ith))\geq (\alpha - u^*(\w^*\cdot\x\ith))/b, 
$$ 
using that the slope of $\hat{f}(\alpha)$ is at least $1/b$. 
On the other hand, when $\alpha < u^*(\w^*\cdot\x\ith)$, we have $\w^*\cdot\x\ith - \hat{f}(\alpha)\geq (u^*(\w^*\cdot\x\ith) - \alpha)/b$. This can be seen similarly from the construction of $\hat{f}$.

Finally, suppose $u^*(\w^*\cdot\x\ith)\in\Delta_s$. Then for any $\alpha < \tau_s$, we have  
$$\w^*\cdot\x\ith - \hat{f}(\alpha)\geq \w^*\cdot\x^{(k_{s-1}+1)} - \hat{f}(\alpha) = \hat{f}(\tau_s) -\hat{f}(\alpha) = \hat{f}(u^*(\w^*\cdot\x\ith)) -\hat{f}(\alpha) \geq (u^*(\w^*\cdot\x\ith) - \alpha)/b.$$ 
Again, we used the fact that $\hat{f}(\alpha_1) - \hat{f}(\alpha_2)\geq (\alpha_1 - \alpha_2)/b$ for all $\alpha_1,\alpha_2\in\R$, $\alpha_1\geq \alpha_2$ in the last inequality.
    
Now we turn to the summation displayed in the statement of the claim. To proceed, we add and subtract $\hat{f}(\htustrmt(\w^t\cdot\x))$ in the second component in the inner product, which yields:
    \begin{align}
        &\quad \frac{1}{m} \sum_{i=1}^m (\htustrmt(\w^t\cdot\x\ith) - u^*(\w^*\cdot\x\ith))(\w^t - \w^*)\cdot\x\ith\nonumber\\
        & = \frac{1}{m} \sum_{i=1}^m (\htustrmt(\w^t\cdot\x\ith) - u^*(\w^*\cdot\x\ith))(\w^t\cdot\x\ith - \hat{f}(\htustrmt(\w^t\cdot\x\ith))) \nonumber\\
        & \quad + \frac{1}{m} \sum_{i=1}^m (\htustrmt(\w^t\cdot\x\ith) - u^*(\w^*\cdot\x\ith))(\hat{f}(\htustrmt(\w^t\cdot\x\ith)) - \w^*\cdot\x\ith). \label{eq:sharp:u^*t-u*.w-w*}
    \end{align}
    To bound below the first term in \eqref{eq:sharp:u^*t-u*.w-w*}, we make use of the following fact, whose proof can be found in \Cref{app:pf:main:lem:hty*-y*.z-f(hty*)>=0}.
\begin{restatable}{fact}{positiveSum}\label{main:lem:hty*-y*.z-f(hty*)>=0}
    Let $\w^t\in\B(W)$. Given $m$ samples $S = \{(\x^{(1)},y^{*(1)}),\cdots,(\x^{(m)},y^{*(m)})\}$, let $\htustrmt$ be one of the solutions to the optimization problem $\eqref{def:htustrmt}$ , i.e., $\htustrmt\in\argmin_{u\in\U} (1/m)\sum_{i=1}^m (u(\w^t\cdot\x\ith) - y\sith)^2.$
Then 
    \begin{equation*}
        \sum_{i=1}^m (\htustrmt(\w^t\cdot\x\ith) - y\sith)(\w^t\cdot\x\ith - f(\htustrmt(\w^t\cdot\x\ith)))\geq 0,
    \end{equation*}
    for any function $f:\R\to\R$ such that $f(0) = 0$, $(\alpha_1 - \alpha_2)/a \geq f(\alpha_1) - f(\alpha_2)$ for all $\alpha_1\geq \alpha_2\geq 0$, and $f(\alpha_1) - f(\alpha_2)\geq (\alpha_1 - \alpha_2)/b$, $\forall \alpha_1,\alpha_2\in\R$, $\alpha_1\geq \alpha_2$.
\end{restatable}
    As we have already argued, $\hat{f}$ satisfies the assumptions of \Cref{main:lem:hty*-y*.z-f(hty*)>=0},  
    hence 
    \begin{equation}\label{eq:sum-hatf(htumt)-wstrx}
        \frac{1}{m} \sum_{i=1}^m (\htustrmt(\w^t\cdot\x\ith) - y\sith)(\w^t - \w^*)\cdot\x\ith\geq \frac{1}{m} \sum_{i=1}^m (\htustrmt(\w^t\cdot\x\ith) - y\sith)(\hat{f}(\htustrmt(\w^t\cdot\x\ith)) - \w^*\cdot\x\ith).
    \end{equation}

    Recall that we have shown the function $\hat{f}$ satisfies $\hat{f}(\alpha) - \w^*\cdot\x\ith\geq (\alpha - u^*(\w^*\cdot\x\ith))/b\geq 0$ whenever $\alpha \geq u^*(\w^*\cdot\x\ith)$, and moreover, $\w^*\cdot\x\ith - \hat{f}(\alpha)\geq (u^*(\w^*\cdot\x\ith) - \alpha)/b\geq 0$ when $\alpha < u^*(\w^*\cdot\x\ith)$. Therefore, letting $\alpha = \htustrmt(\w^t\cdot\x\ith)$ and combining these results we get
$$
    (\htustrmt(\w^t\cdot\x\ith) - u^*(\w^*\cdot\x\ith))(f(\htustrmt(\w^t\cdot\x\ith)) - \w^*\cdot\x\ith)\geq \frac{1}{b}(\htustrmt(\w^t\cdot\x\ith) - u^*(\w^*\cdot\x\ith))^2.
    $$ 
    Plugging the inequality above back into \eqref{eq:sum-hatf(htumt)-wstrx} we then get
    \begin{equation}\label{eq:sharp-inter-2}
        \frac{1}{m} \sum_{i=1}^m (\htustrmt(\w^t\cdot\x\ith) - y\sith)(\w^t - \w^*)\cdot\x\ith\geq \frac{1}{mb} \sum_{i=1}^m (\htustrmt(\w^t\cdot\x\ith) - y\sith)^2 \;.
    \end{equation}
        The goal now is to bound below the right-hand side of \eqref{eq:sharp-inter-2} by $\E[(\htustrmt(\w^t\cdot\x) - y^*)^2]$ and some small error terms using Chebyshev inequality as we did in \Cref{claim:sharp-inter-1}. Plugging in \Cref{main:lem:lower-bound-(f(wx)-u(w*x))^2}, we can further lower bound $\E[(\htustrmt(\w^t\cdot\x) - y^*)^2]$ by $\|(\w^*)^{\perp_{\w^t}}\|_2^2$ and then we are done with the proof of this claim. Note that Chebyshev's inequality yields
    \begin{align}\label{eq:cheby-2}
        \pr\bigg[\bigg|\frac{1}{m} \sum_{i=1}^m (\htustrmt(\w^t\cdot\x\ith) - y\sith)^2 - \Ex[(\htustrmt(\w^t\cdot\x) - y^*)^2]\bigg|\geq s\bigg]\leq \frac{1}{ms^2}\Ex[(\htustrmt(\w^t\cdot\x) - y^*)^4].
    \end{align}
    
    We now bound $\Ex[(\htustrmt(\w^t\cdot\x) - y^*)^4]$. Observe that
    \begin{align}
        \Ex[(\htustrmt(\w^t\cdot\x) - y^*)^4]& = \Ex[(\htustrmt(\w^t\cdot\x) - u^{*t}(\w^t\cdot\x) + u^{*t}(\w^t\cdot\x) - y^*)^2(\htustrmt(\w^t\cdot\x) - y^*)^2]\nonumber\\
        &\leq 4\Ex[(\htustrmt(\w^t\cdot\x) - u^{*t}(\w^t\cdot\x))^2((\htustrmt(\w^t\cdot\x))^2 + (y^*)^2)]\nonumber\\
        &\quad + 4\Ex[(u^{*t}(\w^t\cdot\x) - y^*)^2((\htustrmt(\w^t\cdot\x))^2 + (y^*)^2)].\label{eq:sharp-inter-3}
    \end{align}
    We focus on the two terms in \eqref{eq:sharp-inter-3} separately. Again, choosing $r = \frac{2W}{L}\log(2/(L^2\eps'))$, then by the $L$-sub-exponential tail bound of $\D_\x$, it holds $\pr[|\w^t\cdot\x|\geq r]\leq \eps', \pr[|\w^*\cdot\x|\geq r]\leq \eps'$. Since $y^* = u^*(\w^*\cdot\x)$ and both $u^*$ and $\htustrmt$ are non-decreasing $b$-Lipschitz, it holds:
    \begin{align}\label{eq:sharp-inter-4}
        &\quad \Ex[(\htustrmt(\w^t\cdot\x) - u^{*t}(\w^t\cdot\x))^2((\htustrmt(\w^t\cdot\x))^2 + (y^*)^2)]\nonumber\\
        &\leq b^2\Ex[(\htustrmt(\w^t\cdot\x) - u^{*t}(\w^t\cdot\x))^2((\w^t\cdot\x)^2 + (\w^*\cdot\x)^2)]\nonumber\\
        & = b^2\Ex[(\htustrmt(\w^t\cdot\x) - u^{*t}(\w^t\cdot\x))^2((\w^t\cdot\x)^2 + (\w^*\cdot\x)^2)\1\{|\w^t\cdot\x|\leq r, |\w^*\cdot\x|\leq r\}]\nonumber\\
        &\quad + b^2\Ex[(\htustrmt(\w^t\cdot\x) - u^{*t}(\w^t\cdot\x))^2((\w^t\cdot\x)^2 + (\w^*\cdot\x)^2)\1\{|\w^t\cdot\x|\geq r \;\mathrm{ or }\; |\w^*\cdot\x|\geq r\}]\nonumber\\
        &\leq 2b^2r^2\Ex[(\htustrmt(\w^t\cdot\x) - u^{*t}(\w^t\cdot\x))^2]\nonumber\\
        &\quad + 2b^4\Ex[2(\w^t\cdot\x)^2((\w^t\cdot\x)^2 + (\w^*\cdot\x)^2)\1\{|\w^t\cdot\x|\geq r\; \mathrm{or}\; |\w^*\cdot\x|\geq r\}].
    \end{align}
The first term in \eqref{eq:sharp-inter-4} can be upper bounded using \Cref{main:lem:E[htustrm(wt.x) - u*t(wt.x))^2]<=eps}, which states that when 
    $$m\gtrsim d\log(1/\delta)(b^2W^3\log^2(d/\eps)/(L^2\eps))^{3/2}),$$ 
    with probability at least $1 - \delta$ it holds $\Ex[(\htustrmt(\w^t\cdot\x) - u^{*t}(\w^t\cdot\x))^2]\leq \eps$ for all $\w^t\in\B(W)$. Now suppose this inequality is valid  given $\w^t\in\B(W)$ (which happens with probability at least $1 - \delta$). For the second term in \eqref{eq:sharp-inter-4}, note that for any unit vector $\ba$ it holds $\Ex[(\ba\cdot\x)^8]\leq c^2/L^8$ for some absolute constant $c > 0$, and furthermore, the magnitude of $r$ ensures that $\pr[|\w^t\cdot\x|\geq r\; \mathrm{or}\; |\w^*\cdot\x|\geq r]\leq 2\eps'$; therefore, combining these bounds, we get:
    \begin{align*}
        &\quad \Ex[2(\w^t\cdot\x)^2((\w^t\cdot\x)^2 + (\w^*\cdot\x)^2)\1\{|\w^t\cdot\x|\geq r\; \mathrm{or}\;  |\w^*\cdot\x|\geq r\}]\\
        &\leq 2\sqrt{\Ex[(\w^t\cdot\x)^8]\pr[|\w^t\cdot\x|\geq r\; \mathrm{or}\;  |\w^*\cdot\x|\geq r]}\\
        &\quad + 2\sqrt{2(\Ex[(\w^t\cdot\x)^8] + \Ex[(\w^*\cdot\x)^8])\pr[|\w^t\cdot\x|\geq r\; \mathrm{or}\;  |\w^*\cdot\x|\geq r]}\\
        &\leq 24c(W/L)^4\sqrt{\eps'}.
    \end{align*}
Plugging back into \eqref{eq:sharp-inter-4}, we have
    \begin{equation*}
        \Ex[(\htustrmt(\w^t\cdot\x) - u^{*t}(\w^t\cdot\x))^2((\htustrmt(\w^t\cdot\x))^2 + (y^*)^2)]\leq 2b^2r^2\eps + 48c(bW/L)^4\sqrt{\eps'},
    \end{equation*}
    which is the upper bound on the first term of \eqref{eq:sharp-inter-3}.
    
    For the second term in \eqref{eq:sharp-inter-3}, since by  definition we have $u^{*t}\in\argmin_{u\in\U}\Ex[(u(\w^t\cdot\x) - y^*)^2]$, it holds that
    \begin{equation*}
        \Ex[(u^{*t}(\w^t\cdot\x) - y^*)^2]\leq \Ex[(u^*(\w^t\cdot\x) - u^*(\w^*\cdot\x))^2]\leq b^2\Ex[((\w^t - \w^*)\cdot\x)^2]\leq b^2\|\w^t - \w^*\|_2^2,
    \end{equation*}
   noting in addition that $\Ex[\x\x^\top] \preccurlyeq \mathbf{I}$. Thus, using similar steps as in \eqref{eq:sharp-inter-4}, we have
    \begin{align*}
        &\quad \Ex[(u^{*t}(\w^t\cdot\x) - y^*)^2((\htustrmt(\w^t\cdot\x))^2 + (y^*)^2)]\\
        &\leq 2b^2r^2\Ex[(u(\w^t\cdot\x) - y^*)^2]\\
        &\quad + 2b^4\Ex[2((\w^t\cdot\x)^2 + (\w^*\cdot\x)^2)^2\1\{|\w^t\cdot\x|\geq r\; \mathrm{or}\;  |\w^*\cdot\x|\geq r\}]\\
        &\leq 2b^4r^2\|\w^t - \w^*\|_2^2 + 48c(bW/L)^4\sqrt{\eps}.
    \end{align*}
    In summary, combining all the results and plugging them back into \eqref{eq:sharp-inter-3}, we finally get the upper bound for the variance:
    \begin{equation*}
        \Ex[(\htustrmt(\w^t\cdot\x) - y^*)^4]\leq \frac{32b^2W^2}{L^2}\log^2(2/(L^2\eps'))(b^2\|\w^t - \w^*\|_2^2 + \eps) + 384c(bW/L)^4\sqrt{\eps'}.
    \end{equation*}
    Let $s = b\sqrt{\eps}\|\w^t - \w^*\|_2 + \eps/b$ and plug the last inequality back into \eqref{eq:cheby-2} to get:
    \begin{align*}
        \pr\bigg[\bigg|\frac{1}{m} &\sum_{i=1}^m (\htustrmt(\w^t\cdot\x\ith) - y\sith)^2 - \Ex[(\htustrmt(\w^t\cdot\x) - y^*)^2]\bigg|\geq b\sqrt{\eps}\|\w^t - \w^*\|_2 + \eps/b\bigg]\nonumber\\
        &\leq \frac{1}{m(\eps b^2\|\w^t - \w^*\|_2^2 + \eps^2/b^2)}\bigg(\frac{32b^2W^2}{L^2}\log^2\bigg(\frac{2}{L^2\eps'}\bigg)(b^2\|\w^t - \w^*\|_2^2 + \eps) + 384c(bW/L)^4\sqrt{\eps'}\bigg).
    \end{align*}
    Choosing $\eps' = \eps^2/b^4$ and using similar arguments as in \Cref{claim:sharp-inter-1}, we get that the right-hand side of the inequality above is bounded by $\delta$, given our choice of $m \gtrsim  db^4W^{9/2}\log^4(d/(\eps\delta))(1/\eps^{3/2} + 1/(\eps\delta))$ as specified in the statement of \Cref{main:thm:sharpness}. In summary, after a union bound on the probability above and the event that $\Ex[(\htustrmt(\w^t\cdot\x) - u^{*t}(\w^t\cdot\x))^2]\leq \eps$, we have with probability at least $1 - 2\delta$, 
\begin{equation*}
        \frac{1}{m} \sum_{i=1}^m (\htustrmt(\w^t\cdot\x\ith) - y\sith)^2\geq \Ex[(\htustrmt(\w^t\cdot\x) - y^*)^2] - \sqrt{\eps}b\|\w^t - \w^*\|_2 - \eps/b.
    \end{equation*}
Recall that in \Cref{main:lem:lower-bound-(f(wx)-u(w*x))^2} we showed that $\Ex[(\htustrmt(\w^t\cdot\x) - u^*(\w^*\cdot\x))^2]\geq Ca^2LR^4\|(\w^*)^{\perp_{\w^t}}\|_2^2$ for an absolute constant $C$; thus, our final result is that with probability at least $1 - \delta$,
    \begin{align*}
        \frac{1}{m} \sum_{i=1}^m (\htustrmt(\w^t\cdot\x\ith) - y\sith)(\w^t - \w^*)\cdot\x\ith&\geq \frac{1}{mb} \sum_{i=1}^m (\htustrmt(\w^t\cdot\x\ith) - y\sith)^2\\
        &\geq \frac{Ca^2LR^4}{b}\|(\w^*)^{\perp_{\w^t}}\|_2^2 - \sqrt{\eps}\|\w^t - \w^*\|_2 - \eps/b.
    \end{align*}
    This completes the proof of \Cref{claim:sharp-inter-2}.
    \end{proof}

\subsection{Proof of \Cref{main:lem:hty*-y*.z-f(hty*)>=0}}\label{app:pf:main:lem:hty*-y*.z-f(hty*)>=0}

We prove a modified version of Lemma 1~\cite{kakade2011efficient},
presented as the statement below. The statement considers a smaller activation class and a function $f$ with different properties compared to~\cite{kakade2011efficient}, and the proof is based on a rigorous KKT argument.

\positiveSum*

\begin{proof}
    We transform the optimization problem \eqref{def:htustrmt} to a quadratic optimization problem with linear constraints. To guarantee that the solution of this quadratic problem corresponds to a function that is $(a,b)$-unbounded, we add a sample $(\x^{(k)},y^{*(k)}) = (\vec 0, 0)$ to the sample set. Let $z_i = \w^t\cdot\x\ith$ such that (after sorting the indices) $z_1\leq z_2\leq\cdots\leq z_m$ and $z_k = 0$. We solve the following optimization problem: 
    \begin{equation}\label{eq:opt-def-of-hat-y*}
    \begin{aligned}
        \min_{\ty\ith, i\in [m]} &\; \sum_{i = 1}^m (\ty\ith - y\sith)^2\\
        \mathrm{s.t.}\;\; &\; 0\leq \ty^{(i+1)} - \ty\ith ,\; &&1\leq i\leq k - 1 \; ,\\
        &\; a(z_{i+1} - z_i)\leq \ty^{(i+1)} - \ty\ith ,\; &&k\leq i\leq m-1 \; ,\\
        &\;\ty^{(i+1)} - \ty\ith \leq b(z_{i+1} - z_i), \; &&1\leq i\leq m-1 \;,\\
        &\; \ty^{(k)} = 0\;.
    \end{aligned}
    \end{equation}
    Denote the solution of \eqref{eq:opt-def-of-hat-y*} as $\hat{y}\sith$, $i = 1,\cdots,m$. Let $\htustrmt(z)$ be the linear interpolation function of $(z_i,\hat{y}\sith)$, then $\htustrmt\in\U$ since $\htustrmt(0) = \htustrmt(z_k) = \hat{y}^{*(k)} = 0$, $\htustrmt$ is $b$-Lipschitz and $\htustrmt(z) - \htustrmt(z')\geq a(z - z')$ for all $z \geq z'\geq 0$. In other words, finding a solution of \eqref{def:htustrmt} is equivalent to solving \eqref{eq:opt-def-of-hat-y*}.
    
    Now observe that the summation $\sum_{i=1}^m (\hat{y}\sith - y\sith)(z_i - f(\hat{y}\sith))$ can be transformed into the following:
    \begin{equation}\label{eq:sum-(haty-y*i)(z-f(haty))}
        \sum_{i=1}^m (\hat{y}\sith - y\sith)(z_i - f(\hat{y}\sith)) = \sum_{i=1}^m \bigg(\sum_{j=1}^i (\hat{y}^{*(j)} - y^{*(j)})\bigg)(z_i - f(\hat{y}\sith) - (z_{i+1} - f(\hat{y}^{*(i+1)}))),
    \end{equation}
    where we let $z_{m+1} = 0$, $\htystr_{m+1} = 0$ (and hence $f(\hat{y}_{m+1}^*) = 0$ as $f(0)=0$).
    
    To utilize the information that $\hat{y}\sith$ is the minimizer of the optimization problem \eqref{eq:opt-def-of-hat-y*}, we write down the KKT conditions for the optimization problem \eqref{eq:opt-def-of-hat-y*} described above:
    \begin{align}
        &\hat{y}\sith = y\sith + (\lambda'_i - \lambda'_{i-1})/2 - (\lambda_i - \lambda_{i-1})/2 - (\nu_k/2)\1\{i = k\}, \; && i = 1,\cdots, m; \label{eq:kkt-1}\\
        &-\lambda_i(\hat{y}^{*(i+1)} - \hat{y}\sith) = 0,\; && i=1,\cdots,k-1;\label{eq:csc-0.5}\\
        &\lambda_i(a(z_{i+1} - z_i) - (\hat{y}^{*(i+1)} - \hat{y}\sith)) = 0, \; && i = k,\cdots, m-1; \label{eq:csc-1}\\
        &\lambda_i'((\hat{y}^{*(i+1)} - \hat{y}^{*(i)}) - b(z_{i+1} - z_i)) = 0, \; && i = 1,\cdots, m-1; \label{eq:csc-2}\\
        &\nu_k\hat{y}^{*(k)} = 0 \;,
    \end{align}
    where $\lambda_i,\lambda_i'\geq 0,$ for $i = 1,\dots,m-1,$ and $\nu_k\in\R$ are dual variables, and we let $\lambda_{0} = \lambda'_0= 0$ for the convenience of presenting \eqref{eq:kkt-1}. 
    
    Summing up \eqref{eq:kkt-1} recursively, we immediately get that
    \begin{equation*}
        \sum_{j=1}^i (\hat{y}\sith - y\sith) =\frac{1}{2}((\lambda_i' - \lambda_i) - \nu_k\1\{i\geq k\}).
    \end{equation*}
    Plugging the equality above back into \eqref{eq:sum-(haty-y*i)(z-f(haty))}, we have
    \begin{align}\label{eq:after-kkt-1}
        &\quad\sum_{i=1}^m (\hat{y}\sith - y\sith)(z_i - f(\hat{y}\sith))\nonumber\\
        &= \frac{1}{2}\sum_{i=1}^{m} (\lambda'_i - \lambda_i) (z_i - f(\hat{y}\sith) - (z_{i+1} - f(\hat{y}^{*(i+1)}))) + \frac{1}{2}\sum_{i=k}^m \nu_k(z_i - f(\hat{y}\sith) - (z_{i+1} - f(\hat{y}^{*(i+1)})))\nonumber\\
        & = \frac{1}{2}\sum_{i=1}^{m} (\lambda'_i - \lambda_i) (z_i - f(\hat{y}\sith) - (z_{i+1} - f(\hat{y}^{*(i+1)}))) + \nu_k(z_k - f(\hat{y}^{*(k)}) - (z_{m+1} - f(\hat{y}^{*(m+1)}))).
    \end{align}
    Since by definition, $z_{m+1} = f(\hat{y}^{*(m+1)}) = 0$, $z_k = 0$, and as $\hat{y}\sith$, $i\in[m]$, is a feasible solution of \eqref{eq:opt-def-of-hat-y*}, it holds $\hat{y}^{*(k)} = 0$, we thus have
    \begin{equation*}
        \nu_k(z_k - f(\hat{y}^{*(k)}) - (z_{m+1} - f(\hat{y}^{*(m+1)}))) = 0.
    \end{equation*}
    Plugging this back into \eqref{eq:after-kkt-1}, we get
    \begin{equation}\label{eq:after-kkt-2}
    \begin{aligned}
        \sum_{i=1}^m (\hat{y}\sith - y\sith)(z_i - f(\hat{y}\sith))  &= \frac{1}{2}\sum_{i=1}^m (\lambda'_i - \lambda_i)(z_i - f(\hat{y}\sith) - (z_{i+1} - f(\hat{y}^{*(i+1)})))\\
        &= \underbrace{\frac{1}{2}\sum_{i=1}^{k-1} (\lambda'_i - \lambda_i)(z_i - f(\hat{y}\sith) - (z_{i+1} - f(\hat{y}^{*(i+1)})))}_{S_1}\\
        &\quad + \underbrace{\frac{1}{2}\sum_{i=k}^m (\lambda'_i - \lambda_i)(z_i - f(\hat{y}\sith) - (z_{i+1} - f(\hat{y}^{*(i+1)})))}_{S_2}.
        \end{aligned}
    \end{equation}
    
    Consider first $S_1$. Suppose that for some $i\in\{1,\dots,k-1\}$ we have $\lambda_i',\lambda_i>0$. Then, according to the complementary slackness condition \eqref{eq:csc-0.5} and \eqref{eq:csc-1}, it holds that $0 = \hat{y}^{*(i+1)} - \hat{y}\sith = b(z_{i+1} - z_i)$. Therefore, 
    $$
    (\lambda'_i - \lambda_i)(z_i - f(\hat{y}\sith) - (z_{i+1} - f(\hat{y}^{*(i+1)}))) \geq 0.
    $$ 
    Suppose now that for some $i\in\{1,\dots, k-1\}$, it holds $\lambda_i'>0, \lambda = 0$. Then, it must be the case that $$\hat{y}^{*(i+1)} - \hat{y}\sith = b(z_{i+1} - z_i)\geq 0,$$ according to the KKT condition \eqref{eq:csc-2}. Since $$f(\hat{y}^{*(i+1)}) - f(\hat{y}\sith)\geq (\hat{y}^{*(i+1)} - \hat{y}\sith)/b$$ by assumption on $f$, we thus have $$(z_i - f(\hat{y}\sith) - (z_{i+1} - f(\hat{y}^{*(i+1)})))\geq 0.$$ 
    Finally, if $\lambda_i>0, \lambda' = 0$, then \eqref{eq:csc-0.5} indicates that $0 = \hat{y}^{*(i+1)} - \hat{y}\sith$. Therefore, as $z_{i+1}\geq z_i$, the $i^{\mathrm{th}}$ summand is also positive. In summary, $S_1 \geq 0$.
    
    Now consider $S_2$. Observe that if for some $i\in\{k,\dots,m\}$ it holds $\lambda_i>0$ and $\lambda_i'>0$ at the same time, then KKT conditions \eqref{eq:csc-1} and \eqref{eq:csc-2} imply that $$a(z_{i+1} - z_i) = \hat{y}^{*(i+1)}- \hat{y}\sith = b(z_{i+1} - z_i),$$
    as $a<b$ and it has to be $z_{i+1} - z_i = \hat{y}^{*(i+1)}- \hat{y}\sith = 0$, which indicates that the $i^\mathrm{th}$ summand in the second term must be 0, i.e., $$(\lambda'_i - \lambda_i)(z_i - f(\hat{y}\sith) - (z_{i+1} - f(\hat{y}^{*(i+1)}))) = 0.$$
Now suppose for some $i\in\{1,\dots, m\}$, $\lambda'_i>0$ and $\lambda_i = 0$. Then by the complementary slackness conditions \eqref{eq:csc-1} and \eqref{eq:csc-2}, it must be that $$\hat{y}^{*(i+1)} - \hat{y}\sith = b(z_{i+1} - z_i)\geq 0.$$ 
    Again, since $f$ satisfies $$f(\hat{y}^{*(i+1)}) - f(\hat{y}\sith)\geq (\hat{y}^{*(i+1)} - \hat{y}\sith)/b$$ 
    for any $\hat{y}^{*(i+1)} \geq \hat{y}\sith$, we thus have $$z_i - z_{i+1} + (f(\hat{y}^{*(i+1)}) - f(\hat{y}\sith))\geq 0.$$ 
    Thus, it holds that $$(\lambda'_i - \lambda_i)(z_i - f(\hat{y}\sith) - (z_{i+1} - f(\hat{y}^{*(i+1)})))\geq 0.$$
On the other hand, if $\lambda'_i = 0$ and $\lambda_i > 0$, then complementary slackness implies that $$\hat{y}^{*(i+1)} - \hat{y}\sith = a(z_{i+1} - z_i)\geq 0.$$ 
    Furthermore, since $\hat{y}\sith\geq \hat{y}^{*(k)} \geq 0$ when $i\geq k$, using the assumption that $$(\alpha_1 - \alpha_2)/a \geq f(\alpha_1) - f(\alpha_2)$$ 
    when $\alpha_1 \geq \alpha_2\geq 0$, we get $$z_i - z_{i+1} + (f(\hat{y}^{*(i+1)}) - f(\hat{y}\sith))\leq 0,$$ and hence $$(\lambda'_i - \lambda_i)(z_i - f(\hat{y}\sith) - (z_{i+1} - f(\hat{y}^{*(i+1)})))\geq 0$$ holds as well. Thus we conclude that $S_2 \geq 0.$
    
    In summary, since each summand in \eqref{eq:after-kkt-2} is non-negative, we finally get that
    \begin{equation*}
        \sum_{i=1}^m (\hat{y}\sith - y\sith)(z_i - f(\hat{y}\sith)) = \sum_{i=1}^m \bigg(\sum_{j=1}^i (\hat{y}^{*(j)} - y^{*(j)})\bigg)(z_i - f(\hat{y}\sith) - (z_{i+1} - f(\hat{y}^{*(i+1)})))\geq 0.
    \end{equation*}
    This completes the proof of \Cref{main:lem:hty*-y*.z-f(hty*)>=0}.
\end{proof}

\subsection{Proof of \Cref{claim:sharp-inter-3}}

We restate and prove \Cref{claim:sharp-inter-3} that appeared in the proof of \Cref{main:thm:sharpness} in \Cref{subsec:proof-of-sharpness}.

\claimsharpthree*
\begin{proof}
    By Chebyshev's inequality, we can write
    \begin{align*}
        \pr\bigg[\bigg|\frac{1}{m}\sum_{i=1}^m (y\sith- y\ith)(\w^t\cdot\x\ith - \w^*\cdot\x\ith) &- \Exy[(y^* - y)(\w^t - \w^*)\cdot\x]\bigg|\geq s\bigg]\\
        &\leq \frac{\Exy[(y^* - y)^2(\w^t\cdot\x - \w^*\cdot\x)^2]}{ms^2}.
    \end{align*}
    Let $r = \frac{2W}{L}\log(2/(L^2\eps'))$, then by the fact that $\D_\x$ is sub-exponential, we have $\pr[|(\w^t - \w^*)\cdot\x|\geq r]\leq \eps'$. Furthermore, since $|y|\leq M$ where $M = \frac{bW}{L}\log(16b^4W^4/\eps^2)$, as stated in \Cref{main:lem:y-bounded-by-M}, the variance can be bounded as follows:
    \begin{align*}
        &\quad \Exy[(y^* - y)^2(\w^t\cdot\x - \w^*\cdot\x)^2]\\
        &\leq \Exy[(y^* - y)^2(\w^t\cdot\x - \w^*\cdot\x)^2\1\{|(\w^t - \w^*)\cdot\x|\leq r\}]\\
        &\quad + \Exy[(y^* - y)^2(\w^t\cdot\x - \w^*\cdot\x)^2\1\{|(\w^t - \w^*)\cdot\x|\geq r\}]\\
        &\leq r^2\Exy[(u^*(\w^*\cdot\x) - y)^2]\\
        &\quad + \Exy[(2(u^*(\w^*\cdot\x))^2 + y^2)(\w^t\cdot\x - \w^*\cdot\x)^2\1\{|(\w^t - \w^*)\cdot\x|\geq r\}]\\
        &\leq r^2\opt + \Ex[2(b^2(\w^t\cdot\x)^2 + M^2)(\w^t\cdot\x - \w^*\cdot\x)^2\1\{|(\w^t - \w^*)\cdot\x|\geq r\}].
    \end{align*}
    Since for any unit vectors $\ba,\bb$ we have $\Ex[(\ba\cdot\x)^4]\leq c^2/L^4$ and $\Ex[(\ba\cdot\x)^4(\bb\cdot\x)^4]\leq c^2/L^8$, we have:
    \begin{align*}
        &\quad 2b^2\Ex[(\w^t\cdot\x)^2(\w^t\cdot\x - \w^*\cdot\x^2)^2\1\{|(\w^t - \w^*)\cdot\x|\geq r\}]\\
        &\leq 4b^2(W/L)^4\sqrt{\Ex[((\w^t/\|\w^t\|_2)\cdot\x)^4(((\w^t - \w^*)/\|\w^t - \w^*\|_2)\cdot\x)^4]\pr[|(\w^t - \w^*)\cdot\x|\geq r]}\\
        &\leq 4cb^2(W/L)^4\sqrt{\eps'},
    \end{align*}
    and in addition,
    \begin{align*}
        &\quad \Ex[M^2((\w^t - \w^*)\cdot\x)^2\1\{|(\w^t - \w^*)\cdot\x|\geq r\}]\\
        &\leq 2M^2W^2\sqrt{\Ex[((\w^t - \w^*)\cdot\x)^4]\pr[|(\w^t - \w^*)\cdot\x|\geq r]} \leq cM^2(W/L)^2\sqrt{\eps'}.
    \end{align*}
    Let $s = (\opt + \eps)/b$, $\eps' = \eps^2$, under our choice of $m \gtrsim  db^4W^{9/2}\log^4(d/(\eps\delta))(1/\eps^{3/2} + 1/(\eps\delta))$, it holds that
    \begin{equation*}
        \frac{1}{ms^2}\bigg(\frac{4W^2\log^2(1/(L^2\eps'))\opt}{L^2} + (4cb^2(W/L)^4 + cM^2(W/L)^2)\sqrt{\eps'}\bigg)\leq \delta. \end{equation*}
    Thus, with probability at least $1 - \delta$ it holds that
    \begin{equation*}
        \frac{1}{m}\sum_{i=1}^m (y\sith- y\ith)(\w^t\cdot\x\ith - \w^*\cdot\x\ith)\geq \Exy[(y-y^*)(\w^t - \w^*)\cdot\x] - (\opt + \eps)/b.
    \end{equation*}
    Since 
    \begin{equation*}
        \bigg|\Exy[(y-y^*)(\w^t - \w^*)\cdot\x]\bigg|\leq \sqrt{\Exy[(y-y^*)^2]\Ex[((\w^t - \w^*)\cdot\x)^2]}\leq \sqrt{\opt}\|\w^* - \w^t\|_2,
    \end{equation*}
    we finally have
    \begin{equation*}\frac{1}{m}\sum_{i=1}^m (y\sith- y\ith)(\w^t\cdot\x\ith - \w^*\cdot\x\ith)\geq - \sqrt{\opt}\|\w^* - \w^t\|_2 - (\opt + \eps)/b,
    \end{equation*}
    completing the proof of \Cref{claim:sharp-inter-3}.
    \end{proof}

\section{Omitted Proofs from \Cref{sec:optimization}}\label{app:optimization}

\subsection{Proof of \Cref{thm:fast-converge-main}}\label{app:pf:thm:fast-converge-main}
In this subsection, we restate and prove our main theorem \Cref{thm:fast-converge-main}. The full version of the optimization algorithm as well as the main theorem \Cref{thm:fast-converge-main} is displayed below:
\begin{algorithm}[ht]
   \caption{Optimization}
   \label{alg:optimization}
\begin{algorithmic}[1]
   \STATE {\bfseries Input:} $\w^{\mathrm{ini}} = \vec 0$; $\eps>0$; positive parameters: $a$, $b$, $L$, $R$, $W$; let $\mu \lesssim a^2LR^4/b$; step size $\eta = \mu/(4b^2)$, number of iterations $T = O((b/\mu)^2\log(1/\eps))$. \STATE $\{\w^\mathrm{ini}_0,\dots,\w^\mathrm{ini}_{t_0}\} = \text{Initialization}[\w^{\mathrm{ini}}]$ (\Cref{alg:initialization})
\FOR{$k = 0$ {\bfseries to} $t_0 \lesssim (b/\mu)^6\log(b/\mu)$} \label{line:outestloop}
\STATE $\mathcal{P}_k = \{\}$
\FOR{$j = 1$ {\bfseries to} $J = W/(\eta\sqrt{\eps})$}
\STATE $\barw^{0}_{j,k} = \w^\mathrm{ini}_k$.
\STATE $\frkwstr_j = j\eta\sqrt{\eps}$. \hfill {\tt $\triangleright$ find an $\eta\sqrt{\eps}$ approximation of $\|\w^*\|_2$}
\FOR{$t=0$ {\bfseries to} $T-1$}
\STATE $\htw^t_{j,k} = \frkwstr_{j}(\barw^t_{j,k}/\|\barw^t_{j,k}\|_2)$. \hfill {\tt $\triangleright$ normalize $\barw$}\label{line:initialize-w-update}
\STATE Draw $m \gtrsim W^{11/2}b^{17} \log^5(d/\eps) d/(L^4\mu^{12}\eps^{3/2})$ new i.i.d. samples from $\D$
\STATE $\hat{u}^t_{j,k} = \argmin_{u\in\U}(1/m)\sum_{i=1}^m (u(\htw^t_{j,k}\cdot\x\ith) - y\ith)^2$.
\STATE $\nabla\htLsur(\htw^t_{j,k};\hat{u}^t_{j,k}) = (1/m)\sum_{i=1}^m (\hat{u}^t_{j,k}(\htw^t_{j,k}\cdot\x\ith) - y\ith)\x\ith$. \label{line:empirical-grad}
\STATE $\barw^{t+1}_{j,k} = \htw^t_{j,k} - \eta\nabla\htLsur(\htw^t_{j,k};\hat{u}^t_{j,k})$
\ENDFOR
\STATE $\mathcal{P}_k\gets \mathcal{P}_k\cup \{(\htw^{T}_{j,k}; \hat{u}^T_{j,k})\}$. \ENDFOR
\STATE $\mathcal{P} = \cup_{k=1}^{t_0} \mathcal{P}_k\cup\{(\w = 0; u(z) = 0)\}$ \label{line:P_k} 
\ENDFOR
\STATE $(\htw; \hat{u}) = \text{Test}[(\w;u)\in\mathcal{P}]$ (\Cref{alg:testing})\label{line:testing} \hfill {\tt $\triangleright$ testing}\STATE {\bfseries Return:} $(\htw; \hat{u})$ \end{algorithmic}
\end{algorithm}

\begin{theorem}[Main Result]\label{app:thm:fast-converge-main}
Let $\D$ be a distribution in $\R^d\times \R$ and suppose that $\D_\x$ is $(L,R)$-well-behaved. Furthermore, let $\U$ be as in \Cref{def:well-behaved-unbounded-intro}, and $\eps> 0$. 
    Let $\mu = Ca^2LR^4/b$, where $C$ is an absolute constant. Running \Cref{alg:optimization} with the following parameters: step size $\eta = \mu/(4b^2)$, batch size to be
        $m \gtrsim dW^{11/2}b^{17}\log^5(d/\eps)/(L^4\mu^{12}\eps^{3/2})$ and the total number of iterations to be $T' = t_0JT = O({Wb^{11}}/{(\mu^{10}\sqrt{\eps})}\log(1/\eps))$, where $T = O((b/\mu)^2\log(1/\eps))$, then with probability at least $2/3$,  \Cref{alg:optimization} returns a hypothesis $(\hat{u},\htw)$ where $\hat{u}\in\U$ and $\htw\in\B(W)$
 such that
    \begin{equation*}
        \Ltwo(\htw;\hat{u}) = O\bigg(\frac{b^4}{a^4L^2R^{8}}\bigg) \opt + \eps\;,
    \end{equation*}
   using $N = O(T'm) = \Tilde{O}(dW^{13/2}b^{28}/(L^4\mu^{22}\eps^2))$ samples.
\end{theorem}

\begin{proof}As proved in \Cref{main:lem:initialization}, the initialization subroutine \Cref{alg:initialization} outputs a {list of points $\{\w^{\mathrm{ini}}_{k}\}_{k=1}^{t_0}$ that contains a point $\w^{\mathrm{ini}}_{k^*}$} such that $$\|(\w^*)^{\perp_{\w^{\mathrm{ini}}_{k^*}}}\|_2\leq \max\{\mu\|\w^*\|_2/(4b), {64b^2}{/\mu^3}(\sqrt{\opt} + \sqrt{\eps})\}.$$ 
Suppose first that $\mu\|\w^*\|_2/(4b)\leq {64b^2}{/\mu^3}(\sqrt{\opt} + \sqrt{\eps})$. Then this implies that $\|\w^*\|_2\leq 256b^3/\mu^4(\sqrt{\opt} + \sqrt{\eps})$. Therefore, applying \Cref{main:lem:L2-error-upbd-||w - w^*||^2} we immediately get that the trivial hypothesis $(\w = 0, u(z) = 0)$ works as a constant approximate solution, as in this case $$\Ltwo(\w;u)\leq 8(\opt + \eps) + 4b^2\|\w^*\|_2 = O((b/\mu)^8)\opt + \eps.$$ 
This hypothesis $(\w = 0, u(z) = 0)$ is contained in our solution set $\mathcal{P}$ (see \Cref{line:P_k}) and tested in \Cref{alg:testing}.

Thus, in the rest of the proof we assume that $\w^{\mathrm{ini}}_{k^*}$ satisfies $$\|(\w^*)^{\perp_{\w^{\mathrm{ini}}_{k^*}}}\|_2\leq \mu\|\w^*\|_2/(4b).$$ 
{Let us consider this initialized parameter at $k^*$ step in the outer loop (line \ref{line:outestloop}), $\barw^0_{j,k^*} = \w^{\mathrm{ini}}_{k^*}$. In the rest of the proof we drop  the subscript $k^*$ since the context is clear. }

Since we constructed a grid with grid width  $\eta\sqrt{\eps}$ from $0$ to $W$ to find the (approximate) value of $\|\w^*\|_2$, there must exist an index $j^*$ such that the value of  $\frkwstr_{j^*}$ is $\eta\sqrt{\eps}$ close to $\|\w^*\|_2$, i.e., $|\frkwstr_{j^*} - \|\w^*\|_2|\leq \eta\sqrt{\eps}$. We now consider this $j^{*\mathrm{th}}$ outer loop and ignore the subscript $j^*$ for simplicity. Let $\w^t = \|\w^*\|_2(\barw^t/\|\barw^t\|_2)$, which is the true normalized vector of $\barw^t$ that has no error.
    
    We study the squared distance between $\barw^{t+1}$ and $\w^*$:
    \begin{align}
        \|\barw^{t+1} - \w^*\|_2^2 &= \|\htw^t - \eta\nabla\htLsur(\htw^t;\htumt) - \w^*\|_2^2 \nonumber\\
        &= \|\htw^t - \w^*\|_2^2 + \eta^2\|\nabla\htLsur(\htw^t;\htumt)\|_2^2 - 2\eta \nabla\htLsur(\htw^t;\htumt)\cdot(\htw^t - \w^*).\label{eq:||w^t+1 - w^*||-bound-1}
    \end{align}
    Applying \Cref{main:cor:bound-norm-empirical-grad} to \eqref{eq:||w^t+1 - w^*||-bound-1}, and plugging in \Cref{main:thm:sharpness}, we get that when drawing 
    \begin{equation}\label{eq:batch-size-1}
        m \gtrsim \frac{dW^{9/2}b^4\log^4(d/(\eps\delta))}{L^4}\bigg(\frac{1}{\eps^{3/2}} + \frac{1}{\eps\delta}\bigg),\end{equation}
    samples from the distribution, it holds with probability at least $1 - \delta$ that:
    \begin{equation}\label{eq:||w^t+1 - w^*||-bound-1.5}
    \begin{aligned}
        \|\barw^{t+1} - \w^*\|_2^2 &\leq \|\htw^t - \w^*\|_2^2 + \eta^2(10(\opt + \eps) + 4b^2\|\htw^t - \w^*\|_2^2)  \\
        &\quad + 2\eta(2(\opt + \eps)/b + 2(\sqrt{\opt} + \sqrt{\eps})\|\htw^t - \w^*\|_2 - \mu\|\bv^t\|_2^2),
    \end{aligned}
    \end{equation}
    where $\mu = Ca^2LR^4/b$ with $C$ being an absolute constant, and where $\bv^t$ is the component of $\w^*$ that is orthogonal to $\htw^t$, i.e., 
    $$\bv^t = \w^* - (\w^*\cdot\htw^t)\htw^t/\|\htw^t\|_2^2 = (\w^*)^{\perp_{\htw^t}}. $$
    Note that $\|\bv^t\|_2$ is invariant to the rescaling of $\htw^t$, in other words, $\w^*$ has the same orthogonal component $\bv^t$ for all $\barw^t$, $\w^t$ and $\htw^t$.

    Since $\|\htw^t - \w^t\|_2 \leq \eta\sqrt{\eps}$, we have 
    \begin{equation}\label{ineq:||hwt-w*||-upbd}
    \|\htw^t - \w^*\|_2^2 = \|\htw^t - \w^t + \w^t - \w^*\|_2^2 \leq \|\w^t - \w^*\|_2^2 + \eta^2\eps + 2\eta\sqrt{\eps}\|\w^t - \w^*\|_2.
    \end{equation}
    In addition, by triangle inequality we have $\|\htw^t - \w^*\|_2\leq \|\w^t - \w^*\|_2 + \eta\sqrt{\eps}$. Therefore, substituting $\w^t$ with $\htw^t$ in \eqref{eq:||w^t+1 - w^*||-bound-1.5}, we get:
    \begin{align}\label{eq:||w^t+1 - w^*||-bound-2}
        \|\barw^{t+1} - \w^*\|_2^2 &\leq \|\w^t - \w^*\|_2^2 + \eta^2\eps + 2\eta\sqrt{\eps}\|\w^t - \w^*\|_2 \nonumber\\
        &\quad + \eta^2(10(\opt + \eps) + 4b^2\|\w^t - \w^*\|_2^2 + 4b^2\eta^2\eps + 8b^2\eta\sqrt{\eps}\|\w^t - \w^*\|_2)  \nonumber\\
        &\quad + 2\eta(2(\opt + \eps)/b + 2(\sqrt{\opt} + \sqrt{\eps})(\|\w^t - \w^*\|_2+\eta\sqrt{\eps}) - \mu\|\bv^t\|_2^2)\nonumber\\
        &\leq \|\w^t - \w^*\|_2^2 + \eta^2(24(\opt + \eps) + 4b^2\|\w^t - \w^*\|_2^2)\nonumber\\
        &\quad + 2\eta(2(\opt + \eps)/b + 4(\sqrt{\opt} + \sqrt{\eps})\|\w^t - \w^*\|_2 - \mu\|\bv^t\|_2^2),
    \end{align}
    where we used $4b^2\eta^2\leq 1$, which holds because $\eta = \mu/(4b^2)$. 

    Our goal is to show that $\|\bv^{t+1}\|_2^2 \leq \|\barw^{t+1} - \w^*\|_2^2\leq (1 - c)\|\bv^t\|_2^2 + \eps$, where $c\in(0,1)$ is a constant and $\eps$ is a small error parameter. However, this linear contraction can only be obtained when $\|\bv^t\|_2$ is relatively small compared to $\|\w^*\|_2$. Specifically, as will be manifested in \Cref{main:claim:||w^t+1 - w^*||-up-bound} and the proceeding proof, the linear contraction is achieved only when $\|\bv^t\|_2\leq \mu\|\w^*\|_2/(4b)$. Luckily, we can start with a $\bv^0$ such that this condition is satisfied, due to the initialization subroutine \Cref{alg:initialization}, as proved in \Cref{main:lem:initialization}. We prove the following claim.

    \begin{claim}\label{main:claim:||w^t+1 - w^*||-up-bound}
        Let $\eta = \mu/(4b^2)$. Then, under the assumptions of \Cref{thm:fast-converge-main}, with probability at least $1 - \delta,$ we have
        \begin{equation*}
            \|\barw^{t+1} - \w^*\|_2^2\leq \bigg(1 - \frac{\mu^2}{32b^2}\bigg)\|\bv^t\|_2^2,
        \end{equation*}
        whenever $\|\bv^t\|_2\geq (96/\mu)(\sqrt{\opt} + \sqrt{\eps}).$
    \end{claim}
    \begin{proof}[Proof of~\Cref{main:claim:||w^t+1 - w^*||-up-bound}]
    Since the norm of $\w^t$ is normalized to $\w^*$, the quantity $\|\w^t - \w^*\|^2$ is controlled by $\|\bv^t\|_2^2$. In particular, let $\w^* = \alpha_t\w^t + \bv^t$. Then, since $\bv^t\perp\w^t$, we have $\|\w^*\|_2^2 = \alpha_t^2\|\w^t\|_2^2 + \|\bv^t\|_2^2 = \alpha_t^2\|\w^*\|_2^2 + \|\bv^t\|_2^2,$
    thus, $\alpha_t^2 = 1 - \|\bv^t\|_2^2/\|\w^*\|_2^2$, and $\|\bv^t\|_2^2 = (1 - \alpha_t^2)\|\w^*\|_2^2$. In addition, $\|\w^t - \w^*\|_2^2$ can be expressed as a function of $\alpha_t$ and $\w^*$, as
    \begin{equation}\label{eq:||wt - w*||-as-alpha_t&||w*||}
        \|\w^t - \w^*\|_2^2 = (1 - \alpha_t)^2\|\w^*\|_2^2 + \|\bv^t\|_2^2 = 2(1 - \alpha_t)\|\w^*\|_2^2.
    \end{equation} 
    Note that since $\alpha_t = \sqrt{1 - \|\bv^t\|_2^2/\|\w^*\|_2^2}$, denoting $\rho_t = \|\bv^t\|_2/\|\w^*\|_2$, we further have:
    \begin{equation}\label{eq:alpha_t}
        1 - \alpha_t = 1 - \sqrt{1 - \|\bv^t\|_2^2/\|\w^*\|_2^2} = 1 - \sqrt{1 - \rho_t^2} \leq \frac{1}{2}\rho_t^2 + \frac{1}{2}\rho_t^4\leq \rho_t^2, \;\forall\rho_t\in[0,1].
\end{equation}
    
    Therefore, plugging \eqref{eq:||wt - w*||-as-alpha_t&||w*||} and \eqref{eq:alpha_t} back into \eqref{eq:||w^t+1 - w^*||-bound-2}, we get:
    \begin{align}
        \|\barw^{t+1} - \w^*\|_2^2&\leq 2(1 - \alpha_t)\|\w^*\|_2^2 + 4b^2\eta^2(2(1 - \alpha_t)\|\w^*\|_2^2) + 8\eta(\sqrt{\opt} + \sqrt{\eps})\sqrt{2(1 - \alpha_t)}\|\w^*\|_2 \nonumber\\
        &\quad - 2\eta\mu\|\bv^t\|_2^2 + 24\eta^2(\opt+\eps) + 4\eta(\opt + \eps)/b \nonumber\\
        &\leq (\rho_t^2 + \rho_t^4)\|\w^*\|_2^2 + 4b^2\eta^2(\rho_t^2 + \rho_t^4)\|\w^*\|_2^2 + 8\sqrt{2}\eta(\sqrt{\opt} + \sqrt{\eps})\rho_t\|\w^*\|_2 \nonumber\\
        &\quad - 2\eta\mu\|\bv^t\|_2^2  + 24\eta^2(\opt+\eps) + 4\eta(\opt + \eps)/b \nonumber\\
        & = (1 + \rho_t^2 + 4b^2\eta^2(1 + \rho_t^2))\|\bv^t\|_2^2 + 12\eta(\sqrt{\opt} + \sqrt{\eps})\|\bv^t\|_2 - 2\eta\mu\|\bv^t\|_2^2 \nonumber\\
        &\quad + 4(6\eta^2 + \eta/b)(\opt+\eps)\nonumber\\
        &\leq (1 + \rho_t^2 + 4b^2\eta^2(1 + \rho_t^2))\|\bv^t\|_2^2 + 12\eta(\sqrt{\opt} + \sqrt{\eps})\|\bv^t\|_2 - 2\eta\mu\|\bv^t\|_2^2 + 5\eta(\opt + \eps),\label{eq:||w^t+1 - w^*||-bound-3}
    \end{align}
    where in the last inequality we observed that since $\eta = \frac{\mu}{4b^2}$, it holds that $24\eta\leq 1$, as $\mu$ is small and $b\geq 1$.
    
Note that we have assumed that $\|\bv^t\|_2\geq (96/\mu)(\sqrt{\opt} + \sqrt{\eps})$, which indicates
    \begin{equation*}
        12\eta(\sqrt{\opt} + \sqrt{\eps})\|\bv^t\|_2\leq \frac{1}{8}\eta\mu\|\bv^t\|_2^2,
    \end{equation*}
    since $b\geq 1$ was assumed without loss of generality. Furthermore, when $\|\bv^t\|_2\geq (96/\mu)(\sqrt{\opt} + \sqrt{\eps})$, it also holds that
    \begin{equation*}
         \frac{1}{8}\eta\mu\|\bv^t\|_2^2\geq \frac{(96)^2}{8\mu^2}\eta\mu(\opt + \eps)\geq 5\eta(\opt + \eps),
    \end{equation*}
    since we have assumed $\mu = Ca^2LR^4/b \leq 1$ without loss of generality.
    Finally, as we will show in the rest of the proof, it holds that $\|\bv^{t+1}\|_2\leq \|\bv^t\|_2$ for $t = 0,1,\dots, T$, thus as $\eta = \mu/(4b^2)$, we have $\|\bv^t\|_2\leq \sqrt{\eta\mu}\|\w^*\|_2/2 = \mu\|\w^*\|_2/(4b)$, since $ \|\bv^0\|_2\leq\sqrt{\eta\mu}\|\w^*\|_2/2$. This condition guarantees that
    \begin{equation*}
        \rho_t^2 = \|\bv^t\|_2^2/\|\w^*\|_2^2\leq \frac{1}{4}\eta\mu.
    \end{equation*}
    Plugging these conditions back into \eqref{eq:||w^t+1 - w^*||-bound-3}, it is then simplified as (note that $1 + \rho_t^2\leq 1+ (1/4)\eta\mu\leq 9/8$ for $\eta\mu\leq 1/2$):
    \begin{equation*}
        \|\barw^{t+1} - \w^*\|_2^2\leq \bigg(1 + \frac{9}{2}b^2\eta^2 - \frac{3}{2}\eta\mu\bigg)\|\bv^t\|_2^2.
    \end{equation*}
    Therefore, when $\eta = \mu/(4b^2)$ we have
    \begin{equation*}
        \|\barw^{t+1} - \w^*\|_2^2\leq \bigg(1 - \frac{\mu^2}{32b^2}\bigg)\|\bv^t\|_2^2,
    \end{equation*}
    completing the proof.
    \end{proof}
    
    We proceed first under the condition that $\|\bv^t\|_2\geq (96/\mu)(\sqrt{\opt} + \sqrt{\eps})$ holds for $t = 0,\dots,T$ and show that after some certain number of iterations $T$ this condition must be violated. Observe that if $\|\bv^t\|_2\leq (96/\mu)(\sqrt{\opt} + \sqrt{\eps})$, then it holds $\|\w^t - \w^*\|_2^2\lesssim (1/\mu^2)(\opt + \eps)$, implying that $\htumt(\w^t\cdot\x)$ is a hypothesis achieving constant approximation error according to \Cref{main:lem:L2-error-upbd-||w - w^*||^2}, hence the algorithm can be terminated. However, note that $T$ only works as an upper bound for the iteration complexity of our algorithm, and it is possible that the condition $\|\bv^t\|_2\geq (96/\mu)(\sqrt{\opt} + \sqrt{\eps})$ is violated at some step $t^*<T$. However, as we show later,  the value of $\|\bv^T\|_2$ cannot be larger than $c\|\bv^{t^*}\|_2$,  where $c$ is an absolute constant.
We observe that:
    \begin{align*}
        \bv^{t+1} = \w^* - (\w^*\cdot\w^{t+1})\w^{t+1}/\|\w^{t+1}\|_2^2 = \w^* - (\w^*\cdot\barw^{t+1})\barw^{t+1}/\|\barw^{t+1}\|_2^2 = (\w^*)^{\perp_{\barw^{t+1}}},
    \end{align*}
    therefore, $\|\bv^{t+1}\|_2^2\leq \|\barw^{t+1} - \w^*\|_2^2$, which, combined with \Cref{main:claim:||w^t+1 - w^*||-up-bound}, yields
    \begin{equation*}
        \|\bv^{t+1}\|_2^2\leq \bigg(1 - \frac{\mu^2}{32b^2}\bigg)\|\bv^t\|_2^2\leq \bigg(1 - \frac{\mu^2}{32b^2}\bigg)^t\|\bv^0\|_2^2\leq \exp\bigg(-\frac{\mu^2t}{32b^2}\bigg)2W^2.
    \end{equation*}
    The above contraction only holds when $\|\bv^t\|_2\geq (96/\mu)(\sqrt{\opt} + \sqrt{\eps})$. Hence, after at most
    \begin{equation*}
        T = O\bigg(\frac{b^2}{\mu^2}\log\bigg(\frac{\mu W}{\eps}\bigg)\bigg)
    \end{equation*}
    inner iterations, the algorithm outputs a vector $\w^{t^*}$ with $\|\bv^{t^*}\|_2\leq \frac{96}{\mu}(\sqrt{\opt} + \sqrt{\eps})$, where $t^*\in[T]$.

    Now suppose that at step $t^* < T$ it holds that $\|\bv^{t^*}\|_2\leq 96(\sqrt{\opt} + \sqrt{\eps})/\mu$ but at the next iteration $\|\bv^{t^* + 1}\|_2\geq 96(\sqrt{\opt} + \sqrt{\eps})/\mu$. Recall first that in \Cref{main:cor:bound-norm-empirical-grad} we showed that $\|\nabla\htLsur(\htw^t;\htumt)\|_2^2\leq 4b^2\|\htw^t - \w^*\|_2^2 + 10 (\opt + \eps)$. Therefore, revisiting the updating scheme of the algorithm we have
    \begin{align*}
        \|\bv^{t^*+1}\|_2^2&\leq \|\barw^{t^*+1} -\w^*\|_2^2 = \|\htw^{t^*} - \eta\nabla\htLsur(\htw^{t^*};\hat{u}^{t^*}) - \w^*\|_2^2\\
        &\leq 2\|\htw^{t^*} -\w^*\|_2^2 + 2\eta^2\|\nabla\htLsur(\htw^{t^*};\hat{u}^{t^*})\|_2^2\\
        &\leq (2 + 8b^2\eta^2)\|\htw^{t^*} -\w^*\|_2^2 + 20\eta^2(\opt + \eps)\\
        &\leq 3\|\htw^{t^*} -\w^*\|_2^2 + (\opt + \eps),
    \end{align*}
    where in the last inequality we plugged in the value of $\eta = \mu/(4b^2)$, and used the assumption that $\mu\leq 1$ and $b\geq 1$, hence $20\eta^2\leq 1$ and $8b^2\eta^2\leq 1$. Furthermore, recall that by the construction of the grid, $\|\htw^{t^*} - \w^t\|_2\leq \eta\sqrt{\eps}$, implying that $\|\htw^{t^*} - \w^*\|_2^2\leq 2\|\w^{t^*} - \w^*\|_2^2 + 2\eta^2\eps$ by triangle inequality. Therefore, going back to the inequality of $\|\bv^{t^*+1}\|_2^2$ above, we get
    \begin{align*}
        \|\bv^{t^*+1}\|_2^2 &\leq 6\|\w^{t^*} - \w^*\|_2^2 + 6\eta^2\eps + \opt + \eps\leq 6\|\w^{t^*} - \w^*\|_2^2 + 2(\opt + \eps).
    \end{align*}
    Finally, observe that since $\|\w^{t^*}\|_2 = \|\w^*\|_2$, it holds $\|\w^{t^*} - \w^*\|_2\leq \sqrt{2}\|\bv^{t^*}\|_2$, hence, we get
    \begin{equation*}
        \|\bv^{t^*+1}\|_2^2\leq 12\|\bv^{t^*}\|_2^2 + 2(\opt + \eps).
    \end{equation*}
    Now since $\|\bv^{t^* + 1}\|_2\geq 96(\sqrt{\opt} + \sqrt{\eps})/\mu$, the value of $\|\bv^t\|_2^2$ will start to decrease again for $t\geq t^*+1$. This implies that the value of $\|\bv^T\|_2$ satisfies
    \begin{equation*}
        \|\bv^T\|_2\leq \sqrt{12}\|\bv^{t^*}\|_2 + \sqrt{2}(\sqrt{\opt} + \sqrt{\eps})\leq \frac{384}{\mu}(\sqrt{\opt} + \sqrt{\eps}).
    \end{equation*}

    Combining \Cref{main:lem:L2-error-upbd-||w - w^*||^2} and \Cref{main:lem:E[htutm(wt.x) - ut(wt.x))^2]<=eps}, as we have guaranteed that $\|\bv^{T}\|_2\leq (384/\mu)(\sqrt{\opt} + \sqrt{\eps})$, the hypothesis $\hat{u}^T(\htw^T\cdot\x)$ has the $L_2^2$ error that can be bounded as:
    \begin{equation*}
        \Ltwo(\htw^T;\hat{u}^T)\leq 6\opt + 3b^2(4\|\bv^T\|_2^2 + \eta^2\eps) + \eps = O\bigg(\frac{b^2}{\mu^2}(\opt + \eps)\bigg).
    \end{equation*}
    {For any $\eps_1>0$, setting $\eps = C'(\mu^2/b^2)\eps_1$ with $C'$ being some small universal absolute constant,} we finally get $\Ltwo(\htw^T;\hat{u}^T)\leq O((b^2/\mu^2)\opt) + \eps_1$.

    It still remains to determine the batch size as drawing a sample set of size $m$ as displayed in \eqref{eq:batch-size-1} only guarantees that the contraction of $\|\bv^t\|_2$ at step $t$ holds with probability $1 - \delta$. Applying a union bound on all {$t_0JT = O(\frac{Wb^{10}}{\mu^{9}\sqrt{\eps}}\log(1/\eps)) = O(\frac{Wb^{11}}{\mu^{10}\sqrt{\eps_1}}\log(1/\eps_1))$} iterations yields that the contraction holds at every step with probability at least $1 - t_0JT\delta$. Therefore, setting $\delta \gets \delta(t_0JT)$ and bringing the value of $\delta$ back to \eqref{eq:batch-size-1}, we get that it suffices to choose the batch size as:
    \begin{equation*}
        m = \Theta\bigg(\frac{dW^{9/2}b^4\log^4(d/(\eps\delta))}{L^4}\bigg(\frac{1}{\eps^{3/2}} + \frac{Wb^{10}}{\mu^9\eps^{3/2}\delta}\bigg)\bigg) = \Theta\bigg(\frac{dW^{11/2}b^{17}\log^5(d/(\eps_1\delta))}{L^4\mu^{12}\delta\eps_1^{3/2}}\bigg),
    \end{equation*}
    to guarantee that we get an {$O(\opt) + \eps_1$}-solution with probability at least $1 - \delta$. {Note that we have set $\eps = C'(\mu^2/b^2)\eps_1$ in the last equality above.}
    
    The argument above justifies the claim that among all $t_0J = Wb^7\log(b/\mu)/(\eta\mu^7\sqrt{\eps_1})$ hypotheses in $\mathcal{P} = \{(\htw^T_j; \hat{u}^T_{j})\}_{j=1}^{t_0 J}$, there exists at least one hypothesis that achieves $L_2^2$ error {$O(\opt) + \eps_1$}. To select the correct hypothesis from the set $\mathcal{P}$, one only needs to draw a new batch of $m' = \tilde{\Theta}(b^4W^4\log(1/\delta)/(L^4\eps_1^2))$ i.i.d.\ samples from $\D$, and choose the hypothesis from $\mathcal{P}$ that achieves the minimal empirical error defined in \Cref{line:testing}. As discussed in \Cref{subsec:testing}, this procedure introduces an error at most $\epsilon_1.$

In conclusion, it holds by a union bound that \Cref{alg:optimization} delivers a solution with $O(\opt) + \eps_1$ error with probability at least $1 - 2\delta$. The total sample complexity of our algorithm is 
$$N = t_0JTm + m' = \Theta\bigg(\frac{W^{13/2}b^{28}d\log^5(d/(\eps_1\delta))}{L^4\mu^{22}\delta\eps_1^2} + \frac{b^4W^4\log(1/\delta)\log^5(1/\eps_1)}{L^4\eps_1^2}\bigg) = \Theta\bigg(\frac{W^{13/2}b^{28}d\log^6(d/(\eps_1\delta))}{L^4\mu^{22}\delta\eps_1^2}\bigg).$$

Choosing $\delta = 1/6$ above we get that the \Cref{alg:optimization} {succeeds to generate an $O(\opt) + \eps_1$-solution for any $\eps_1>0$ with probability at least $1 - 2\delta = 2/3$, hence replacing $\eps_1$ with $\eps$} completes the proof of \Cref{app:thm:fast-converge-main}.
\end{proof}

\subsection{Proof of \Cref{main:cor:bound-norm-empirical-grad}}\label{app:pf:main:cor:bound-norm-empirical-grad}
This subsection is devoted to the proof of \Cref{main:cor:bound-norm-empirical-grad}. To this aim, we first show the following lemmas that bound from above the norm of the population gradient $\nabla\Lsur(\w^t;\htumt)$ and the difference between the population gradient and the empirical gradient $\nabla\htLsur(\w^t;\htumt)$.

\begin{restatable}{lemma}{normPopulationGrad}\label{lem:upbd-||nabla Lusr(wt)||}
    Let $S$ be a sample set of $m$ i.i.d.\ samples of size at least $m \gtrsim d\log^4(d/(\eps\delta))(b^2W^3/L^2\eps)^{3/2}$. 
Furthermore, given $\w^t\in\B(W)$, let $\htumt$ be defined as in \eqref{def:htumt}. Then, it holds that with probability at least $1 - \delta$,\begin{equation*}
    \|\nabla\Lsur(\w^t;\htumt)\|_2^2\leq 8(\opt + \eps) + 2b^2\|\w^t - \w^*\|_2^2.
    \end{equation*}
\end{restatable}

\begin{proof}
    By the definition of $\ell_2$ norms, we have:
    \begin{align*}
        \|\nabla\Lsur(\w^t;\htumt)\|_2 &= \max_{\|\bv\|_2 = 1}\nabla\Lsur(\w^t;\htumt)\cdot\bv\\
        &=\max_{\|\bv\|_2 = 1} \Exy[(\htumt(\w^t\cdot\x) - y)\bv\cdot\x]\\
        &=\max_{\|\bv\|_2 = 1}\bigg\{\Exy[(\htumt(\w^t\cdot\x) - u^t(\w^t\cdot\x) + u^t(\w^t\cdot\x) - u^{*t}(\w^t\cdot\x))(\bv\cdot\x)]\\
        &\quad +\Exy[(u^{*t}(\w^t\cdot\x) - u^*(\w^*\cdot\x) + u^*(\w^*\cdot\x)-y)(\bv\cdot\x)]\bigg\}.
    \end{align*}
    By the Cauchy-Schwarz inequality, we further have:
    \begin{align*}
     &\quad \|\nabla\Lsur(\w^t;\htumt)\|_2\\
     &\leq \max_{\|\bv\|_2 = 1}\bigg\{\sqrt{\Ex[(\htumt(\w^t\cdot\x) - u^t(\w^t\cdot\x))^2]\Ex[(\bv\cdot\x)^2]} + \sqrt{\Ex[(u^t(\w^t\cdot\x) - u^{*t}(\w^t\cdot\x))^2]\Ex[(\bv\cdot\x)^2]} \\
        &\quad \quad  + \sqrt{\Ex[(u^{*t}(\w^t\cdot\x) - u^*(\w^*\cdot\x))^2]\Ex[(\bv\cdot\x)^2]}  + \sqrt{\Ex[(u^*(\w^*\cdot\x)-y)^2]\Ex[(\bv\cdot\x)^2]}\bigg\}\\
        &\leq \underbrace{\sqrt{\Ex[(\htumt(\w^t\cdot\x) - u^t(\w^t\cdot\x))^2]}}_{\mathcal{T}_1} + \underbrace{\sqrt{\Ex[(u^t(\w^t\cdot\x) - u^{*t}(\w^t\cdot\x))^2]}}_{\mathcal{T}_2}  \\
        &\quad + \underbrace{\sqrt{\Ex[(u^{*t}(\w^t\cdot\x) - u^*(\w^*\cdot\x))^2]}}_{\mathcal{T}_3} + \underbrace{\sqrt{\Ex[(u^*(\w^*\cdot\x)-y)^2]}}_{\mathcal{T}_4},
    \end{align*}
    where in the last inequality we used the assumption that $\Ex[\x\x^\top] \preccurlyeq \mathbf{I}$, hence $\Ex[(\bv\cdot\x)^2] \leq 1$. It remains to bound $\mathcal{T}_1 - \mathcal{T}_4$. Observe first that  $\mathcal{T}_1\leq \sqrt{\eps}$ for every $\w^t\in\B(W)$, with probability at least $1 - \delta$, due to \Cref{main:lem:E[htutm(wt.x) - ut(wt.x))^2]<=eps}. By definition, $\mathcal{T}_4 = \sqrt{\opt}$. Recall that in \Cref{main:lem:upper-bound-u_t-u_t^*} we showed the following $\mathcal{T}_2^2 = \Ex[(u^t(\w^t\cdot\x) - u^{*t}(\w^t\cdot\x))^2]\leq \opt$. For $\mathcal{T}_3$, note that $u^{*t}\in\argmin_{u\in\U}\Ex[(u(\w^t\cdot\x) - u^*(\w^*\cdot\x))^2]$, therefore, since $u^*\in\U$, we have 
    \begin{equation*}
        \mathcal{T}_3^2 = \Ex[(u^{*t}(\w^t\cdot\x) - u^*(\w^*\cdot\x))^2]\leq \Ex[(u^*(\w^t\cdot\x) - u^*(\w^*\cdot\x))^2]\leq b^2\|\w^t - \w^*\|_2^2,
    \end{equation*}
    after applying the assumption that $u^*$ is $b$-Lipschitz. Thus, in conclusion, we have $$\|\nabla\Lsur(\w^t;\htumt)\|_2\leq 2\sqrt{\opt} + \sqrt{\eps} + b\|\w^t - \w^*\|_2.$$ 
    Furthermore, since $(a+b)^2\leq 2a^2+2b^2$ for any $a, b\in\R$, we get with probability at least $1 - \delta$:
    \begin{equation*}
        \|\nabla\Lsur(\w^t;\htumt)\|_2^2\leq 8\opt + 8\eps + 2b^2\|\w^t - \w^*\|_2^2,
    \end{equation*}
    completing the proof of \Cref{lem:upbd-||nabla Lusr(wt)||}.
\end{proof}

We now prove that the distance between $\nabla\Lsur(\w^t;\htumt)$ and $\nabla\htLsur(\w^t;\htumt)$ is bounded by $b^2\|\w^t - \w^*\|_2^2 + \opt + \eps$ with high probability.

\begin{restatable}{lemma}{normDiffEmpiricalPopulationGrad}\label{lem:empirical-grad-close-to-expectation}
Let $S$ be a sample set of $m \gtrsim  (dW^{9/2}b^4\log^4(d/(\eps\delta))/L^4)(1/\eps^{3/2} + 1/(\eps\delta))$ i.i.d.\ samples. 
Given a vector $\w^t\in\B(W)$, it holds that with probability at least $1 - \delta$,\begin{equation*}
        \|\nabla\htLsur(\w^t;\htumt) - \nabla\Lsur(\w^t;\htumt)\|_2\leq \sqrt{b^2\|\w^t - \w^*\|_2^2 + \opt + \eps}.
    \end{equation*}
\end{restatable}

\begin{proof}
Since for any zero-mean independent random variable $\bz_j$, we have $\E[||\littlesum_{j} \bz_j||_2^2] = \littlesum_{j}\E[\|\bz_j\|_2^2]$, by Chebyshev's inequality:
    \begin{equation}\label{eq:markov-||htlsur - lsur||<=eps}
        \pr[\|\nabla \htLsur(\w^t;\htumt) - \nabla\Lsur(\w^t;\htumt)\|_2\geq s]\leq \frac{1}{ms^2}\Exy[\|(\htumt(\w^t\cdot\x) - y)\x\|_2^2] \;.
    \end{equation}
   By linearity of expectation,  we have:
    \begin{align*}
        \Exy[\|(\htumt(\w^t\cdot\x) - y)\x\|_2^2]= \sum_{k=1}^d \Exy[(\htumt(\w^t\cdot\x) - y)^2(\x_k)^2], 
    \end{align*}
    where $\x_k = \e_k\cdot\x$ and $\e_k$ is the $k^\mathrm{th}$ unit basis of $\R^d$. Let $r = O(W/L\log(1/(L\eps')))$, then it holds $\pr[|\x_k|\geq r]\leq \eps'$. Then, the variance above can be decomposed into the following parts:
    \begin{align*}
        \Exy[(\htumt(\w^t\cdot\x) - y)^2\x_k^2] &= \Exy[(\htumt(\w^t\cdot\x) - y)^2\x_k^2\1\{|\x_k|\geq r\}]\\
        &\quad + \Exy[(\htumt(\w^t\cdot\x) - y)^2\x_k^2\1\{|\x_k|\leq r\}].
    \end{align*}
    Since $|y|\leq M = O(bW/L\log(bW/\eps))$, and $\Ex[(\w^t\cdot\x)^4\x_k^4]\leq W^4c^2/L^8$, $\Ex[\x_k^4]\leq c^2/L^4$ for $\D_\x$ is $L$-sub-exponential, we have
    \begin{align}\label{eq:||htlsur - lsur||<=eps-inter-1}
        \Exy[(\htumt(\w^t\cdot\x) - y)^2\x_k^2\1\{|\x_k|\geq r\}]&\leq 2 \Exy[(\htumt(\w^t\cdot\x))^2 + y^2)\x_k^2\1\{|\x_k|\geq r\}]\nonumber\\
        &\leq 2 \Exy[(b(\w^t\cdot\x))^2 + y^2)\x_k^2\1\{|\x_k|\geq r\}] \nonumber\\
        &\leq 2b^2\sqrt{\Ex[((\w^t\cdot\x)^4\x_k^4]\pr[|\x_k|\geq r]} \nonumber\\
        &\quad + 2M^2\sqrt{\Ex[\x_k^4]\pr[|\x_k|\geq r]}\nonumber\\
        &\leq (2cb^2W^2/L^4)\sqrt{\eps'} + (2cM^2/L^2)\sqrt{\eps'}\leq (4cM^2/L^2)\sqrt{\eps'}.
    \end{align}
    In addition, $(\htumt(\w^t\cdot\x) - y)^2$ can be decomposed as the following:
    \begin{align*}
        \Exy[(\htumt(\w^t\cdot\x) - y)^2]&\leq 4\Ex[(\htumt(\w^t\cdot\x) - u^t(\w^t\cdot\x))^2] + 4\Ex[(u^t(\w^t\cdot\x) - u^{*t}(\w^t\cdot\x))^2]\\
        &\quad + 4\Ex[(u^{*t}(\w^t\cdot\x) - u^*(\w^*\cdot\x))^2] + 4\Exy[(u^*(\w^*\cdot\x) - y)^2].
    \end{align*}
    The first term is bounded above by $4\eps$ with probability at least $1 - \delta$ for every $\w^t\in\B(W)$ whenever $m \gtrsim  d\log^4(d/(\eps\delta))(b^2W^3/L^2\eps)^{3/2}$, as proved in \Cref{main:lem:E[htutm(wt.x) - ut(wt.x))^2]<=eps}. The second term is smaller than $4\opt$, which is shown in \Cref{main:lem:upper-bound-u_t-u_t^*}. The third term can  be  bounded above using again the definition of $u^{*t} = \argmin_{u\in\U}\Ex[(u(\w^t\cdot\x) - y^*)^2]$, as
    \begin{equation*}
        4\Ex[(u^{*t}(\w^t\cdot\x) - u^*(\w^*\cdot\x))^2]\leq 4\Ex[(u^*(\w^t\cdot\x) -u^*(\w^*\cdot\x))^2]\leq 4b^2\|\w^t - \w^*\|_2^2,
    \end{equation*}
    using the fact that $u^*$  is $b$-Lipschitz and $\Ex[\x\x^\top] \preccurlyeq \mathbf{I}$. Lastly, the fourth term is bounded by $4\opt$ by the definition of $u^*(\w^*\cdot\x)$. In summary, we have
    \begin{align*}
        \Exy[(\htumt(\w^t\cdot\x) - y)^2\x_k^2\1\{|\x_k|\leq r\}]&\leq r^2\Exy[(\htumt(\w^t\cdot\x) - y)^2]\\
        &\leq 4r^2(b^2\|\w^t - \w^*\|_2^2 + 2\opt + \eps),
    \end{align*}
    which, combining with \eqref{eq:||htlsur - lsur||<=eps-inter-1}, implies that the expectation $\Exy[(\htumt(\w^t\cdot\x) - y)^2\x_k^2]$ is bounded by:
    \begin{align*}
       \Exy[(\htumt(\w^t\cdot\x) - y)^2\x_k^2]&\leq 4r^2b^2\|\w^t - \w^*\|_2^2 + 4r^2(2\opt + 2\eps) \\
        &\leq \frac{CW^2}{L^2}\log^2\bigg(\frac{b}{L\eps}\bigg)(b^2\|\w^t - \w^*\|_2^2 + \opt + \eps),
\end{align*}
    where $C$ is a large absolute constant.
    Note to get the inequality above we chose $\eps' = C\eps^2(L/b)^4$, which then indicates that $4c(M/L)^2\sqrt{\eps'}\leq r^2\eps$. Summing the inequality above from $k=1$ to $d$ delivers the final upper bound on the variance:
    \begin{equation*}
        \Exy[\|(\htumt(\w^t\cdot\x) - y)\x\|_2^2]\leq \frac{dCW^2}{L^2}\log^2\bigg(\frac{b}{L\eps}\bigg)(b^2\|\w^t - \w^*\|_2^2 + \opt + \eps).
    \end{equation*}
    
    Thus, plugging the upper bound on the variance above back to \eqref{eq:markov-||htlsur - lsur||<=eps}, as long as $m\gtrsim (dW^2/L^2)\log^2(b/(L\eps))/\delta$, we get with probability at least $1 - \delta$, 
    \begin{equation*}
        \|\nabla\htLsur(\w^t;\htumt) - \nabla\Lsur(\w^t;\htumt)\|_2\leq \sqrt{b^2\|\w^t - \w^*\|_2^2 + \opt + \eps}.
    \end{equation*}
    Noting that $m \gtrsim  (dW^{9/2}b^4\log^4(d/(\eps\delta))/L^4)(1/\eps^{3/2} + 1/(\eps\delta))$ certainly satisfies the condition on $m$ above as $m\gtrsim (dW^2/L^2)\log^2(b/(L\eps))/\delta$, thus, we completed the proof of \Cref{lem:empirical-grad-close-to-expectation}
\end{proof}

We can now proceed to the proof of \Cref{main:cor:bound-norm-empirical-grad} {(detailed statement in \Cref{app:cor:bound-norm-empirical-grad} below)}, which can be derived directly from the preceding lemmas.

\begin{restatable}[Upper Bound on Empirical Gradient Norm]{lemma}{normEmpiricalGrad}\label{app:cor:bound-norm-empirical-grad}
    Let $S$ be a set of i.i.d.\ samples of size $m \gtrsim  (dW^{9/2}b^4\log^4(d/(\eps\delta))/L^4)(1/\eps^{3/2} + 1/(\eps\delta))$. 
Given any $\w^t\in\B(W)$, let $\htumt\in\U$ be the solution of optimization problem \eqref{def:htumt} with respect to $\w^t$ and sample set $S$. Then, with probability at least $1 - \delta$, we have that
    $        \|\nabla\htLsur(\w^t;\htumt)\|_2^2\leq 4b^2\|\w^t - \w^*\|_2^2 + 10 (\opt + \eps)$.
\end{restatable}

\begin{proof}The lemma follows directly by combining \Cref{lem:upbd-||nabla Lusr(wt)||}, \Cref{lem:empirical-grad-close-to-expectation} and the triangle inequality. 
\end{proof}

\subsection{Proof of \Cref{main:lem:L2-error-upbd-||w - w^*||^2}}\label{app:pf:main:lem:L2-error-upbd-||w - w^*||^2}
We restate (providing a more detailed statement for the sample size) and prove \Cref{main:lem:L2-error-upbd-||w - w^*||^2}.

\begin{claim}\label{app:lem:L2-error-upbd-||w - w^*||^2}
    Let $\w$ be any vector from $\B$. Let $\htuw$ be a solution to  \eqref{def:htumt} for a fixed parameter vector $\w\in\R^d$ with sample size $m \gtrsim d\log^4(d/(\eps\delta))(b^2W^3/(L^2\eps))^{3/2}$. Then 
    $$\Exy[(\htuw(\w\cdot\x) - y)^2]\leq 8(\opt + \eps) + 4b^2\|\w - \w^*\|_2^2.$$
\end{claim}

\begin{proof}
    Let $\ustrw$, $\uw$ be the optimal activations for problems \eqref{def:ut*} and \eqref{def:ut} under parameter $\w$,  respectively. Then, a direct calculation gives:
\begin{align}
        &\quad \Exy[(\htuw - y)^2]\nonumber\\
        & = \Exy[(\htuw(\w\cdot\x) -\uw(\w\cdot\x) + \uw(\w\cdot\x) - u^*_\w(\w\cdot\x) + u^*_\w(\w\cdot\x) - u^*(\w^*\cdot\x) + u^*(\w^*\cdot\x)-y)^2]\nonumber\\
        &\leq 4\Ex[(\htuw(\w\cdot\x) - \uw(\w\cdot\x))^2] + 4\Ex[(u_\w(\w\cdot\x) - u^*_\w(\w\cdot\x))^2] \nonumber\\
        &\quad + 4\Ex[(u^*_\w(\w\cdot\x) - u^*(\w^*\cdot\x))^2] + 4\opt\nonumber\\
        &\leq 8(\opt + \eps) + 4b^2\|\w - \w^*\|_2^2,
    \end{align}
    where in the second inequality we used the results from \Cref{main:lem:upper-bound-u_t-u_t^*}, \Cref{main:lem:E[htutm(wt.x) - ut(wt.x))^2]<=eps} and  we applied the observation that:
    \begin{equation*}
        \Ex[(u^*_\w(\w\cdot\x) - u^*(\w^*\cdot\x))^2]\leq \Ex[(u^*(\w\cdot\x) - u^*(\w^*\cdot\x))^2]\leq b^2\|\w - \w^*\|_2^2,
    \end{equation*}
    by the definition of $u^*_\w$.
\end{proof}

\subsection{Proof of \Cref{main:lem:testing}}\label{app:pf:main:lem:testing}
We restate \Cref{main:lem:testing} and show the number of samples needed for the testing subroutine \Cref{alg:testing}.

\TestingLemma*
\begin{proof}
    Fix $(\w_j,u_j)\in\mathcal{P}$. Since $\D_\x$ is sub-exponential, we have $\pr[|\w_j\cdot\x|\geq \|\w_j\|_2r]\leq \frac{1}{L^2}\exp(-Lr)$. Consider random variables $Z_{i,j} = (u_j(\w_j\cdot\x\ith) - y\ith)^2\1\{|\w\cdot\x\ith|\leq r\}$, $i = 1,\cdots,m$, $j = 1,\cdots, t_0J$, where $(\x\ith,y\ith)$ are independent random variables drawn from $\D$. Using \Cref{main:lem:y-bounded-by-M} (\Cref{app:uniform-convergence}), we can truncate the labels $y$ such that $|y|\leq M$, where $M = C(bW/L)\log(bW/\eps_1)$ for some large absolute constant $C$. Hence, $|Z_{i,j}|\leq 2(u_j^2(\w_j\cdot\x\ith) + (y\ith)^2)\1\{|\w_j\cdot\x\ith|\leq Wr\}\leq 2(b^2W^2r^2 + M^2)$, where we used the assumption that $u$ is $b$-Lipschitz in the last inequality. Therefore, applying Hoeffding's inequality to $Z_{i,j}$ we get:
    \begin{align*}
        \pr\bigg[\bigg|\sum_{i = 1}^{m'} (Z_{i,j} - E[Z_{i,j}])\bigg|\geq {m'}t\bigg]\leq 2\exp\bigg(-\frac{m't^2}{8(b^2W^2r^2 + M^2)^2}\bigg).
    \end{align*}
Since there are $t_0J = Wb^7/(\mu^7\eta\sqrt{\eps_1}) = 4b^9W/(\mu^8\sqrt{\eps_1})$ elements in the set $\mathcal{P}$,  applying a union bound leads to:
    \begin{equation*}
        \pr\bigg[\bigg|\sum_{i = 1}^{m'} Z_{i,j} - E[Z_{i,j}]\bigg|\geq {m'}t, \forall j\in[J]\bigg]\leq 2\exp\bigg(-\frac{{m'}t^2}{8(b^2W^2r^2 + M^2)^2} + \log(4b^9W/(\mu^8\sqrt{\eps_1}))\bigg).
    \end{equation*}
    
    Therefore, when 
    \begin{equation}\label{eq:testing-m}
        {m'} \geq \frac{8(b^2W^2r^2 + M^2)^2}{\eps_1^2}\bigg(\log\bigg(\frac{4b^9W}{\mu^8\sqrt{\eps_1}}\bigg) + \log(2/\delta)\bigg),
    \end{equation}
    we have that with probability at least $1 - \delta$:
    \begin{equation}\label{eq:testing-empirical-close-to-population}
        \bigg| \frac{1}{{m'}}\sum_{i=1}^{m'} (u_j(\w_j\cdot\x\ith) - y\ith)^2\1\{|\w_j\cdot\x\ith|\leq Wr\} - \Exy[(u_j(\w_j\cdot\x) - y)^2\1\{|\w_j\cdot\x|\leq Wr\}] \bigg| \leq  \eps_1,
    \end{equation}
    for any $(\w_j,u_j)\in\mathcal{P}$. In addition, as $\pr[|\w_j\cdot\x|\geq Wr]\leq \pr[|\w_j\cdot\x|\geq \|\w_j\|_2r]\leq \frac{2}{L^2}\exp(-Lr)$, letting $\eps' = \frac{2}{L^2}\exp(-Lr)$, we further have:
    \begin{align*}
        &\quad \Exy[(u_j(\w_j\cdot\x) - y)^2\1\{|\w_j\cdot\x|\geq Wr\}]\\
        &\leq 2\Ex[((u_j(\w_j\cdot\x))^2 + M^2)\1\{|\w_j\cdot\x|\geq Wr\}]\\
        &\leq 2b^2\sqrt{\Ex[(\w_j\cdot\x)^4]\pr[|\w_j\cdot\x|\geq Wr]} + M^2\pr[|\w_j\cdot\x|\geq Wr]\\
        &\leq 2cb^2(W/L)^2\sqrt{\eps'} + M^2\eps'\leq (2cb^2(W/L)^2 + M^2)\sqrt{\eps'},
    \end{align*}
    where in the second inequality we used Cauchy-Schwarz inequality and in the last inequality we used the property that for any unit vector $\vec a$ it holds $\E[(\vec a\cdot\x)^4]\leq c^2/L^4$ for some absolute constant $c$  as $\x$ possesses a $\frac{1}{L}$-sub-exponential tail. Therefore, choosing $r = \frac{1}{L}\log(\frac{C^2b^4W^4}{L^6\eps_1^2}\log^2(\frac{bW}{\eps_1})) = \tilde{O}(\frac{1}{L}\log(\frac{bW}{L\eps_1}))$ for some large absolute constant $C$ renders $\sqrt{\eps'} \leq \eps_1/(2Cb^2(W/L)^2\log^2(bW/\eps_1))$, and we have
    \begin{equation*}
        \Exy[(u_j(\w_j\cdot\x) - y)^2\1\{|\w_j\cdot\x|\geq Wr\}]\leq \eps_1.
    \end{equation*}
    Observe that as $\Exy[(u_j(\w_j\cdot\x) - y)^2]$ is the sum of $\Exy[(u_j(\w_j\cdot\x) - y)^2\1\{|\w_j\cdot\x|\geq Wr\}]$ and $\Exy[(u_j(\w_j\cdot\x) - y)^2\1\{|\w_j\cdot\x|\leq Wr\}]$, we have
    \begin{align*}
        0 &\leq  \Exy[(u_j(\w_j\cdot\x) - y)^2] - \Exy[(u_j(\w_j\cdot\x) - y)^2\1\{|\w_j\cdot\x|\leq Wr\}]\\
        &\leq \Exy[(u_j(\w_j\cdot\x) - y)^2\1\{|\w_j\cdot\x|\geq Wr\}] \leq \eps_1.
    \end{align*}
    Plugging the choice of $r$ back into \eqref{eq:testing-m}, we get that it is sufficient to choose $m'$ as
    \begin{equation*}
        m' = \frac{C\log(\log(1/\eps_1))}{\eps_1^2}\bigg(b^2\bigg(\frac{W}{L}\bigg)^2\log^2\bigg(\frac{bW}{L\eps_1^2}\bigg)\bigg)^2\bigg(\log\bigg(\frac{4b^9W}{\mu^8\sqrt{\eps_1}}\bigg) + \log(1/\delta)\bigg) = \tilde{\Theta}\bigg(\frac{b^4W^4\log(1/\delta)}{L^4\eps_1^2}\log^5\bigg(\frac{bW}{L\mu\eps_1}\bigg)\bigg).
    \end{equation*}
    Therefore, using $m' = \tilde{\Omega}(b^4W^4/(L^4\eps_1^2))$ samples, \eqref{eq:testing-empirical-close-to-population} indicates that with probability at least $1 - \delta$, for any $(\w_j,u_j)\in\mathcal{P}$ it holds
    \begin{align*}
        &\quad \bigg| \frac{1}{m'}\sum_{i=1}^{m'} (u_j(\w_j\cdot\x\ith) - y\ith)^2\1\{|\w_j\cdot\x\ith|\leq Wr\} - \Exy[(u_j(\w_j\cdot\x) - y)^2] \bigg|\\
        &\leq \bigg| \frac{1}{m'}\sum_{i=1}^{m'} (u_j(\w_j\cdot\x\ith) - y\ith)^2\1\{|\w_j\cdot\x\ith|\leq Wr\} - \Exy[(u_j(\w_j\cdot\x) - y)^2\1\{|\w_j\cdot\x|\leq Wr\}] \bigg| \\
        &\quad + \bigg| \Exy[(u_j(\w_j\cdot\x) - y)^2] - \Exy[(u_j(\w_j\cdot\x) - y)^2\1\{|\w_j\cdot\x|\leq Wr\}]\bigg|\\
        &\leq 2\eps_1,
    \end{align*}
    thus completing the proof of \Cref{main:lem:testing}.
\end{proof} 

\section{Efficiently Computing the Optimal Empirical Activation}\label{app:isotonic-regression}

In this section, we show that the optimization problem \eqref{def:htumt} can be solved efficiently, 
following the framework from~\cite{LH2022} with minor modifications.
We show that, for any $\eps>0$, there is an efficient algorithm 
that runs in $\tilde{O}(m^2\log(1/\eps))$ time and 
outputs a solution $\htvmt(z)$ such that 
$\|\htvmt(z) - \htumt(z)\|_\infty\leq \eps$. {We then argue that using such approximate solutions to the optimization problem \eqref{def:htumt} does not negatively impact our error guarantee, sample complexity, or runtime (up to constant factors).}

\begin{proposition}[Approximating the Optimal Empirical Activation]\label{app:prop:approx-empirical-optimal-activation}
    Let $\eps>0$, and $\D_\x$ be $(L,R)$-well behaved. Let $\htumt\in\U$ be the optimal solution of the optimization problem \eqref{def:htumt} given a sample set $S$  of size $m$ drawn from $\D$ and a parameter $\w^t\in\B(W)$. There exists an algorithm that produces an activation $\htvmt\in\U$ such that $\|\htvmt- \htumt\|_\infty\leq \eps$, with runtime $\tilde{O}(m^2\log(bW/(L\eps)))$.
\end{proposition}

To prove \Cref{app:prop:approx-empirical-optimal-activation}, we leverage the following result:

\begin{lemma}[Section 5~\citep{LH2022}]\label{app:prop:genreal-algo-for-general-isotonic-problem}
    Let $f_i(y)$ and $h_i(y)$ be any convex lower semi-continuous functions for $i=1,\dots,m$. Consider the following convex optimization problem
    \begin{equation}\label{general-isotonic-problem}
        (\hat{y}_1,\dots,\hat{y}_m) = \argmin_{y_1,\dots,y_m} \sum_{i=1}^m f_i(y_i) + \sum_{i=1}^{m-1} h_i(y_i - y_{i+1}),
    \end{equation}
    where $y_i\in[-U,U]$ for all $i=1,\dots,m$ for some positive constant $U$. Then, for any $\eps>0$, there exists an algorithm (the cc-algorithm~\citep{LH2022}) that outputs an $\eps$-close solution $\{y_1,\dots,y_m\}$ such that $|y_i - \hat{y}_i|\leq \eps$ for all $i\in [m]$ with runtime $O(m^2\log(U/\eps))$.
\end{lemma}

\begin{proof}[Proof of \Cref{app:prop:approx-empirical-optimal-activation}]
We first reformulate problem \eqref{def:htumt} as a quadratic optimization problem with linear constraints. To guarantee that $\htumt$ is an element in $\U$ that satisfies $\htumt(0) = 0$, we add a zero point $(\x^{(0)},y^{(0)}) = (\vec 0,0)$ to the data set $S$ if $S$ does not contain $(\vec 0,0)$ in the first place. We can thus assume without loss of generality that the data set contains $(\vec 0, 0)$. Denote $z_i = \w\cdot\x\ith$ such that $z_1\leq z_2\leq \cdots\leq z_m$ after rearranging the order of $(\x\ith,y\ith)$'s, and suppose $z_k = \w\cdot\x_0 = 0$ for a $k\in[m]$. Then \eqref{def:htumt} is equivalent to the following optimization problem: 
\begin{equation*}
    \begin{aligned}
        (\hat{y}^{(1)},\cdots,\hat{y}^{(m)}) = \argmin_{\ty^{(i)}, i\in [m]} \; &\sum_{i = 1}^{m} (\ty^{(i)} - y\ith)^2\\
        \mathrm{s.t.}\;\; &\;0\leq \ty^{(i+1)} - \ty^{(i)},\; &&1\leq i\leq k-1,\\
        &\; a(z_{i+1} - z_i)\leq \ty^{(i+1)} - \ty^{(i)},\; &&1\leq i\leq k-1,\\
        &\; \ty^{(i+1)} - \ty^{(i)} \leq b(z_{i+1} - z_i),\; &&1\leq i\leq m-1,\\
        &\; \ty^{(k)} = 0.
    \end{aligned}
\end{equation*}

Define $h_i(y) = \mathcal{I}_{[-b(z_{i+1} - z_i), 0]}(y)$ for $i = 1\dots, k-1$, $h_i(y) = \mathcal{I}_{[-b(z_{i+1} - z_i), -a(z_{i+1} - z_i)]}(y)$ for $i = k \dots, m-1$, where $\mathcal{I}_{\mathcal{Y}}(y)$ is the indicator function of a convex set ${\mathcal{Y}}$, i.e., $\mathcal{I}_{\mathcal{Y}}(y) = 0$ if $y\in{\mathcal{Y}}$ and $\mathcal{I}_{\mathcal{Y}}(y) = +\infty$ otherwise. It is known that $h_i$'s are convex and sub-differentiable on their domain $\mathcal{Y}_i$. In addition, let $f_i(y) = \frac{1}{2}(y - y\ith)^2$ for $i\neq k$ and $f_k(y) = \mathcal{I}_{\{0\}}(y)$. Then, we have the following formulation for problem \eqref{def:htumt}:
\begin{equation*}
    (\hat{y}^{(1)},\cdots,\hat{y}^{(m)}) = \argmin_{\ty\ith, i=1,\dots,m} \sum_{i=1}^{m} f_i(\ty\ith) + \sum_{i=1}^{m-1} h_i(\ty\ith - \ty^{(i+1)}) \label{solve-ut-subprob1}\tag{P1}
\end{equation*}

Note that the functions $f_i$ and $h_i$ we defined above satisfy the conditions of \Cref{app:prop:genreal-algo-for-general-isotonic-problem}. Thus, it only remains to find the bounds on the variables $\ty^{(i)}$. This is easy to achieve as all $\ty^{(i)}$ must satisfy $|\ty\ith|\leq b|z_i| = b|\w\cdot\x\ith|$ and we know that $\x\ith$ are sub-exponential random variables. Therefore, following the same idea from the proof of \Cref{main:lem:E[htustrm(wt.x) - u*t(wt.x))^2]<=eps}, we know that for $U = \frac{2W}{L}\log(m/(L\delta))$, it holds that with probability at least $1 - \delta$, $|\ty^{(i)}|\leq b|\w\cdot\x\ith|\leq bU$ for all $i\in[m]$. Hence, applying \Cref{app:prop:genreal-algo-for-general-isotonic-problem} to problem \eqref{solve-ut-subprob1}, we get that it can be solved within $\eps$-error in runtime $\tilde{O}(m^2\log(bW/(L\eps)))$. 
\end{proof}

\paragraph{The effect of approximation error in \eqref{def:htumt}} 
Since the solution $\htvmt$ is $\eps$-close to $\htumt$, this approximated solution will only result in an $\eps$-additive error in the sharpness result \Cref{main:thm:sharpness} and the gradient norm concentration \Cref{main:cor:bound-norm-empirical-grad}. In more detail, for the result of \Cref{main:thm:sharpness}, we have
\begin{align*}
    \bigg|(\nabla\htLsur(\w^t;\htvmt) - \nabla\htLsur(\w^t;\htumt))\cdot(\w^t - \w^*)\bigg|& = \bigg|\frac{1}{m}\sum_{i=1}^m (\htvmt(\w^t\cdot\x\ith) - \htumt(\w^t\cdot\x\ith))(\w^t - \w^*)\cdot\x\ith\bigg|\\
    &\leq \frac{\eps}{m}\sum_{i=1}^m\big|(\w^t - \w^*)\cdot\x\ith\big|\leq 2 \eps U,
\end{align*}
since $|\w^t\cdot\x\ith|\leq U$ and $|\w^*\cdot\x\ith|\leq U$ with probability at least $1 - \delta$. Therefore, choosing $\eps' = \eps/U$ we have that \Cref{main:thm:sharpness} holds for approximate activations $\htvmt$ with an additional $\eps$ error. Observe that this does not affect the approximation factor in our final $O(\opt) + \eps$ result, while the value of $\eps$ only needs to be rescaled by a constant factor, effectively increasing the sample size and the runtime by constant factors. 

Let us denote the unit ball by $\B$. For the gradient norm concentration lemma \Cref{main:cor:bound-norm-empirical-grad}, note that at any iteration $t$, it always holds that
\begin{align*}
    \|\nabla\htLsur(\w^t;\htvmt) - \nabla\htLsur(\w^t;\htumt)\|_2& = \max_{\bv\in\B}\frac{1}{m}\sum_{i=1}^m (\htvmt(\w^t\cdot\x\ith) - \htumt(\w^t\cdot\x\ith))\x\ith\cdot\bv \leq \max_{\bv\in\B}\frac{\eps}{m}\sum_{i=1}^m|\x\ith\cdot\bv|.
\end{align*}
Since $\Ex[\x\x^\top] \preccurlyeq \mathbf{I}$  and $\bv\in\B$, we have $\Ex[|\x\cdot\bv|]\leq \sqrt{\E[(\x\cdot\bv)^2]}\leq 1$. Now since $|\x\ith\cdot\bv|$ are independent $1/L$-sub-exponential random variables, applying Bernstein's inequality it holds that for any $\bv\in\B$ and an absolute constant $c$,
\begin{align*}
    \pr\bigg[\bigg|\frac{1}{m}\sum_{i=1}^m |\x\ith\cdot\bv| - \Ex[|\x\cdot\bv|]\bigg|\geq s\bigg]\leq 2\exp\bigg(-c\min\bigg\{\frac{m^2s^2}{m/L^2}, \frac{ms}{1/L}\bigg\}\bigg) = 2\exp(-cmL^2s^2).
\end{align*}
Let $N(\B,\eps;\ell_2)$ be the $\eps$-net of the unit ball $\B$. Note that the cover number of these $\bv\in\B$ is of order $(1/\eps)^{O(d)}$, therefore, applying a union bound on $N(\B,\eps;\ell_2)$ and for all $t_0JT = O(\log(1/\eps)/\sqrt{\eps})$ iterations, and setting $s = 1$, it holds
\begin{equation*}
    \pr\bigg[\forall\bv\in N(\B,\eps;\ell_2), \, \bigg|\frac{1}{m}\sum_{i=1}^m |\x\ith\cdot\bv| - \Ex[|\x\cdot\bv|]\bigg|\geq 1\bigg]\leq 2\exp(-cmL^2 + c'd\log(1/\eps)) \leq \delta,
\end{equation*}
where the last inequality comes from the fact that we have $m \gtrsim W^{9/2}b^{14}d\log(1/\delta)\log^4(d/\eps)/(L^4\mu^9\delta\eps^{3/2})$ as the batch size. Let $\bv^* = \argmax_{\bv\in\B}\sum_{i=1}^m |\x\ith\cdot\bv|$. Then there exists a $\bv'\in N(\B,\eps;\ell_2)$ such that $\|\bv' - \bv^*\|_2\leq \eps$ and hence, 
\begin{align*}
    \frac{1}{m}\sum_{i=1}^m |\x\ith\cdot\bv^*|&\leq \frac{1}{m}\sum_{i=1}^m |\x\ith\cdot(\bv^* -\bv')| + \frac{1}{m}\sum_{i=1}^m |\x\ith\cdot\bv'|\\
    & = \frac{\eps}{m}\sum_{i=1}^m |\x\ith\cdot\frac{\bv^* - \bv'}{\eps}| + \frac{1}{m}\sum_{i=1}^m |\x\ith\cdot\bv'|\\
    &\leq \frac{\eps}{m}\sum_{i=1}^m |\x\ith\cdot\bv^*| + \frac{1}{m}\sum_{i=1}^m |\x\ith\cdot\bv'|,
\end{align*}
where the last inequality comes from the observation that as $(\bv^* - \bv')/\eps\leq\B$, it holds $\sum_{i=1}^m|\x\ith\cdot((\bv^* - \bv')/\eps)|\leq \sum_{i=1}^m |\x\ith\cdot\bv^*|$, by the definition of $\bv^*$.
Therefore, with probability at least $1 - \delta$ we have
$$\frac{1}{m}\sum_{i=1}^m |\x\ith\cdot\bv^*|\leq \frac{1}{1-\eps}\frac{1}{m}\sum_{i=1}^m |\x\ith\cdot\bv'|\leq 2(1+\Ex[|\bv\cdot\x|])\leq 4.$$
This implies that $\|\nabla\htLsur(\w^t;\htvmt)\|_2\leq \|\nabla\htLsur(\w^t;\htumt)\|_2 + 4\eps$ for all iterations with probability at least $1 - \delta$. Therefore, \Cref{main:cor:bound-norm-empirical-grad} continues to hold for the $\eps$-approximate activation $\htvmt$.

Thus, we have that the inequalities \eqref{eq:||w^t+1 - w^*||-bound-1} and \eqref{eq:||w^t+1 - w^*||-bound-1.5} in the proof of \Cref{thm:fast-converge-main} remain valid for $\eps$-approximate $\htvmt$, and hence the results in \Cref{thm:fast-converge-main} are unchanged.

\section{Uniform Convergence of Activations}\label{app:uniform-convergence}

In this section, we review and provide standard uniform convergence results showing that the  sample-optimal activations concentrate around their population-optimal counterparts.
We first bound the $L_2^2$ distance between the sample-optimal and population-optimal activations under $\w^t$. To do so, we build on Lemma 8 in \cite{kakade2011efficient}. Note that Lemma 8 from \cite{kakade2011efficient} only works for bounded $1$-Lipschitz activations $u:\R\mapsto[0,1]$, hence it is not directly applicable to our case. Fortunately, since $\D_\x$ has a sub-exponential tail (see \Cref{def:bounds}), we are able to bound the range of $u(\w\cdot\x)$ for $u\in\U$ and $\w\in\B(W)$ with high probability. Concretely, we prove the following lemma. Note that in the lemma statement, $\htustrmt$ is a random variable defined w.r.t.\ the (random) dataset $S^*,$ and thus the probabilistic statement is for this random variable.

We make use of the following fact from \cite{kakade2011efficient}:
\begin{fact}[Lemma 8~\citep{kakade2011efficient}]\label{lem:kakade-lem-8}
    Let $\mathcal{V}$ be the set of non-deceasing 1-Lipschitz functions such that $v: \R\to [0,1]$, $\forall v\in\mathcal{V}$. Given  $S_m = \{(\x\ith,y\ith)\}_{i=1}^m$, where $(\x\ith,y\ith)$ are sampled i.i.d. from some distribution $\D'$, let \begin{equation*}
        \hat{v}_{\w} \in \argmin_{v\in\mathcal{V}}\frac{1}{m}\sum_{i=1}^m (v(\w\cdot\x\ith) - y\ith)^2.
    \end{equation*}
    Then, with probability at least $1 - \delta$ over the random dataset $S_m$, for any $\w\in\B(W)$ it holds uniformly that
    \begin{equation*}
        \E_{(\x,y)\sim\D'}[(\hat{v}_{\w}(\w\cdot\x) - y)^2] - \inf_{v\in \mathcal{V}} \E_{(\x,y)\sim\D'}[(v(\w\cdot\x) - y)^2] = O\bigg(W\bigg(\frac{d \log(Wm/\delta)}{m}\bigg)^{2/3}\bigg).
    \end{equation*}
\end{fact}

The first lemma states that with sufficient many of samples, the idealized sample-optimal activation $\htustrmt$ defined as the optimal solution of \eqref{def:htustrmt} is close to its population counterpart $u^{*t}$, the optimal solution of \eqref{def:ut*}.

\begin{restatable}[Approximating Population-Optimal Noiseless Activation by Sample-Optimal]{lemma}{empPopStarActivationDistance}\label{main:lem:E[htustrm(wt.x) - u*t(wt.x))^2]<=eps}
    Let $\D_\x$ be $(L,R)$-well behaved and let $\w^t\in\B(W)$. Provided a dataset $S^* = \{(\x\ith,y\sith)\}$, where $\x\ith$ are i.i.d.\ samples from $\D_\x$ and $y\sith = u^*(\w^*\cdot\x\ith)$, let $\htustrmt$ be the sample-optimal activation on $S^*$ as defined in \eqref{def:htustrmt}.
In addition, let $u^{*t}$ be the corresponding population-optimal activation, following the definition in \eqref{def:ut*}. 
Then, for any $\eps, \delta > 0,$ if the size $m$ of the dataset $S^*$ is  sufficiently large 
$$m \gtrsim d\log^4(d/(\eps\delta))\bigg(\frac{b^2W^3}{L^2\eps}\bigg)^{3/2},$$ 
we have that with probability at least $1 - \delta$, for any $\w^t\in\B(W)$:
\begin{equation*}
    \Ex[(\htustrmt(\w^t\cdot\x) - u^*(\w^*\cdot\x))^2] \leq \E_{\x\sim\D_\x}[(u^{*t}(\w^t\cdot\x) - u^*(\w^*\cdot\x))^2] + \eps\;,
\end{equation*}
and, furthermore,
\begin{equation*}
    \Ex[(\htustrmt(\w^t\cdot\x) - u^{*t}(\w^t\cdot\x))^2]\leq \eps.
\end{equation*}
\end{restatable}
\begin{proof}
     Our goal is to show that with high probability, the sample-optimal activation $\htustrmt\in\U$ and the population optimal activation $u^{*t}\in\U$ can be scaled to 1-Lipschitz functions mapping $\R$ to $[0,1]$, then, \Cref{lem:kakade-lem-8} can be applied.
     
     Since $\x$ possesses a sub-exponential tail, for any $\w\in\B(W)$ we have $\pr[|\w\cdot\x|\geq \|\w\|_2 r]\leq \frac{2}{L^2}\exp(-Lr)$. Therefore, with probability at least $1 - (\delta_1/m)^2$ it holds $|\w\cdot\x|\leq\frac{2W}{L}\log(m/(L\delta_1))$. Since we have $m$ samples, a union bound on these $m$ samples yields that with probability at least $1 - \delta_1^2/m$ it holds $|\w\cdot\x\ith|\leq \frac{2W}{L}\log(m/(L\delta_1))$, for any given $\w\in\B(W)$. Let $r = \frac{2W}{L}\log(m/(L\delta_1))$. In the remainder of the proof, we assume that $\w^t\cdot\x\ith\leq r$ holds for every $\x\ith$ in the dataset $S^*$, which happens with probability at least $1 - \delta_1^2/m\geq 1 - \delta_1$.

     Let $\mathcal{V}$ be the set of non-decreasing 1-Lipschitz functions $v:\R\to[0,1]$ such that $v(0) = 1/2$, and $v(z_1) - v(z_2)\geq (a/(2br))(z_1 - z_2)$ for all $z_1\geq z_2\geq 0$. We observe that restricted on the interval $|z|\leq r$, $(\htustrmt(z)/(2br)+ 1/2)|_{|z|\leq r}$ is 1-Lipschitz, non-decreasing and bounded in the interval $[0,1]$. Thus, $ (\htustrmt(z)/(2br)+ 1/2)|_{|z|\leq r} = \hat{v}^{*t}(z)|_{|z|\leq r}$, for some $\hat{v}^{*t}\in\mathcal{V}$. Furthermore, under the condition that $|\w^t\cdot\x\ith|\leq r$, since $ (\htustrmt(z)/(2br)+ 1/2)|_{|z|\leq r} = \hat{v}^{*t}(z)|_{|z|\leq r}$, we observe that $v^{*t}(z)$ is the optimal activation in the function space $\mathcal{V}$, given the dataset $S^*$ and parameter $\w^t$, i.e., 
     \begin{equation*}
         \hat{v}^{*t}\in\argmin_{v\in\mathcal{V}}\frac{1}{m}\sum_{i=1}^m (v(\w^t\cdot\x\ith) - (u^*(\w^*\cdot\x\ith)/(2br) + 1/2))^2.
     \end{equation*}

      In other words, $\htustrmt(z)/(2br) + 1/2$ is the sample-optimal activation in the function class $\mathcal{V}$ when restricted to the interval $|z|\leq r$. Consider $\x\sim\D_\x$. Then $\pr[|\w^t\cdot\x|\geq r]\leq (\delta_1/m)^2$ and for any $\w^t\in\B(W)$, the expectation $\Ex[(\htustrmt(\w^t\cdot\x) - u^*(\w^*\cdot\x))^2]$ can be decomposed into the following terms
     \begin{align}
         \Ex[(\htustrmt(\w^t\cdot\x) - u^*(\w^*\cdot\x))^2] &= \Ex[(\htustrmt(\w^t\cdot\x) - u^*(\w^*\cdot\x))^2\1\{|\w^t\cdot\x|\leq r\}] \notag\\
         &\quad + \Ex[(\htustrmt(\w^t\cdot\x) - u^*(\w^*\cdot\x))^2\1\{|\w^t\cdot\x|> r\}]\notag\\
         &\leq \Ex[(\htustrmt(\w^t\cdot\x) - u^*(\w^*\cdot\x))^2\1\{|\w^t\cdot\x|\leq r\}] \notag\\
         &\quad + 2\Ex[(\htustrmt(\w^t\cdot\x))^2 + (u^*(\w^*\cdot\x))^2\1\{|\w^t\cdot\x|> r\}] \;.\label{eq:idealized-sample-population}
    \end{align}
    Since both $ \htustrmt$ and $u^*$ are $(a,b)$-unbounded functions such that $\htustrmt(0) = u^*(0) = 0$, we have $(\htustrmt(\w^t\cdot\x))^2\leq b^2W^2(({\w^t}/{\|\w^t\|_2})\cdot\x)^2$ and similarly, $(u^*(\w^*\cdot\x))^2\leq b^2W^2((\w^*/\|\w^*\|_2)\cdot\x)^2$. 
    Furthermore, since for any unit vector $\ba$, the random variable $\ba\cdot\x$ follows a $(1/L)$-sub-exponential distribution as $\D_\x$ is $(L,R)$-well behaved, thus, it holds that $\Ex[(\ba\cdot\x)^4]\leq c/L^4$ for some absolute constant $c$.
    Therefore, after applying Cauchy-Schwarz inequality to $\E[(\htustrmt(\w^t\cdot\x))^2\1\{|\w^t\cdot\x|\geq r\}]$, we get
    \begin{align}\label{eq:bound-(u(wx))^2*1[wx>r]}
        \E[(\htustrmt(\w^t\cdot\x))^2\1\{|\w^t\cdot\x|\geq r\}]&\leq b^2W^2\sqrt{\Ex[(({\w^t}/{\|\w^t\|_2})\cdot\x)^4]\pr[|\w^t\cdot\x|\geq r]}\nonumber\\
        &\leq cb^2W^2\delta_1/(L^2m),
    \end{align}
    and similarly, $\E[(u^*(\w^*\cdot\x))^2\1\{|\w^t\cdot\x|\geq r\}]\leq cb^2W^2\delta_1/(L^2m)$. Thus, plugging these inequalities back into \eqref{eq:idealized-sample-population}, we get
    \begin{equation}\label{eq:idealized-sample-pop-2}
        \begin{aligned}
        \Ex[(\htustrmt(\w^t\cdot\x) - u^*(\w^*\cdot\x))^2]&\leq \Ex[(\htustrmt(\w^t\cdot\x) - u^*(\w^*\cdot\x))^2\1\{|\w^t\cdot\x|\leq r\}] \nonumber\\
        & \quad + 2c b^2W^2 \delta_1/(L^2m).
    \end{aligned}
    \end{equation}

     We are now ready to apply \Cref{lem:kakade-lem-8} (note that $\mathcal{V}$ is a smaller function class compared to the class of 1-Lipschitz functions described \Cref{lem:kakade-lem-8}, hence  \Cref{lem:kakade-lem-8} applies). Denote $A = \{\x:|\w^t\cdot\x|\leq r\}$. Let $y' = y^*/(2br) + 1/2$, $y^* = u^*(\w^*\cdot\x)$. Since conditioning on  $A$, $\htustrmt(z)/(2br) + 1/2$ is the sample-optimal  activation, applying \Cref{lem:kakade-lem-8} we get that with probability at least $1  -\delta_2$:
     \begin{align*}
     &\quad \Ex[(\htustrmt(\w^t\cdot\x)/(2br) + 1/2 - (u^*(\w^*\cdot\x)/(2br) + 1/2))^2 | A]\\
     &=
         \Ex[(\hat{v}^{*t}(\w^t\cdot\x) - y')^2|A]\\
         &\leq \inf_{v\in\mathcal{V}}\E_{\x\sim\D_\x}[(v(\w^t\cdot\x) - y')^2|A] + \tilde{O}(W(d\log(m/\delta_2)/m)^{2/3}).
     \end{align*}
     Let $\mathcal{V}|_{|z|\leq r}$ and $\U|_{|z|\leq r}$ be the functions from $\mathcal{V}$ and $\U$ restricted on the interval $|z|\leq r$, respectively. It is not hard to see that by the definition of $\U$ and $\mathcal{V}$, $(\U|_{|z|\leq r})/(2br) + 1/2\subset \mathcal{V}|_{|z|\leq r}$. Therefore, 
     \begin{align*}
         \inf_{v\in\mathcal{V}}\E_{\x\sim\D_\x}[(v(\w^t\cdot\x) - y')^2|A]&\leq \inf_{u\in\mathcal{U}}\E_{\x\sim\D_\x}[(u(\w^t\cdot\x)/(2br) + 1/2 - y')^2|A]\\
         &\leq \frac{1}{4b^2r^2}\inf_{u\in\mathcal{U}}\E_{\x\sim\D_\x}[(u(\w^t\cdot\x) - y^*)^2|A].
     \end{align*}
     Hence, with probability at least $1  -\delta_2$,
     \begin{align*}
         &\quad \Ex[(\htustrmt(\w^t\cdot\x) - u^*(\w^*\cdot\x))^2\1\{A\}]\\
         &= 4b^2r^2\Ex[(\hat{v}^{*t}(\w^t\cdot\x) - y')^2|A]\pr[A]\\
         &\leq 4b^2r^2\inf_{v\in\mathcal{V}}\E_{\x\sim\D_\x}[(v(\w^t\cdot\x) - y')^2|A]\pr[A] + \tilde{O}(b^2r^2 W(d\log(m/\delta_2)/m)^{2/3})\pr[A]\\
         &\leq \inf_{u\in\mathcal{U}}\E_{\x\sim\D_\x}[(u(\w^t\cdot\x) - u^*(\w^*\cdot\x))^2\1\{A\}] + \tilde{O}(b^2r^2 W(d\log(m/\delta_2)/m)^{2/3})\\& \leq \inf_{u\in\mathcal{U}}\E_{\x\sim\D_\x}[(u(\w^t\cdot\x) - u^*(\w^*\cdot\x))^2] + \tilde{O}(b^2r^2 W(d\log(m/\delta_2)/m)^{2/3}) \;.
     \end{align*}    
     Setting $\delta_1 = \delta_2 = \delta/2$ and plugging everything  back into \eqref{eq:idealized-sample-pop-2}, we finally get that with probability at least $1 - \delta$,\begin{align*}
         &\quad \Ex[(\htustrmt(\w^t\cdot\x) - u^*(\w^*\cdot\x))^2]\\
         &\leq \inf_{u\in\mathcal{U}}\E_{\x\sim\D_\x}[(u(\w^t\cdot\x) - u^*(\w^*\cdot\x))^2] + {O}\bigg(\frac{b^2W^3}{L^2}\log^2\bigg(\frac{m}{L\delta}\bigg) \bigg(\frac{d\log(m/\delta)}{m}\bigg)^{2/3}\bigg).
     \end{align*}
     To complete the first part of the claim, it remains to choose $m$ as the following value
     $$m = \Theta\bigg(d\log^4(d/(\eps\delta))\bigg(\frac{b^2W^3}{L^2\eps}\bigg)^{3/2}\bigg).$$ 

     For the second part of the claim, note that $\U$ is a closed convex set of functions, and that the infimum $\inf_{u\in\mathcal{U}}\E_{\x\sim\D_\x}[(u(\w^t\cdot\x) - u^*(\w^*\cdot\x))^2]$ is attained by $u^{*t}(z)$. Observe that we have shown that with the sample size $m$ specified above, with probability at least $1 - \delta$, it holds
     \begin{align*}
         \eps&\geq \Ex[(\htustrmt(\w^t\cdot\x) - u^*(\w^*\cdot\x))^2 - (u^{*t}(\w^t\cdot\x) - u^*(\w^*\cdot\x))^2]\\
         & = \Ex[(\htustrmt(\w^t\cdot\x) - u^{*t}(\w^t\cdot\x))(\htustrmt(\w^t\cdot\x) + u^{*t}(\w^t\cdot\x) - 2u^*(\w^*\cdot\x))]\\
         & = \Ex[(\htustrmt(\w^t\cdot\x) - u^{*t}(\w^t\cdot\x))^2] + 2\Ex[(\htustrmt(\w^t\cdot\x) - u^{*t}(\w^t\cdot\x))(u^{*t}(\w^t\cdot\x) - u^*(\w^*\cdot\x))].
     \end{align*}
     Since $\htustrmt(z)\in\U$, applying the second part of \Cref{app:claim:(ut-v)(y-ut)geq0} with $v' = \htustrmt$ we get
     \begin{equation*}
         \Ex[(u^{*t}(\w^t\cdot\x) - \htustrmt(\w^t\cdot\x))(u^*(\w^*\cdot\x) - u^{*t}(\w^t\cdot\x))]\geq 0.
     \end{equation*}
     Thus, we have:
     \begin{equation*}
         \Ex[(\htustrmt(\w^t\cdot\x) - u^{*t}(\w^t\cdot\x))^2]\leq \eps.
     \end{equation*}
     This completes the proof of \Cref{main:lem:E[htustrm(wt.x) - u*t(wt.x))^2]<=eps}.\end{proof}

To prove a similar uniform convergence result for the attainable activations $\htumt,$ we make use of the following fact from prior literature, which shows that we can without loss of generality take the noisy labels to be bounded by $M = O(\frac{bW}{L}\log(bW/\eps))$, due to $\D_\x$ being $(L,R)$-well behaved.

\begin{fact}[Lemma D.8~\citep{WZDD2023}]\label{main:lem:y-bounded-by-M}
    Let $y' = \sgn(y)\min(|y|, M)$ for $M=\frac{bW}{L}\log(\frac{16b^4 W^4}{\eps^2})$. Then:
    \begin{equation*}
        \Exy[(u^*(\w^*\cdot\x) - y')^2] = \opt+\eps.
    \end{equation*}
\end{fact}
In other words,  we can assume $|y|\leq M$ without loss of generality by truncating labels that are larger than $M$. 
Under this assumption, as stated in \Cref{main:lem:E[htutm(wt.x) - ut(wt.x))^2]<=eps} below, we bound the $L_2^2$ distance between $\htumt$ and $u^t$ using similar arguments  as in \Cref{main:lem:E[htustrm(wt.x) - u*t(wt.x))^2]<=eps}.

\begin{restatable}[Approximating Population-Optimal Activation by Sample-Optimal]{lemma}{empPopActivationDistance}\label{main:lem:E[htutm(wt.x) - ut(wt.x))^2]<=eps}
    Let $\w^t\in\B(W)$. Given a distribution $\D$ whose marginal $\D_\x$ is $(L,R)$-well behaved, let $S = \{(\x\ith,y\ith)\}_{i=1}^m$, where $(\x\ith,y\ith)$ for $ i\in [m]$ are i.i.d.\ samples from $\D$. Let $\htumt$ be a sample-optimal activation for the dataset $S$ and parameter vector $\w^t$, as defined in \eqref{def:htumt}.
In addition, let $u^t$ be the corresponding population-optimal activation, as defined in \eqref{def:ut}.
Then, for any $\eps, \delta > 0,$ choosing a sufficiently large 
$$
m \gtrsim d\log^4(d/(\eps\delta))\bigg(\frac{b^2W^3}{L^2\eps}\bigg)^{3/2},
$$ 
we have that for any $\w^t\in\B(W)$, with probability at least $1 - \delta$ {over the dataset $S$}:
\begin{equation*}
    \Exy[(\htumt(\w^t\cdot\x) - y)^2] \leq \Exy[(u^t(\w^t\cdot\x) - y)^2] + \eps\;,
\end{equation*}
and, furthermore,
\begin{equation*}
    \Ex[(\htumt(\w^t\cdot\x) - u^t(\w^t\cdot\x))^2]\leq \eps.
\end{equation*}
\end{restatable}
\begin{proof}
    As in the proof of \Cref{main:lem:E[htustrm(wt.x) - u*t(wt.x))^2]<=eps}, we choose $r = \frac{2cW}{L}\log(m/(L\delta_1))$ so that $|\w^t\cdot\x\ith|\leq r$ for all $\x\ith$'s from the dataset with probability at least $1 - \delta_1^2/m\geq 1 - \delta_1$. We now condition on the event that $|\w^t\cdot\x\ith|\leq r$ for all $i=1,\dots,m$. Let $\mathcal{V}$ be the set of non-decreasing 1-Lipschitz functions such that $\forall v\in\mathcal{V}$, $v(0) = 1/2$, and $v(z_1) - v(z_2)\geq (a/(2br))(z_1 - z_2)$ for all $z_1\geq z_2\geq 0$. Then, conditioned on this event, we similarly have that $(\htumt(z)/(2br) + 1/2)|_{|z|\leq r} = \hat{v}^t(z)\in\mathcal{V}$, and $\hat{v}^t(z)$ satisfies:
    \begin{equation*}
        \hat{v}^t(z)\in\argmin_{v\in\mathcal{V}} \frac{1}{m}\sum_{i=1}^m (v(\w^t\cdot\x\ith) - y\ith)^2.
    \end{equation*}
    
    Again, studying the $L_2^2$ distance between $\htumt(z)$ and $u^t(z)$,  we have:
    \begin{align*}
        \Exy[(\htumt(\w^t\cdot\x) - y)^2]&=\Exy[(\htumt(\w^t\cdot\x) - y)^2\1\{|\w^t\cdot\x|\leq r\}]\\
        &\quad + \Exy[(\htumt(\w^t\cdot\x) - y)^2\1\{|\w^t\cdot\x|>r\}].
    \end{align*}
    The probability of $|\w^t\cdot\x|>r$ is small due to the fact that $\D_\x$ possesses sub-exponential tail: $\pr[|\w^t\cdot\x|>r]\leq (\delta_1/m)^2$. Now note that $|y|\leq M$ and $\Ex[((\w^t/\|\w^t\|_2)\cdot\x)^4]\leq c/L^4$ by the sub-exponential property of $\D_\x$, we thus have:
    \begin{align*}
        &\quad \Exy[(\htumt(\w^t\cdot\x) - y)^2\1\{|\w^t\cdot\x|>r\}]\\
        &\leq 2\Exy[((\htumt(\w^t\cdot\x))^2 + y^2)\1\{|\w^t\cdot\x|>r\}]\\
        &\leq 2\Ex[b^2W^2((\w^t/\|\w^t\|_2)\cdot\x)^2\1\{|\w^t\cdot\x|>r\}] + 2M^2\pr[|\w^t\cdot\x|>r]\\
        &\leq 2b^2W^2\sqrt{\Ex[((\w^t/\|\w^t\|_2)\cdot\x)^4]\pr[|\w^t\cdot\x|>r]} + 2M^2\pr[|\w^t\cdot\x|>r]\\
        &\leq 2cb^2W^2\delta_1/(L^2 m) + 2M^2(\delta_1/m)^2,
    \end{align*}
    where in the second inequality we used the fact that $\htumt$ is $b$-Lipschitz and $\w^t\in\B(W)$, and in the third inequality we applied Cauchy-Schwarz. Since $M = \frac{bW}{L}\log(\frac{16b^4 W^4}{\eps^2})$, we have $M^2(\delta_1/m)\lesssim cb^2W^2/L^2$ for $m\gtrsim \log(bW/\eps)$, thus,  we get
    \begin{equation}\label{eq:(htumt(wtx)-y)^2(1 - A)-upbd}
         \Exy[(\htumt(\w^t\cdot\x) - y)^2\1\{|\w^t\cdot\x|>r\}]\leq 4c(bW/L)^2\delta_1/m,
    \end{equation}
    for some absolute constant $c$.

    The rest remains the same as in the proof of \Cref{main:lem:E[htustrm(wt.x) - u*t(wt.x))^2]<=eps}. Let $A = \{\x: |\w^t\cdot\x|\leq r\}$. Let $y' = y/(2br) + 1/2$. As $\hat{v}^t(z) = \htumt(z)/(2br) + 1/2$ is the sample-optimal activation in $\mathcal{V}$ given $\w^t$ (conditioned on $A$), applying \Cref{lem:kakade-lem-8} we have that with probability at least $1 - \delta$:
    \begin{align*}
        \Exy[((\htumt(\w^t\cdot\x)/(2br) + 1/2) - y')^2|A]& = \Exy[(\hat{v}^t(\w^t\cdot\x) - y')^2|A]\\
        &\leq \inf_{v\in\mathcal{V}}\Exy[(v(\w^t\cdot\x) - y')^2|A] + \tilde{O}(W(d\log(m/\delta_2)/m)^{2/3}).
    \end{align*}
    Since $\U|_{|z|\leq r}/(2br) + 1/2\subset \mathcal{V}|_{|z|\leq r}$, we further have
    \begin{align*}
        \inf_{v\in\mathcal{V}}\Exy[(v(\w^t\cdot\x) - y')^2|A]&\leq \inf_{u\in\U}\Exy[(u(\w^t\cdot\x)/(2br) + 1/2 - y')^2|A]\\
        &\leq \frac{1}{4b^2r^2}\inf_{u\in\U}\Exy[(u(\w^t\cdot\x) - y)^2|A].
    \end{align*}
    Therefore, $\Exy[(\htumt(\w^t\cdot\x) - y)^2\1\{A\}]$ can be bounded from above by
    \begin{align*}
        &\quad \Exy[(\htumt(\w^t\cdot\x) - y)^2\1\{A\}]\\
        &=4b^2r^2\Exy[(\hat{v}^t(\w^t\cdot\x) - y')^2|A]\pr[A]\\
        &\leq 4b^2r^2\inf_{v\in\mathcal{V}}\Exy[(v(\w^t\cdot\x) - y')^2|A]\pr[A] + \tilde{O}(b^2r^2W(d\log(m/\delta_2)/m)^{2/3}) \\
        &\leq \inf_{u\in\U}\Exy[(u(\w^t\cdot\x) - y)^2\1\{A\}] + \tilde{O}(b^2r^2W(d\log(m/\delta_2)/m)^{2/3})\\
        &\leq \inf_{u\in\U}\Exy[(u(\w^t\cdot\x) - y)^2] + \tilde{O}(b^2r^2W(d\log(m/\delta_2)/m)^{2/3}).
    \end{align*}
    Thus, combining with \eqref{eq:(htumt(wtx)-y)^2(1 - A)-upbd}, we get that with probability at least $1 - \delta_1 - \delta_2$,
    \begin{equation*}
        \Exy[(\htumt(\w^t\cdot\x) - y)^2]\leq \Exy[(u^t(\w^t\cdot\x) - y)^2] + \tilde{O}\bigg(Wb^2r^2\bigg(\frac{d\log(m/\delta_2)}{m}\bigg)^{2/3}\bigg) + \bigg(\frac{bW}{L}\bigg)^2\frac{\delta_1}{m}.
    \end{equation*}
    Choosing the size of the sample set to be:
    $$m = \Theta\bigg(d\log^4(d/(\eps\delta))\bigg(\frac{b^2W^3}{L^2\eps}\bigg)^{3/2}\bigg),$$
    and recalling that $r = \frac{2cW}{L}\log(m/(L\delta_1))$, we finally have
    \begin{equation*}
        \Exy[(\htumt(\w^t\cdot\x) - y)^2]\leq \Exy[(u^t(\w^t\cdot\x) - y)^2] + \eps,
    \end{equation*}
    with probability at least $1 - \delta$, after choosing $\delta_1 = \delta_2 = \delta/2$.

    To prove the final claim of the lemma, we follow the same argument as in the proof of  \Cref{main:lem:E[htustrm(wt.x) - u*t(wt.x))^2]<=eps}. Since we have just shown that with probability at least $1 - \delta$, it holds 
    \begin{align*}
        \eps&\geq \Ex[(\htumt(\w^t\cdot\x) - y)^2 - (u^t(\w^t\cdot\x) - y)^2]\\
         & = \Ex[(\htumt(\w^t\cdot\x) - u^t(\w^t\cdot\x))^2] + 2\Ex[(\htumt(\w^t\cdot\x) - u^t(\w^t\cdot\x))(u^t(\w^t\cdot\x) - y)], 
    \end{align*}
    applying the first statement in \Cref{app:claim:(ut-v)(y-ut)geq0} completes the proof.
\end{proof}

\end{document}